\def\year{2022}\relax
\documentclass[letterpaper]{article} 
\usepackage{aaai22}  
\usepackage{times}  
\usepackage{helvet}  
\usepackage{courier}  
\usepackage[hyphens]{url}  
\usepackage{graphicx} 
\usepackage{natbib}  
\usepackage{caption} 
\DeclareCaptionStyle{ruled}{labelfont=normalfont,labelsep=colon,strut=off} 
\usepackage{algorithm}

\usepackage{hyperref}
\usepackage{booktabs}       
\usepackage{amsfonts}       
\usepackage{nicefrac}       
\usepackage{microtype}      
\usepackage{xcolor}         

\usepackage{algpseudocode}

\usepackage{comment}
\usepackage{subfigure}
\usepackage{mathrsfs}
\usepackage{amsmath}
\usepackage{pifont}
\usepackage{wrapfig}

\newtheorem{theorem}{Theorem}
\newtheorem{lemma}{Lemma}
\newtheorem{defini}{Definition}
\newtheorem{Corollary}{Corollary}
\newtheorem{remark}{Remark}

\newtheorem{question}{Question}
\newtheorem{applemma}{Lemma}

\newtheorem{appCorollary}{Corollary}

\newenvironment{proof}{{\noindent\it Proof.}\quad}{\hfill $\square$\par}

\usepackage{newfloat}
\usepackage{listings}
\floatstyle{ruled}
\newfloat{listing}{tb}{lst}{}
\floatname{listing}{Listing}
\title{What About Inputing Policy in Value Function: Policy Representation and Policy-extended Value Function Approximator}
\author {
    Hongyao Tang,\textsuperscript{\rm 1}
    Zhaopeng Meng,\textsuperscript{\rm 1}
    Jianye Hao,\textsuperscript{\rm 1}\thanks{Corresponding author: Jianye Hao $<$jianye.hao@tju.edu.cn$>$. Please contact Hongyao Tang $<$bluecontra@tju.edu.cn$>$ for queries and discussions.}
    Chen Chen, \textsuperscript{\rm 2}
    Daniel Graves, \textsuperscript{\rm 2}
    Dong Li, \textsuperscript{\rm 2}
    Changmin Yu, \textsuperscript{\rm 3}
    Hangyu Mao, \textsuperscript{\rm 2}
    Wulong Liu, \textsuperscript{\rm 2}
    Yaodong Yang,\textsuperscript{\rm 1}
    Wenyuan Tao,\textsuperscript{\rm 1}
    Li Wang\textsuperscript{\rm 1}
}
\affiliations {
    \textsuperscript{\rm 1}College of Intelligence and Computing, Tianjin University
    \textsuperscript{\rm 2}Noah's Ark Lab, Huawei \textsuperscript{\rm 3}University College London\\
}

\usepackage{bibentry}

\begin{document}

\maketitle

\begin{abstract}
  We study Policy-extended Value Function Approximator (PeVFA) in Reinforcement Learning (RL),
  which extends conventional value function approximator (VFA) to take as input not only the state (and action) but also an explicit policy representation.
  Such an extension enables PeVFA to preserve values of multiple policies at the same time
  and brings an appealing characteristic, i.e., \emph{value generalization among policies}.
  We formally analyze the value generalization under Generalized Policy Iteration (GPI).
  From theoretical and empirical lens,
  we show that generalized value estimates offered by PeVFA
  may have lower initial approximation error to 
  true values of successive policies,
  which is expected to improve consecutive value approximation during GPI.
  Based on above clues, we introduce a new form of GPI with PeVFA 
  which leverages the value generalization along policy improvement path.
  Moreover, we propose a representation learning framework for RL policy, 
  providing several approaches to learn effective policy embeddings from policy network parameters or state-action pairs.
  In our experiments, we evaluate the efficacy of value generalization offered by PeVFA and policy representation learning in several OpenAI Gym continuous control tasks.
  For a representative instance of algorithm implementation,
  Proximal Policy Optimization (PPO) re-implemented under the paradigm of GPI with PeVFA achieves about 40\% performance improvement on its vanilla counterpart in most environments.
\end{abstract}

\section{Introduction}
\label{sec:intro}
Reinforcement Learning (RL) has been widely considered as a promising way to learn optimal policies in many decision-making problems  \cite{Mnih2015DQN,Lillicrap2015DDPG,SilverHMGSDSAPL16AlphaGO,YouLYPL18GCPN,schreck2019retrosyn,vinyals2019grandmaster,HafnerLB020Dream}.
One fundamental element of RL
is value function which defines the long-term evaluation of a policy.
With function approximation (e.g., deep neural networks),
a value function approximator (VFA)
is able to approximate the values of a policy under large and continuous state spaces.
As commonly recognized, most RL algorithms can be described as Generalized Policy Iteration (GPI) 
\cite{SuttonB98}.
As illustrated on the left of Fig.\ref{figure:GPI}, 
at each iteration the VFA is trained to approximate the true values of current policy (i.e., policy evaluation),
regarding which the policy is further improved (i.e., policy improvement).
The value function approximation error hinders the effectiveness of policy improvement 
and then the overall optimality of GPI \cite{BertsekasT96NDP,ScherrerGGLG15AMPI}.
Unfortunately, such errors are inevitable under function approximation. 
A large number of samples are usually required to ensure high-quality value estimates,
resulting in the sample-inefficiency of deep RL algorithms.
Therefore, this raises an urgent need for more efficient value approximation methods
\cite{Hasselt10DoubleQ,BellemareDM17C51,Fujimoto2018TD3,KuznetsovSGV20TQC}.

An intuitive idea to improve the efficiency value approximation is to leverage the knowledge on the values of previous encountered policies.
However, a conventional VFA usually approximates the values of one policy and values learned from old policies are over-written gradually during the learning process. 
This means that the previously learned knowledge cannot be preserved and utilized with one conventional VFA.
Thus, such limitations prevent the potentials to leverage the previous knowledge for future learning.
In this paper, we study Policy-extended Value Function Approximator (PeVFA),
which additionally takes an explicit policy representation as input in contrast to conventional VFA. 
Thanks to the policy representation input, PeVFA is able to approximate values for multiple policies
and induces value generalization among policies.
We formally analyze the generalization of approximate values among policies in a general form. 
From both theoretical and empirical lens, 
we show that the generalized value estimates can be closer to the true values of the successive policy,
which can be beneficial to consecutive value approximation along the policy improvement path, called \textit{local generalization}.
Based on above clues,
we introduce a new form of GPI with PeVFA (the right of Fig.\ref{figure:GPI}) that leverages the local generalization to improve the efficiency of consecutive value approximation along the policy improvement path.


\begin{figure*}
\begin{center}
\includegraphics[width=0.75\textwidth]{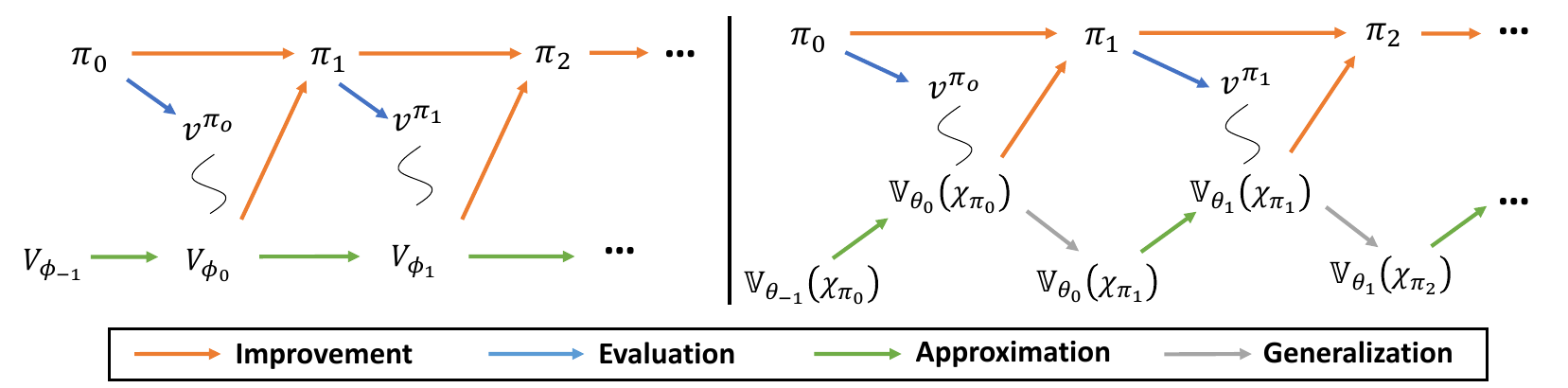}
\end{center}
\vspace{-0.1cm}
\caption{Generalized Policy Iteration (GPI) with function approximation.
\emph{Left}: GPI with conventional value function approximator $V_{\phi}$.
\emph{Right}: GPI with PeVFA $\mathbb{V}_{\theta}(\chi_{\pi})$ (Sec. \ref{sec:PeVFA}) where extra generalization steps exist.
The subscripts of policy $\pi$ and value function parameters $\phi, \theta$ denote the iteration number.
The squiggle lines represent non-perfect approximation of true values.
}
\label{figure:GPI}
\end{figure*}

One key point of GPI with PeVFA is the representation of policy since it determines how PeVFA generalizes the values.
For this,
we propose a framework to learn effective low-dimensional embedding of RL policy.
We use network parameters or state-action pairs as policy data and encode them into low-dimensional embeddings;
then the embeddings are trained to capture the effective information through contrastive learning and policy recovery. 
Finally, we evaluate the efficacy of GPI with PeVFA and our policy representations.
In principle, GPI with PeVFA is general and can be implemented in different ways.
As a practical instance, 
we re-implement Proximal Policy Optimization (PPO) with PeVFA and propose PPO-PeVFA algorithm.
Our experimental results on several OpenAI Gym continuous control tasks demonstrate the effectiveness of both value generalization offered by PeVFA and learned policy representations, 
with an about 40\% improvement in average returns achieved by our best variants on standard PPO in most tasks.

We summarize our main contributions below.
1) We study the value generalization among policies induced by PeVFA. 
From both theoretical and empirical aspects, we shed the light on the situations where the generalization can be beneficial to the learning along policy improvement path.
2) We propose a framework for policy representation learning.
To our knowledge, we make the first attempt to learn a low-dimensional embedding of over 10k network parameters for an RL policy.
3) We introduce GPI with PeVFA that leverages the value generalization in a general form.
Our experimental results demonstrate the potential of PeVFA in deriving practical and more effective RL algorithms.


\section{Related Work}
\label{sec:related_work}

\paragraph{Extensions of Conventional Value Function}
Sutton et al. \shortcite{SuttonMDDPWP11Horde} propose General Value Functions (GVFs) as a general form of knowledge representation of rewards and arbitrary cumulants.
Later, conventional value functions are extended to take extra inputs for different purposes of generalization.
One notable work is Universal Value Function Approximator (UVFA) \cite{SchaulHGS15UVFA}, which is proposed to generalize values among different goals for goal-conditioned RL.
UVFA is further developed in \cite{AndrychowiczCRS17HER,Nachum19HIRO,EysenbachGLS20Rewriting}
and influences the occurrence of other value function extensions in
context-based Meta-RL \cite{RakellyZFLQ19PEARL,Lee20CADM},
Hierarchical RL \cite{WangY0R20I2HRL}
and multiagent RL
\cite{HeB16Opponent,GroverAGBE18MAPR}
and etc.
Most of the above works study how to generalize the policy or value function among extrinsic factors, i.e., environments, tasks and opponents;
while we mainly study the value generalization among policies along policy improvement path, an intrinsic learning process of the agent itself.

\textbf{Policy Embedding and Representation.}
Although 
not well studied, 
representation (or embedding) learning for RL policies is involved in a few works
\cite{HausmanS0HR18PESkill,GroverAGBE18MAPR,ArnekvistKS19VPE}.
The most common way to learn a policy representation is to extract from interaction experiences.
As a representative, \cite{GroverAGBE18MAPR} propose learning the representation of opponent policy from interaction trajectories with a generative policy recovery loss and a discriminative triplet loss. 
These losses are later adopted in \cite{WangY0R20I2HRL,Raileanu20PDVF}.
Another straightforward idea is to represent policy parameters.
Network Fingerprint \cite{Harb20PENs} is such a differentiable representation that uses the concatenation of the vectors of action distribution outputted by policy network on a set of probing states.
The probing state set is co-optimized along with the primary learning objective, which can be non-trivial especially when the dimensionality of the set is high.
Besides, some early attempts in learning low-dimensional embedding of policy parameters are studies in Evolutionary Algorithms \cite{GaierAM20DiscoverRep,Rakicevic21PMS}, mainly with the help of VAE \cite{Kingma2013AEVB}.
Our work introduce a learning framework of policy representation including both above two perspectives. 

\paragraph{PVN and PVFs}
Recently, several works study the generalization among policy space.
\cite{Harb20PENs} propose Policy Evaluation Network (PVN) to directly approximate the distribution of policy $\pi$'s objective function $J(\pi)=\mathbb{E}_{\rho_0}[v^{\pi}(s_0)]$ with initial state $s_0 \sim \rho_0$.
PVN takes as input Network Fingerprint (mentioned above) of policy network.
After training on a pre-collected set of policies, a random initialized policy can be optimized in a zero-shot manner 
with the policy \textbf{g}radients of PVN by backpropagting \textbf{t}hrough the differentiable \textbf{p}olicy \textbf{i}nput.
We call such gradients \textit{GTPI} for short below.
Similar ideas are later integrated with task-specific context learning in multi-task RL \cite{Raileanu20PDVF}, 
leveraging the generalization among policies and tasks for fast policy adaptation on new tasks.
In PVN \cite{Harb20PENs}, as an early attempt, the generalization among policies is studied with small policy network and simple tasks;
besides, the most regular online learning setting is not studied.
Concurrent to our work, \cite{Faccio20PVFs} propose a class of Parameter-based Value Functions (PVFs) that take vectorized policy parameters as inputs.
Based on PVFs, new policy gradient algorithms are introduced in the form of a combination of conventional policy gradients and GTPI (i.e., by backpropagating through policy parameters in PVFs).
Except for zero-shot policy optimization as conducted in PVN, PVFs are also evaluated for online policy learning.
Due to directly taking parameters as input, PVFs suffer from the curse of dimensionality when the number of parameters is high.
Besides, GTPI can be non-trivial to rein since policy parameter space are complex 
and extrapolation generalization error can be large when the value function is only trained on finite policies (usually much fewer than state-action samples) thus further resulting in erroneous policy gradients.
We provide more discussions on GTPI in Appendix \ref{app:GTPI}.

Our work differs with PVFs from several aspects.
First, we make use of learned policy representation rather than policy network parameters.
Second, we do not resort to GTPI for the policy update in our algorithms but focus on utilizing value generalization for more efficient value estimation in GPI. 
Furthermore, 
we shed the light on two important problems --- how value generalization among policies can happen formally and whether it is beneficial to learning or not --- which are neglected in
in previous works
from both theoretical and empirical lens.
We refer one to read the original papers and our Appendix  \ref{app:pr_related_works} for better understandings of the differences.


\section{Policy-extended Value Function Approximator}
\label{sec:PeVFA}

In this section, we propose Policy-extended Value Function Approximator (PeVFA), an extension of conventional VFA that explicitly takes as input a policy representation.
First, we introduce the formulation (Sec. \ref{subsec:problem_formulation}),
then we study value generalization among policies theoretically (Sec. \ref{subsec:generalization}) along with some empirical evidences (Sec. \ref{subsec:demonstrative_exp}).
Finally, we derive a new form of GPI (Sec. \ref{subsec:PeVFA_GPI}).

\subsection{Formulation}
\label{subsec:problem_formulation}
Consider a Markov Decision Process (MDP) 
defined as
$\langle \mathcal{S}, \mathcal{A}, r, \mathcal{P}, \gamma \rangle$ where $\mathcal{S}$ is the state space, $\mathcal{A}$ is the action space, $r$ is the (bounded) reward function, $\mathcal{P}$ is the transition function and $\gamma \in [0,1)$ is the discount factor.
A policy $\pi \in P(\mathcal{A})^{|S|}$ defines the distribution over all actions for each state.
The goal of an RL agent is to find an optimal policy $\pi^{*}$ that maximizes the expected long-term discounted return.
The state-value function $v^{\pi}(s)$ is defined as the expected discounted return obtained through following the policy $\pi$ from a state $s$:
$v^{\pi}(s) = \mathbb{E}_{\pi}  \left[\sum_{t=0}^{\infty} \gamma^{t} r_{t+1}|s_0=s \right]$ for 
where $r_{t+1} = r(s_t,a_t)$.
We use $V^{\pi}$ to denote the vectorized form of value function.

In a general form, we define \textit{policy-extended value function} $\mathbb{V}: \mathcal{S} \times \Pi \rightarrow \mathbb{R}$ over state and policy space:
$\mathbb{V}(s, \pi) = v^{\pi}(s)$ for all $s \in \mathcal{S}$ and $\pi \in \Pi$.
In this paper, we focus on $\mathbb{V}(s, \pi)$ and policy-extended action-value function $\mathbb{Q}(s,a,\pi)$ can be obtained similarly.
We use $\mathbb{V}(\pi)$
to denote the value vector for all states in the following.
The key point is that $\mathbb{V}$ 
is able to preserve the values of multiple policies.
With function approximation, 
a PeVFA is expected to approximate the values of policies among policy space, i.e., $\{V^{\pi}\}_{\pi \in \Pi}$
then enable value generalization among policies.

Formally, 
given a function $g: \Pi \rightarrow \mathcal{X} \subseteq \mathbb{R}^n$ that maps any policy $\pi$ to 
an $n$-dimensional representation $\chi_{\pi} = g(\pi) \in \mathcal{X}$,
a PeVFA $\mathbb{V}_{\theta}$ with parameter $\theta \in \Theta$ 
is to minimize the approximation error over all possible states and policies generally:
$F_{\mu,p,\rho}(\theta, g, \Pi) = \sum_{\pi \in \Pi} \mu(\pi) \|\mathbb{V}_{\theta}(\chi_{\pi}) - V^{\pi}\|_{p, \rho}$,
where $\mu$, $\rho$ are distributions over policies and states respectively, $\|f\|_{p,\rho}=(\int_{s}\rho(\mathrm{d}s)|f(s)|^{p})^{1/p}$ is $\rho$-weighted $L_{p}$-norm \cite{LagoudakisP03LSPI,ScherrerGGLG15AMPI} for any $f: \mathcal{S} \rightarrow \mathbb{R}$.
The policy distribution $\mu$ of interest depends on the scenario where value generalization is considered.
As illustrated in Fig.\ref{figure:twe_generalization}, we provide two value generalization scenarios.
In the global generalization scenario,
a uniform distribution over known policy set may be considered 
with a general purpose of value generalization for unknown policies.
For the specific local generalization scenario along policy improvement path during GPI,
a sophisticated distribution that adaptively weights recent policies more during the learning process
may be more suitable in this case.
In the following, we care more about the local generalization scenario and use uniform state distribution $\rho$ and $L_2$-norm for demonstration. 
The subscripts are omitted and we use $\|\cdot\|$ for clarity.

\begin{figure*}
\centering
\hspace{-0cm}
\subfigure[Global Generalization]{
\includegraphics[width=0.18\textwidth]{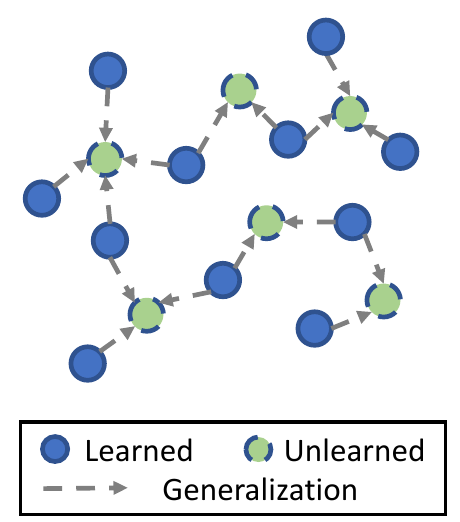}
\label{figure:global_gen}
}
\hspace{-0.cm}
\subfigure[Local Generalization]{
\includegraphics[width=0.57\textwidth]{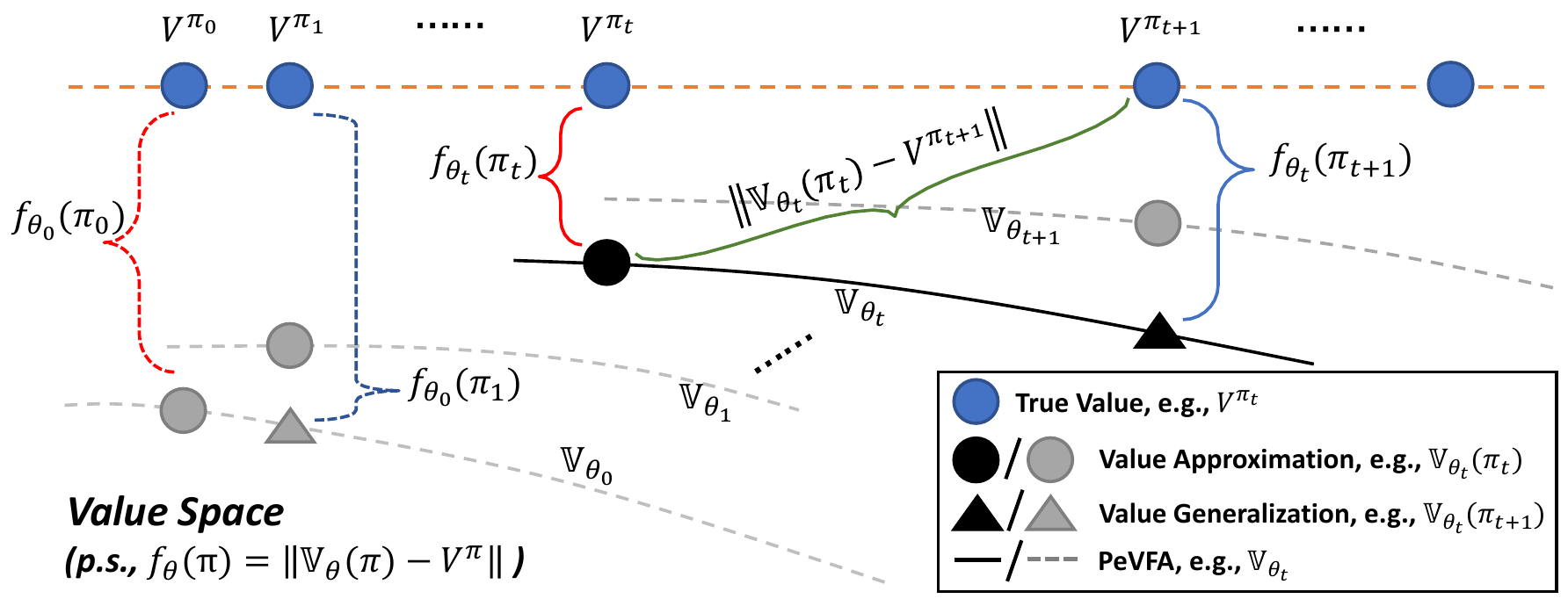}
\label{figure:local_gen}
}
\vspace{-0.2cm}
\caption{
Illustrations of value generalization among policies of PeVFA.
Each circle denotes value function (estimate) of a policy.
(a) \emph{Global Generalization}: values of known policies can be generalized to unknown policies.
(b) \emph{Local Generalization}: values of previous policies (e.g., $\pi_t$) can be generalized to successive policies (e.g., $\pi_{t+1}$) along policy improvement path.
}
\vspace{-0.4cm}
\label{figure:twe_generalization}
\end{figure*}

\subsection{Insights on Generalization among Policies}
\label{subsec:generalization}

In this part, 
we provide preliminary theoretical analysis on value generalization among policies induced by PeVFA,
to shed some light on whether the generalization can be beneficial to conventional RL under GPI paradigm.
We start from a two-policy case and study whether the value approximation learned for one policy can be generalized to the other one.
Later, we study the local generalization scenario (Fig.\ref{figure:local_gen}) and shed the light on the superiority of PeVFA for GPI.
All the proofs are provided in Appendix \ref{app:proofs}.


For the convenience of demonstration, we use an identical policy representation function,
i.e., $\chi_{\pi} = \pi$,
and define the approximation loss of PeVFA $\mathbb{V}_{\theta}$ for any policy $\pi \in \Pi$
as $f_{\theta}(\pi) = \| \mathbb{V}_{\theta}(\pi) - V^{\pi} \| \ge 0$.
We use the following definitions
for a formal description of value approximation process with PeVFA
and local property of loss function $f_\theta$ that influences generalization \cite{NovakBAPS18SenGen,Wang18GenPro} respectively:
\begin{defini}[$\pi$-Value Approximation]
\label{defi:approx_mapping}
We define a value approximation process $\mathscr{P}_{\pi}: \Theta \rightarrow \Theta$ with PeVFA as a $\gamma$-contraction mapping on the approximation loss for policy $\pi$, i.e., 
for $\hat{\theta} = \mathscr{P}_{\pi}(\theta)$, we have $f_{\hat{\theta}}(\pi) \le \gamma f_{\theta}(\pi)$ where $\gamma \in [0,1)$.
\vspace{-0.1cm}
\end{defini}
\begin{defini}[$L$-Continuity]
\label{defi:l_continuity}
We call $f_{\theta}$ is $L$-continuous at policy $\pi$ 
if $f_{\theta}$ is Lipschitz continuous at $\pi$ with a constant $L \in [0, \infty)$, 
i.e., $| f_{\theta}(\pi) - f_{\theta}(\pi^{\prime}) | \le L \cdot
d(\pi,\pi^{\prime})$ for $\pi^{\prime} \in \Pi$ 
with some distance metric $d$ for policy space $\Pi$.
\end{defini}
With Definition \ref{defi:approx_mapping}, the consecutive value approximation for the policies along policy improvement path during GPI can be described as:
$\theta_{-1} \xrightarrow{\mathscr{P}_{\pi_0}} \theta_{0} \xrightarrow{\mathscr{P}_{\pi_1}} \theta_{1} \xrightarrow{\mathscr{P}_{\pi_2}} \dots$,
as the green arrows illustrated in Fig.\ref{figure:GPI}.
One may refer to Appendix \ref{disc:defini} for a discussion on the rationality of the two definitions.

To start our analysis, we first study the generalized value approximation loss in a two-policy case where only the value of policy $\pi_1$ is approximated by PeVFA as below:
\begin{lemma}
\label{lem:gen_bound}
\vspace{-0.2cm}
For $\theta \xrightarrow{\mathscr{P}_{\pi_1}} \hat{\theta}$,
if $f_{\hat{\theta}}$ is $\hat{L}$-continuous at $\pi_1$ and $f_{\theta}(\pi_1) \le f_{\theta}(\pi_2)$, we have: 
$f_{\hat{{\theta}}}(\pi_2) \le \gamma f_{\theta}(\pi_2) + \mathcal{M}(\pi_1, \pi_2, \hat{L})$,
where $\mathcal{M}(\pi_1, \pi_2, \hat{L}) = \hat{L} \cdot d(\pi_1,\pi_2)$.
\vspace{-0.2cm}
\end{lemma}
\begin{Corollary}
\label{cor:gen_cond}
$\mathscr{P}_{\pi_1}$
is $\gamma_{g}$-contraction ($\gamma_{g} \in [0,1)$) for $\pi_2$ when $f_{\theta}(\pi_2) > \frac{\hat{L} \cdot d(\pi_1, \pi_2)}{1 - \gamma}$.
\vspace{-0.2cm}
\end{Corollary}
Lemma \ref{lem:gen_bound} shows that the post-$\mathscr{P}_{\pi_1}$ approximation loss for $\pi_2$ is upper bounded by
a generalized contraction of prior loss
plus a locality margin term $\mathcal{M}$ which is related to $\pi_1$, $\pi_2$ and 
the locality property of $f_{\hat{\theta}}$.
In general, the form of $\mathcal{M}$ depends on the local property assumed. 
Some higher-order variants are provided in Appendix \ref{prf:lem1}.
For a step further, Corollary \ref{cor:gen_cond} reveals the condition where
a contraction on value approximation loss for $\pi_2$ is achieved when PeVFA is only trained to approximate the values of $\pi_1$.
Concretely, such a condition is apt to reach 
with tighter contraction for policy $\pi_1$ is, closer two policies, or smoother approximation loss function $f_{\hat{\theta}}$.

%

Then we consider the local generalization scenario
as illustrated in Fig.\ref{figure:local_gen}.
For any iteration $t$ of GPI,
the values of current policy $\pi_t$ are approximated by PeVFA,
followed by a improved policy $\pi_{t+1}$ whose values are to be approximated in the next iteration.
The value generalization from each $\pi_t$ and $\pi_{t+1}$ can be similarly considered as the two-policy case.
In addition to the former results, we shed the light on the value generalization loss of PeVFA along policy improvement path below:
\begin{lemma}
\label{lem:recursive_relation}
\vspace{-0.2cm}
For $\theta_{-1} \xrightarrow{\mathscr{P}_{\pi_0}} \theta_{0} \xrightarrow{\mathscr{P}_{\pi_1}} \theta_{1} \xrightarrow{\mathscr{P}_{\pi_2}} \dots$ with $\gamma_t$ for each $\mathscr{P}_{\pi_t}$,
if $f_{\theta_t}$ is $L_{t}$-continuous at $\pi_t$ for any $t \ge 0$,
we have
$f_{\theta_t}(\pi_{t+1}) \le \gamma_t f_{\theta_{t-1}}(\pi_{t}) + \mathcal{M}_{t}$,
where $\mathcal{M}_{t} = L_{t} \cdot d(\pi_t, \pi_{t+1})$.
\vspace{-0.15cm}
\end{lemma}
\begin{Corollary}
\label{cor:local_gen_induction}
By induction, we have
$f_{\theta_t}(\pi_{t+1}) \le \prod_{i=0}^t \gamma_t f_{\theta_{-1}}(\pi_{0}) + \sum_{i=0}^{t-1} \prod_{j=i+1}^{t} \gamma_j \mathcal{M}_{i} + \mathcal{M}_{t}$ .
\vspace{-0.1cm}
\end{Corollary}
The above results indicate that the value generalization loss can be recursively bounded and has a upper bound formed by a repeated contraction on initial loss plus the accumulation of locality margins induced from each local generalization.
An infinity-case discussion
is in Appendix \ref{prf:cor2}.

Although value generalization among policies demonstrated above is also intuitive,
it is not necessarily useful to conventional RL under GPI paradigm.
Therefore, the next question is naturally whether 
PeVFA with value generalization among policies is preferable to the conventional VFA;
if yes, what the case is to be like.
To this end, 
we 
introduce a desirable condition which reveals the superiority of PeVFA during
consecutive value approximation along the policy improvement path:
\begin{theorem}
\label{thm:closer_target}
During $\theta_{-1} \xrightarrow{\mathscr{P}_{\pi_0}} \theta_{0} \xrightarrow{\mathscr{P}_{\pi_1}} \theta_{1} \xrightarrow{\mathscr{P}_{\pi_2}} \dots$,
for any $t \ge 0$, if $f_{\theta_t}(\pi_t) + f_{\theta_t}(\pi_{t+1}) \le 
\|V^{\pi_t} - V^{\pi_{t+1}}\|$,
then $f_{\theta_t}(\pi_{t+1}) \le \|\mathbb{V}_{\theta_{t}}(\pi_t) - V^{\pi_{t+1}}\| $.
\end{theorem}
Theorem \ref{thm:closer_target} shows that the generalized value estimates $\mathbb{V}_{\theta_{t}}(\pi_{t+1})$ can be closer to the true values of policy $\pi_{t+1}$ than $\mathbb{V}_{\theta_{t}}(\pi_t)$.
Note that $\mathbb{V}_{\theta_{t}}(\pi_t)$ is the value approximation for $\pi_t$ which is equivalent to the counterpart $V_{\phi_t}$ for a conventional VFA as value generalization among policies does not exist.
%
To consecutive value approximation along policy improvement path, 
this means that the value generalization of PeVFA has the potential to offer closer start points at each iteration.
If such closer start points can often exist,
we expect PeVFA to be preferable to conventional VFA since 
value approximation can be more efficient with PeVFA and it in turn facilitates the overall GPI process.

However, the condition in Theorem \ref{thm:closer_target} is not necessarily met in practice.
It depends on the locality margins that may be related to function family and optimization method of PeVFA, as well as the scale of policy improvement.
We conjecture that there are many looser sufficient conditions that lead to the consequence of Theorem 1, and the presented condition is the strictest one among them to achieve.
This can be interpreted by considering the geometrical relationship between $V^{\pi_t},V^{\pi_{t+1}},\mathbb{V}_{\theta_t}(\pi_t)$ and $\mathbb{V}_{\theta_t}(\pi_{t+1})$.
One special case that requires the sufficient condition presented in Theorem 1 is where $\mathbb{V}_{\theta_t}(\pi_t)$ and $\mathbb{V}_{\theta_t}(\pi_{t+1})$ and lie on the line segment between $V^{\pi_t}$ and $V^{\pi_{t+1}}$. 
We leave these further theoretical investigations for future work.
Instead, we empirically examine the existence of such desirable generalization results in the following.

\subsection{Empirical Evidences}
\label{subsec:demonstrative_exp}
We empirically investigate the value generalization of PeVFA
with didactic environments.
In this section, PeVFA $\mathbb{V}_{\theta}$ is parameterized by neural network 
and we use the concatenation of all weights and biases of the policy network as a straightforward representation $\chi_{\pi}$ for each policy, called \textit{Raw Policy Representation (RPR)}.
Experimental details are provided in Appendix \ref{app:demon_details}.

First, we demonstrate the global generalization (illustrated in Fig.\ref{figure:global_gen}) in a continuous 2D Point Walker environment.
We build the policy set $\Pi$ with synthetic policies, each of which is a randomly initialized 2-layer \emph{tanh}-activated neural network with 2 units for each layer.
The size of $\Pi$ is 20k and the behavioral diversity of synthetic policies is verified (see Fig.\ref{figure:syn_popu} in Appendix).
We divide $\Pi$ into training set (i.e., known policies $\Pi_0$) and testing set (i.e., unseen policies $\Pi_1$).
We rollout the policies in the environment to collect trajectories, based on which we perform value approximation training.
Fig.\ref{figure:global_gen_order_opt} shows the value predictions for policies from training and testing set (100 for each).
Our results show that a PeVFA trained on $\Pi_0$ achieves reasonable generalization performance when evaluating on $\Pi_1$.
The average losses on training and testing set are 1.782 and 2.071 over 6 trials.

\begin{figure}[t]
\centering
\hspace{-0.25cm}
\subfigure[]{
\includegraphics[width=0.16\textwidth]{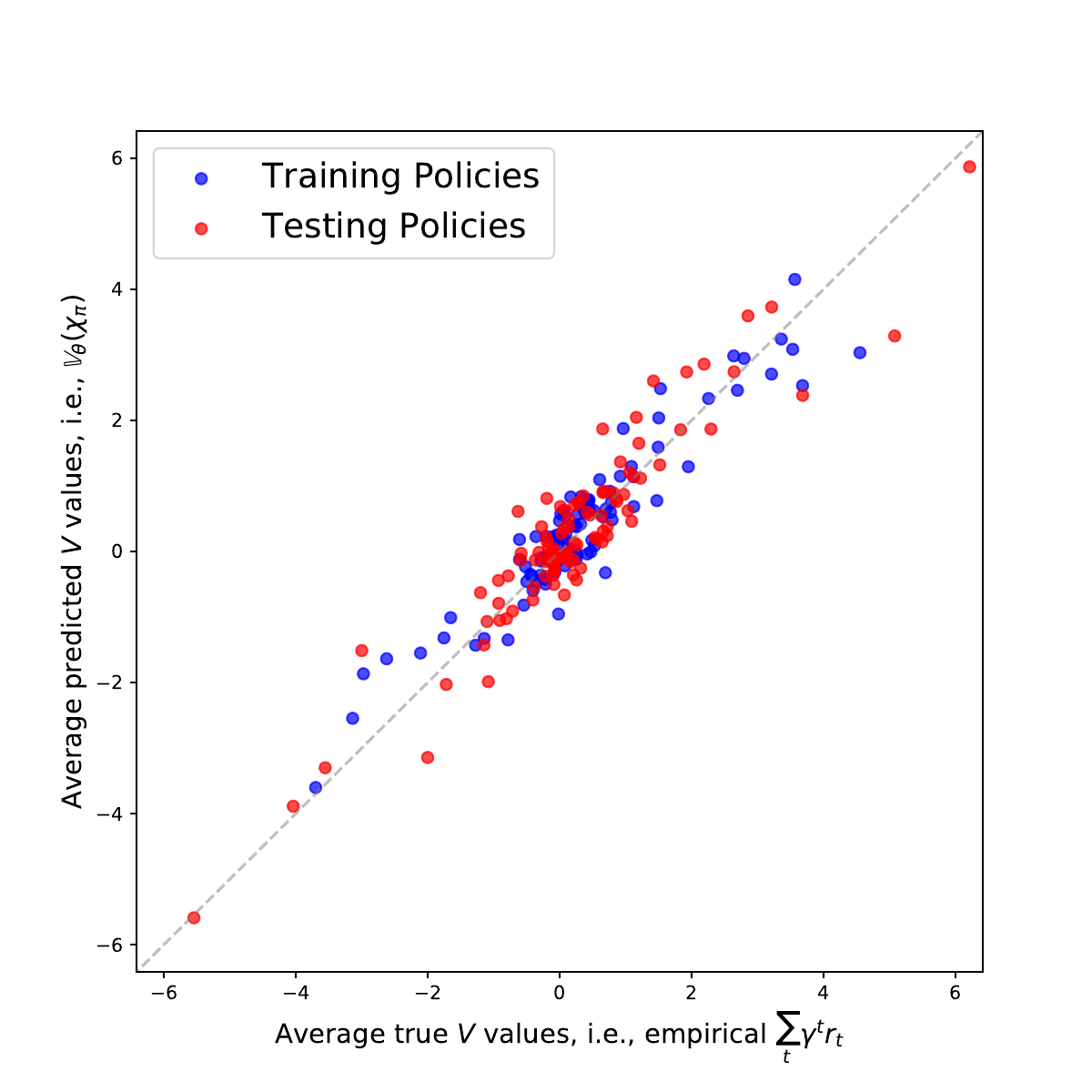}
\label{figure:global_gen_order_opt}
}
\hspace{0.cm}
\subfigure[]{
\includegraphics[width=0.24\textwidth]{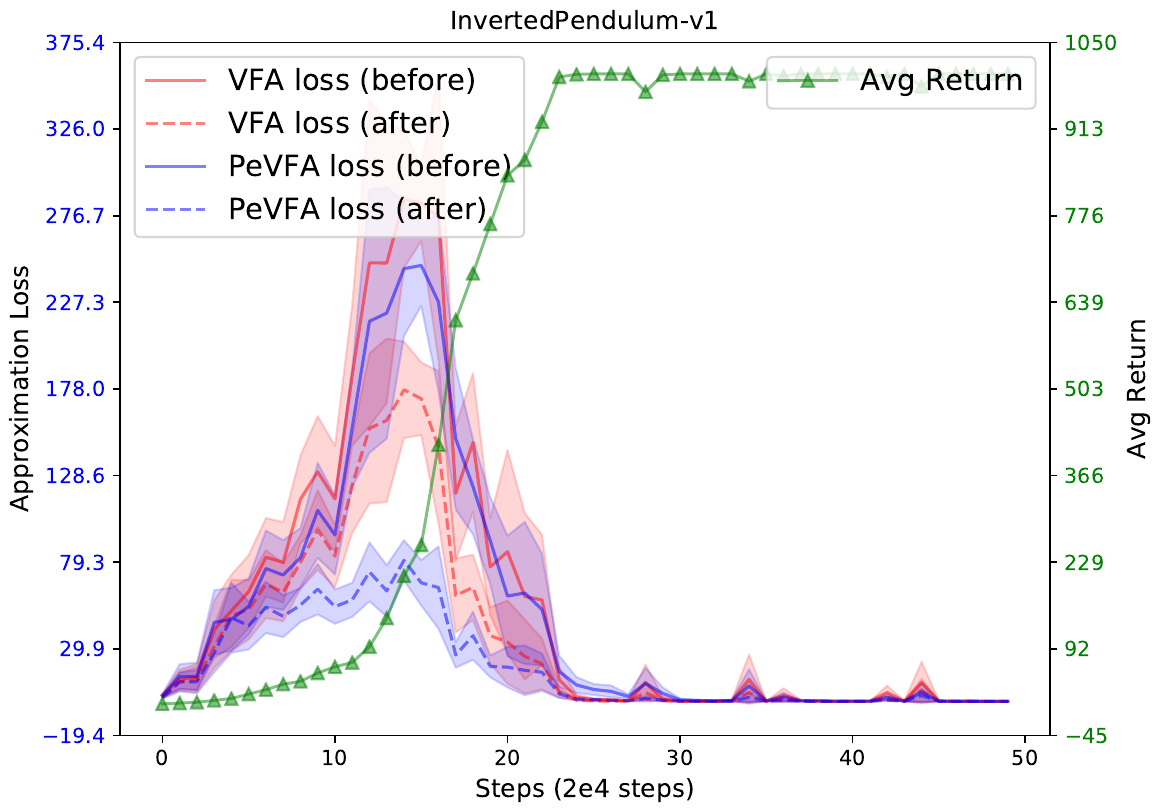}
\label{figure:local_gen_invertpen}
}
\vspace{-0.2cm}
\caption{Empirical evidences of two kinds of generalization of PeVFA.
(a) \emph{Global generalization}:
PeVFA shows comparable value estimation performance on testing policy set (red) after learning on training policy set (blue).
(b) \emph{Local generalization}: PeVFA ($\mathbb{V}_{\theta}(\chi_{\pi})$) shows lower losses
than conventional VFA ($V_{\phi}$) before and after the value approximation training for successive policies
along policy improvement path.
In (b), the left axis is for approximation loss (lower is better) and the right axis is for average return as a reference of the policy learning process (green curve).
}
\vspace{-0.5cm}
\label{figure:demon_exp}
\end{figure}

Next, we investigate the value generalization along policy improvement path, i.e., local generalization as in Fig.\ref{figure:local_gen}.
We use a 2-layer 8-unit policy network trained by standard PPO algorithm \cite{Schulman2017PPO} in MuJoCo continuous control tasks.
Parallel to the conventional value network $V_{\phi}(s)$ (i.e., VFA) in PPO, 
we set a PeVFA network $\mathbb{V}_{\theta}(s,\chi_{\pi})$ as a reference for the comparison on value approximation loss.
Compared to $V_{\phi}$,
PeVFA $\mathbb{V}_{\theta}(s,\chi_{\pi})$ takes RPR as input and approximates the values of all historical policies ($\{\pi_i\}_{i=0}^{t}$) in addition.
We compare the value approximation losses of $V_{\phi}$ (red) and $\mathbb{V}_{\theta}$ (blue)
before (solid) and after (dashed) updating with on-policy samples collected by the improved policy $\pi_{t+1}$ at each iteration.
Fig.\ref{figure:local_gen_invertpen} shows the results for InvertedPendulum-v1.
Results for all 7 MuJoCo tasks can be found in Appendix \ref{app:local_gen_mujoco}.
By comparing approximation losses before updating (red and blue solid curves), we can observe that
the approximation loss of $\mathbb{V}_{\theta_{t}}(\chi_{\pi_{t+1}})$ is almost consistently lower than that of $V_{\phi_t}$.
This means that the generalized value estimates offered by PeVFA are usually closer to the true values of $\pi_{t+1}$,
demonstrating the consequence arrived in Theorem \ref{thm:closer_target}.
For the dashed curves, 
it shows that PeVFA $\mathbb{V}_{\theta_{t+1}}(\chi_{\pi_{t+1}})$ can achieve lower approximation loss for $\pi_{t+1}$ than conventional VFA $V_{\phi_{t+1}}$ 
after the same number of training with the same on-policy samples.
The empirical evidence above indicates that PeVFA can be preferable to the conventional VFA for consecutive value approximation.
The generalized value estimates along policy improvement path have the potential to expedite the process of GPI.

\begin{algorithm}[t]
\small
  \caption{RL under the paradigm of GPI with PeVFA ($\mathbb{V}(s,\chi_{\pi})$ is used for demonstration)}
  \begin{algorithmic}[1]
    \State Initialize policy $\pi_{0}$, policy representation model $g$, PeVFA $\mathbb{V}_{-1}$ and experience buffer $\mathcal{D}$ 
	\For{iteration $t = 0, 1, \dots$}
        \State Rollout policy $\pi_{t}$ in the environment and obtain $k$ trajectories $\mathcal{T}_{t} = \{\tau_{i}\}_{i=0}^{k}$ 
        \State  Get representation $\chi_{\pi_t} = g(\pi)$ for policy $\pi_t$ and add experiences $(\chi_{\pi_t}, \mathcal{T}_{t})$ in buffer $\mathcal{D}$
        \If{$t \ \% \ M = 0$}
            \State Update PeVFA $\mathbb{V}_{t-1}(s,\chi_{\pi_i})$ for previous policies with data $\{(\chi_{\pi_i}, \mathcal{T}_{i})\}_{i=0}^{t-1}$
            \State Update policy representation model $g$, e.g., with approaches provided in Sec. \ref{sec:policy_repres}
        \EndIf
        \State Update PeVFA $\mathbb{V}_{t-1}(s,\chi_{\pi_t})$ for current policy $\chi_{\pi_t}$ and set $\mathbb{V}_{t} \longleftarrow \mathbb{V}_{t-1}$
        \State Update $\pi_t$ w.r.t $\mathbb{V}_{t}(s,\chi_{\pi_t})$ by policy improvement algorithm and set $\pi_{t+1} \longleftarrow \pi_{t}$
	\EndFor
  \end{algorithmic}
\label{algorithm:RL_PeVFA}
\end{algorithm}

\subsection{Reinforcement Learning with PeVFA}
\label{subsec:PeVFA_GPI}

Based on the results above,
we expect to leverage the value generalization of PeVFA to facilitate RL.
In Algorithm \ref{algorithm:RL_PeVFA}, 
we propose a general description of RL algorithm under the paradigm of GPI with PeVFA.
For each iteration, the interaction experiences of current policy and the policy representation are stored in a buffer (line 3-4).
At an interval of $M$ iterations,
PeVFA is trained via value approximation for previous policies with the stored data and the policy representation model is updated according to the method used (line 5-8).
This part is unique to PeVFA for preservation and generalization of knowledge of historical policies.
Next, 
value approximation for current policy is performed with PeVFA (line 9).
A key difference here is that the generalized value estimates (i.e., $\mathbb{V}_{t-1}(\chi_{\pi_{t}})$) are used as start points.
Afterwards, 
a successive policy is obtained from typical policy improvement (line 10).
Algorithm \ref{algorithm:RL_PeVFA} can be implemented in different ways and we propose an instance implemented based on PPO \cite{Schulman2017PPO} in our experiments later.
We refer the readers to Appendix \ref{app:GPI_with_PeVFA} for more discussions on GPI with PeVFA.
In the next section, we introduce our methods for policy representation learning.


\section{Policy Representation Learning}
\label{sec:policy_repres}

\begin{figure}
\begin{center}
\includegraphics[width=0.475\textwidth]{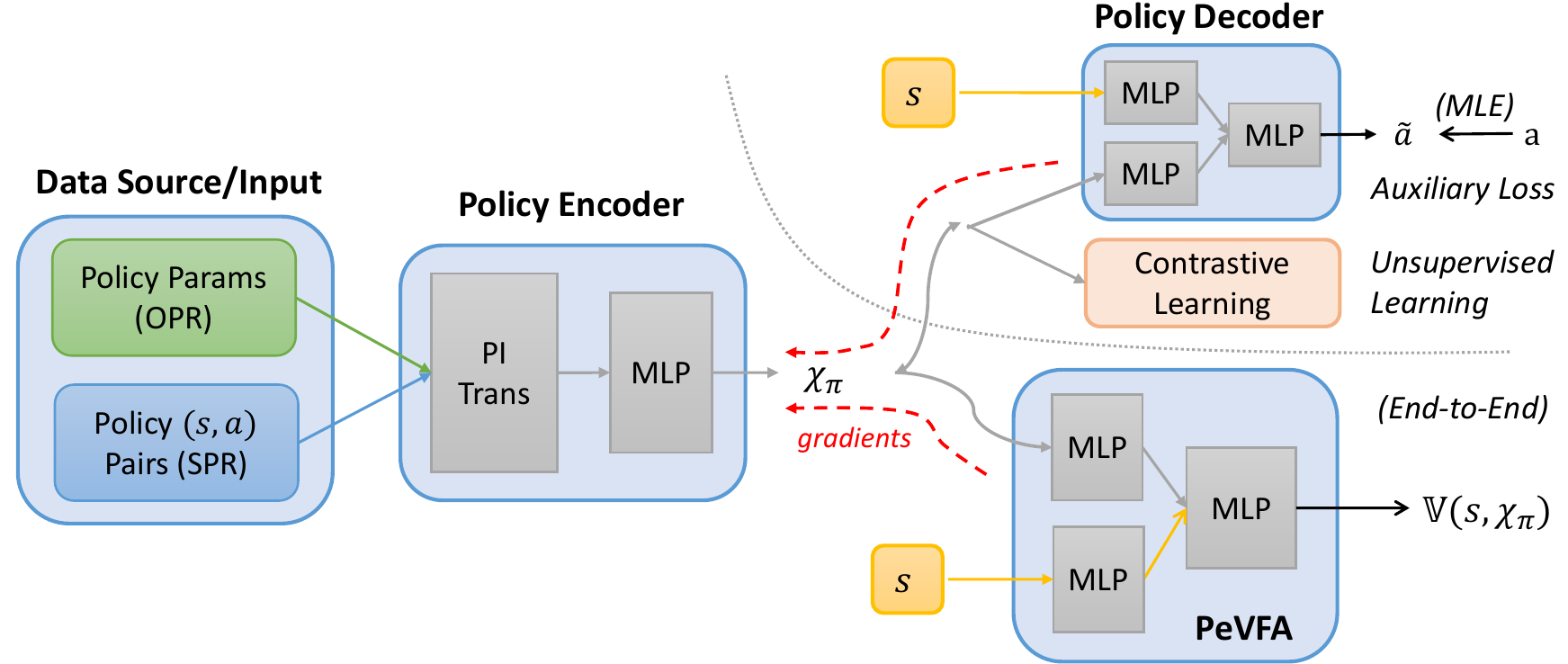}
\end{center}
\vspace{-0.2cm}
\caption{The framework of policy representation training.
Policy network parameters used for OPR or policy state-action pairs used for SPR are fed into policy encoder with permutation-invariant (PI) transformations followed by an MLP, producing the representation $\chi_{\pi}$.
Afterwards, $\chi_{\pi}$ can be trained by gradients from the value approximation loss of PeVFA (i.e., End-to-End), 
as well as the auxiliary loss of policy recovery or contrastive learning (i.e., InfoNCE) loss.
}
\vspace{-0.4cm}
\label{figure:pr_framework}
\end{figure}

To derive practical deep RL algorithms,
one key point is policy representation, i.e., a low-dimensional embedding of RL policy.
Intuitively, policy representation influences the approximation and generalization of PeVFA.
Thus, 
it is of interest to find 
an effective policy representation
based on which the superiority of PeVFA can be
leveraged to improve RL algorithms.
To our knowledge, policy representation is not well studied
and 
it remains unclear on how to obtain an effective representation for an RL policy in a general case in practice.
In previous section, 
we demonstrate the effectiveness of using policy parameters as a naive representation when policy network is small,
called RPR.
However, a usual policy network may have large number of parameters,
thus making it inefficient and even irrational to use RPR for approximation and generalization \cite{Harb20PENs,Faccio20PVFs}.
More generally, 
policy parameters of the policy we wish to represent may not be accessible.

To this end, we propose a general framework of policy representation learning as illustrated in Fig.\ref{figure:pr_framework}.
The first thing to consider is data source, i.e., from which we can extract the information for an effective policy representation.
Recall that the policy is a distribution over state and action space of high dimensionality.
The features of such a distribution is not directly available.
Therefore, we consider two kinds of data source below that indirectly contains the information of policies:
1) \emph{Surface Policy Representation (SPR)}:
The first data source is state-action pairs (or trajectories \cite{GroverAGBE18MAPR}), 
since they reflect how policy may behave under such states.
This data source is general since no explicit form of policy is assumed.
In a geometric view, learning policy representation from state-action pairs can be viewed as capturing the features of policy via scattering sample points on the curved surface of policy distribution.
2) \emph{Origin Policy Representation (OPR)}:
The other data source is parameters of policy since they determine the underlying form of policy distribution.
Such a data source is often available during the learning process of deep RL algorithms when policy is parameterized by neural networks.
Generally, we consider a policy network to be an MLP with well represented state features (e.g., features extracted by CNN for pixels or by LSTM for sequences) as input.

The remaining question is how we extract the policy representation from the data sources mentioned above.
As shown in Fig.\ref{figure:pr_framework}, 
we use permutation-invariant (PI) transformations followed by an MLP
to encode the data of policy $\pi$ into an embedding $\chi_{\pi}$ for both SPR and OPR.
For SPR, each state-action pair of $\{(s_i,a_i)\}_{i=1}^k$ is fed into a common MLP, followed by a Mean-Reduce operation on the outputted features across $k$.
For OPR, we perform PI transformation (similar as done for state-action pairs) inner-layer weights and biases $\{(w_i, b_i)\}_{i=1}^h$ for each layer first, 
where $h$ denotes the number of nodes in this layer and $w_i, b_i$ is the income weight vector from previous layer and the bias of $i$th node;
then we concatenate encoding of layers and obtain the OPR.
An illustration for the encoding of OPR is in Fig.\ref{figure:opr_encoder} of Appendix.

To train the policy embedding $\chi_{\pi}$
obtained above,
the most straightforward way is to backpropagate the value approximation loss of PeVFA in an \textit{End-to-End (E2E)} fashion as illustrated on the lower-right of Fig.\ref{figure:pr_framework}.
In addition, we provide two self-supervised training losses for both OPR and SPR, as illustrated on the upper-right of Fig.\ref{figure:pr_framework}.
The first one is an \textit{auxiliary loss (AUX) of policy recovery} \cite{GroverAGBE18MAPR}, i.e., to recover the action distributions of $\pi$ from $\chi_{\pi}$ under different states.
To be specific, an auxiliary policy decoder $\bar{\pi}(\cdot|s,\chi_{\pi})$ is trained through
behavioral cloning,
formally to minimize cross-entropy objective
$\mathcal{L}_{\text{AUX}} =- \mathbb{E}_{(s,a)} \left[ \log \bar{\pi}(a|s,\chi_{\pi}) \right]$.
For the second one,
we propose to train $\chi_{\pi}$ by \textit{Contrastive Learning (CL)} \cite{Srinivas20CURL,Schwarzer20MPR}:
policies are encouraged to be close to similar ones (i.e., positive samples $\pi^{+}$),
and to be apart from different ones (i.e., negative samples $\pi^{-}$) in representation space.
For each policy, we construct positive samples by data augmentation on policy data, depending on SPR or OPR considered;
and different policies along the policy improvement path naturally provide negative samples for each other.
Finally, the embedding $\chi_{\pi}$ is optimized through minimizing the InfoNCE loss \cite{Oord18CPC} below: 
$\mathcal{L}_{\text{CL}} = - \mathbb{E}_{(\pi^{+},\{\pi^{-}\})}
\left[ 
\log \frac{\exp(\chi_{\pi}^{T} W \chi_{\pi^{+}})}{\exp(\chi_{\pi}^{T} W \chi_{\pi^{+}}) + \sum_{\pi^{-}} \exp(\chi_{\pi}^{T} W \chi_{\pi^{-}})} 
\right]$.

Now, the training of policy representation in Algorithm \ref{algorithm:RL_PeVFA} can be performed with any combination of data sources and training losses provided above.
The pseudo-code of the overall policy representation training framework is shown in Algorithm \ref{algorithm:overall_view_pr} in Appendix \ref{app:overall_view_pr}.
In addition, 
complete implementation details and more discussions (e.g., on the scalability, representation criteria) are provided in Appendix \ref{app:PR}.

\section{Experiments}
\label{sec:experiments}

\begin{table*}[t]
  \caption{Average returns ($\pm$ half a std) over 10 trials for algorithms. 
  Each result is the maximum evaluation along the training process.
  Top two values for each environment are bold. 
  }
  \label{table:overall_evaluations}
  \centering  
  \scalebox{0.68}{
  \begin{tabular}{c|ccc|ccc|ccc}
    \toprule 
    & \multicolumn{3}{c}{Benchmarks} & \multicolumn{3}{c}{Origin Policy Representation (Ours)} & \multicolumn{3}{c}{Surface Policy Representation (Ours)} \\
    \cmidrule(r){2-4} \cmidrule(r){5-7} \cmidrule(r){8-10}
    Environments & PPO & Ran PR & RPR & E2E & CL & AUX & E2E & CL & AUX\\
    \midrule
    HalfCheetah-v1 &
    2621 & 2470 & 2325 \small{$\pm$ 399.27} & 3171 \small{$\pm$ 427.63} & \textbf{3725 \small{$\pm$ 348.55}} & 3175 \small{$\pm$ 517.52} & 2774 \small{$\pm$ 233.39} & \textbf{3349 \small{$\pm$ 341.42}} & 3216 \small{$\pm$ 506.39} \\
    Hopper-v1 &
    1639 & 1226 & 1097 \small{$\pm$ 213.47} & 2085 \small{$\pm$ 310.91} & 2351 \small{$\pm$ 231.11} & 2214 \small{$\pm$ 360.78} & 2227 \small{$\pm$ 297.35} & \textbf{2392 \small{$\pm$ 263.93}} & \textbf{2577 \small{$\pm$ 217.73}} \\
    Walker2d-v1 &
    1505 & 1269 & 317 \small{$\pm$ 152.68} & 1856 \small{$\pm$ 305.51} & 2038 \small{$\pm$ 315.51} & \textbf{2044 \small{$\pm$ 316.32}} & 1930.57  \small{$\pm$ 456.02} & \textbf{2203 \small{$\pm$ 381.95}} & 1980 \small{$\pm$ 325.54} \\
    Ant-v1 &
    2835 & 2742 & 2143 \small{$\pm$ 406.64} & 3581 \small{$\pm$ 185.43} & \textbf{4019 \small{$\pm$ 162.47}} & \textbf{3784 \small{$\pm$ 268.99}} & 3173 \small{$\pm$ 184.75} & 3632 \small{$\pm$ 134.27} & 3397 \small{$\pm$ 200.03} \\
    InvDouPend-v1 & 
    9344 & 9355 & 8856 \small{$\pm$ 551.90} & \textbf{9357 \small{$\pm$ 0.29}} & 9355 \small{$\pm$ 0.64} & 9355 \small{$\pm$ 0.68} & 9355 \small{$\pm$ 0.89} & \textbf{9356 \small{$\pm$ 0.96}} & 9355 \small{$\pm$ 1.42} \\
    LunarLander-v2 & 
    219 & 226 & -22 \small{$\pm$ 35.08} & \textbf{238 \small{$\pm$ 3.37}} & \textbf{239 \small{$\pm$ 3.70}} & 234 \small{$\pm$ 3.47} & 236 \small{$\pm$ 3.13} & 234 \small{$\pm$ 3.13} & 235 \small{$\pm$ 5.70} \\
    \bottomrule
  \end{tabular}
  }
\vspace{-0.4cm}
\end{table*}

In this section, we conduct experimental study with focus on the following questions:
\begin{question}
\vspace{-0.1cm}
\label{question:pevfa}
    Can value generalization offered by PeVFA improve a deep RL algorithm in practice?
\vspace{-0.25cm}
\end{question}
\begin{question}
\label{question:policy_repre}
    Can our proposed framework to learn effective policy representation?
\vspace{-0.1cm}
\end{question}
Our experiments are conducted in several OpenAI Gym continuous control tasks (1 from Box2D and 5 from MuJoCo) \cite{Brockman2016Gym,Todorov2012MuJoCo}.
All experimental details and curves can be found in Appendix \ref{app:demon_details}.


\textbf{Algorithm Implementation.}
We use PPO \cite{Schulman2017PPO} as the basic algorithm and propose a representative implementation of Algorithm \ref{algorithm:RL_PeVFA}, called \textbf{PPO-PeVFA}.
PPO is a policy optimization algorithm that follows the paradigm of GPI (Fig.\ref{figure:GPI}, left).
A value network $V_{\phi}(s)$ with parameters $\phi$ (i.e., conventional VFA) is trained to approximate the value of current policy $\pi$;
while $\pi$ is optimized with respect to a surrogate objective \cite{Schulman2015TRPO} using advantages calculated by $V_{\phi}$ and GAE \cite{Schulman2016GAE}.
Compared with original PPO,
PPO-PeVFA makes use of a PeVFA network $\mathbb{V}_{\theta}(s, \chi_{\pi})$ with parameters $\theta$ rather than the conventional VFA $V_{\phi}(s)$,
and follows the training scheme as in Algorithm \ref{algorithm:RL_PeVFA}.
Note PPO-PeVFA uses the same policy optimization method as original PPO and only differs at value approximation.

\textbf{Baselines and Variants.}
Except for original PPO as a default baseline,
we use another two baselines:
1) PPO-PeVFA with randomly generated policy representation for each policy, denoted by \textbf{Ran PR}; 
2) PPO-PeVFA with Raw Policy Representation (\textbf{RPR}), i.e., use the vector of all parameters of policy network as representation as adopted in PVFs \cite{Faccio20PVFs}.
Our variants of PPO-PeVFA differ at the policy representation used.
In total, we consider 6 variants denoted by the combination of the policy data choice (i.e., \textbf{OPR}, \textbf{SPR}) and representation principle choice (i.e., \textbf{E2E}, \textbf{CL}, \textbf{AUX}).

\textbf{Experimental Details.}
For all baselines and variants, we use a normal-scale policy network with 2 layers and 64 units for each layer, 
resulting in over 3k to 10k (e.g., Ant-v1) policy parameters depending on the environments.
We do not assume the access to pre-collected policies.
Thus the size of policy set increases from 1 (i.e., the initial policy) during the learning process,
to about 1k to 2 for a single trial.
The dimensionality of all kinds of policy representation expect for RPR is set to 64.
The buffer $D$ maintains recent 200k steps of interaction experience and the policy data of corresponding policy.
The number of interaction step of each trial is 1M for InvDouPend-v1 and LunarLander-v2, 4M for Ant-v1 and 2M for the others.

\textbf{Results.}
The overall experimental results are summarized in Tab.\ref{table:overall_evaluations}.
In Fig.\ref{figure:evaluation_merged}, we provide aggregated results across all environments expect for InvDouPend-v1 and LunarLander-v2 (since most algorithms achieve near-optimal results),
where all returns are normalized by the results of PPO in Tab.\ref{table:overall_evaluations}.
Full learning curves are omitted and can be found in Appendix \ref{app:complete_learning_curves}.

\begin{figure}[t]
\centering
\hspace{-0.1cm}
\subfigure[]{
\includegraphics[width=0.225\textwidth]{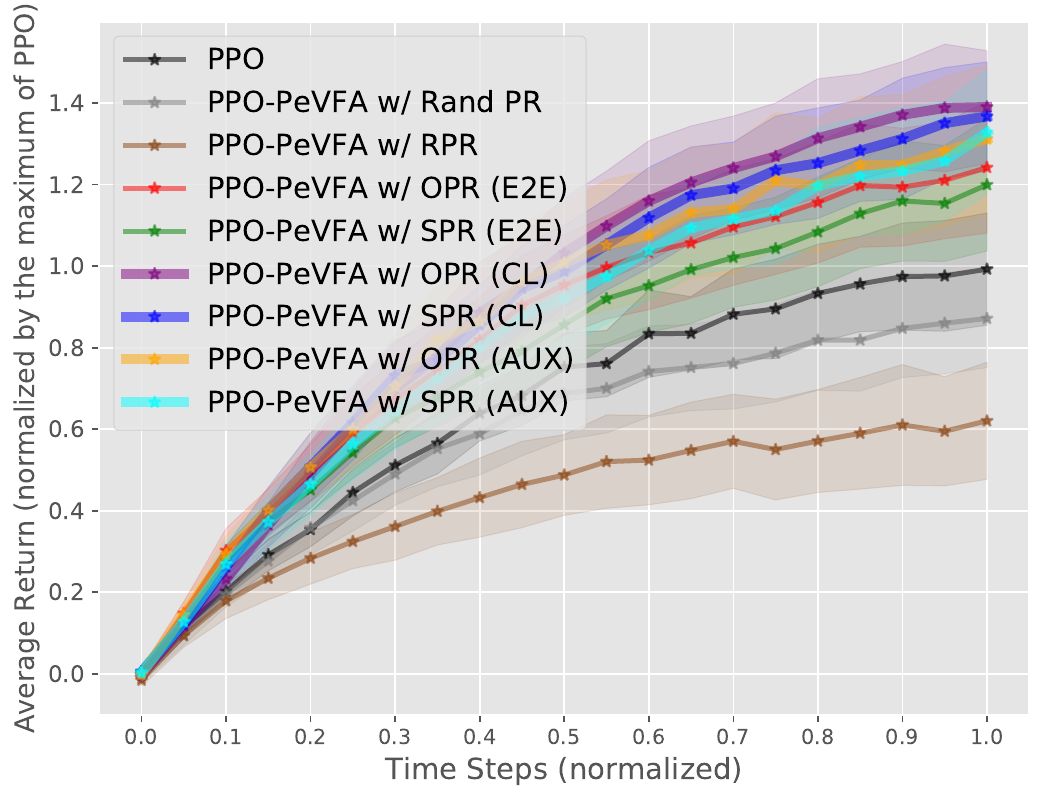}
\label{figure:evaluation_merged}
}
\hspace{-0.1cm}
\subfigure[]{
\includegraphics[width=0.22\textwidth]{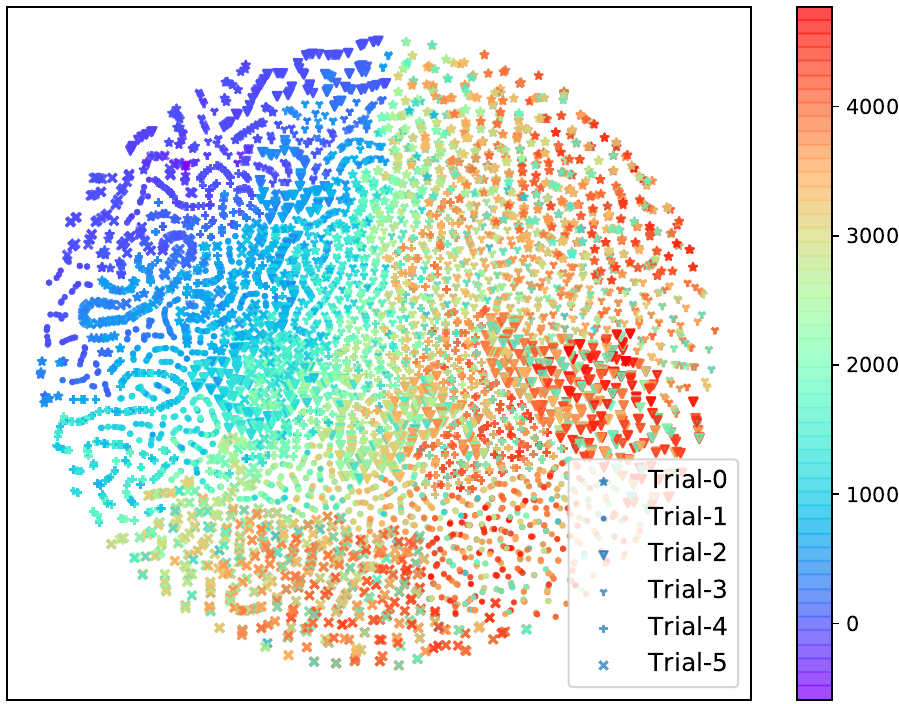}
\label{figure:visual_demo}
}
\vspace{-0.2cm}
\caption{
(a) Normalized averaged returns aggregated over 4 MuJoCo tasks. 
(b)
A t-SNE visualization for representations learned by PPO-PeVFA OPR (E2E) in Ant-v1.
In total, 6k policies from 5 trials (denoted by different markers) are plotted,
which are colored according to average return.
}
\vspace{-0.5cm}
\end{figure}

\textbf{To Question \ref{question:pevfa}.}
From Tab.\ref{table:overall_evaluations},
we can find that both PPO-PeVFA w/ OPR (E2E) and PPO-PeVFA w/ SPR (E2E) outperforms PPO in all 6 tasks,
and achieve over 20\% improvement in Fig.\ref{figure:evaluation_merged}.
This demonstrates the effectiveness of PeVFA.
Moreover, the improvement is further enlarged (to about 40\%) by CL and AUX for both OPR and SPR.
This indicates that the superiority of PeVFA can be further utilized with better policy representation
that offers a more suitable space for value generalization.

\textbf{To Question \ref{question:policy_repre}.}
In Tab.\ref{table:overall_evaluations}, 
consistent degeneration is observed for PPO-PeVFA w/ Ran PR
due to the negative effects on generalization caused by the randomness and disorder of policy representation.
This phenomenon seems to be more severe for PPO-PeVFA w/ RPR due to the complexity of high-dimensional parameter space.
In contrast, the improvement achieved by our proposed PPO-PeVFA variants
shows that 
effective policy representation can be learned from policy parameters (OPR) and state-action pairs (SPR) though value approximation loss (i.e., E2E)
and further improved when additional self-supervised representation learning is involved as CL and AUX.
Overall, OPR slightly outperforms SPR as CL does over AUX.
We hypothesize that it is due to the stochasticity of state-action pairs which serve as inputs of SPR and training samples for AUX.
This reveals the space for future improvement.
In addition, we visualize the learned representation
in Fig.\ref{figure:visual_demo}.
We can observe that policies from different trials are locally continuous and show different modes of embedding trajectories due to random initialization and optimization;
while a global evolvement among trials emerges with respect to policy performance.

\section{Conclusion and Future Work}
\label{sec:discussion}
In this paper, we propose Policy-extended Value Function Approximator (PeVFA) and study value generalization among policies.
We propose a new form of GPI based on PeVFA 
which is potentially preferable to conventional VFA for value approximation.
Moreover, we propose a general framework to learn low-dimensional embedding of RL policy.
Our experiments demonstrate the effectiveness of the generalization characteristic of PeVFA and our proposed policy representation learning methods.

Our work opens up some research directions on value generalization among policies and policy representation.
A possible future study on the theory of value generalization among policies is to consider the interplay between approximation error, policy improvement and local generalization 
during GPI with PeVFA.
Besides, analysis on influence factors of value generalization among policies (e.g., policy representation, architecture of PeVFA) and other utilization of PeVFA are expected.
For better policy representation,
inspirations on OPR may be got from studies on Manifold Hypothesis of neural network;
the selection of more informative state-action pairs for SPR is also worth research. 



\section*{Acknowledgments}
The work is supported by the National Natural Science Foundation of China (Grant Nos.: U1836214),
the New Generation of Artificial Intelligence Science and Technology Major Project of Tianjin under grant: 19ZXZNGX00010.

\small
\bibliography{AAAI2022_PeVFA}

\clearpage

\onecolumn
\appendix

\section*{Appendix}

\section{Supplementary Materials for Theoretical Analysis}
\label{app:proofs}


\subsection{More on Definition \ref{defi:approx_mapping} and \ref{defi:l_continuity}}
\label{disc:defini}

In Definition \ref{defi:approx_mapping}, we use $\mathscr{P}_{\pi}$ for a formal description of value approximation process, 
i.e., the learning process of a parametrized PeVFA $\mathbb{V}_{\theta}$ with parameters $\theta \in \Theta$ to approximate the values of policy $\pi$.
For a usual example, one can consider $\mathscr{P}_{\pi}$ as multiple times of parameter update via gradient descent with respect to $f_{\theta}(\pi) = \|\mathbb{V}_{\theta}(\pi) - V^{\pi}\|$.
Note that $f_{\theta}(\pi)$ can be equivalent to a common value approximation loss function $L(\theta) = \mathbb{E}_{s \sim \rho(s)}\left( \mathbb{V}_{\theta}(s, \pi) - \hat{V}^{\pi}(s) \right)^2$
with some unbiased estimates $\hat{V}^{\pi}$ from experiences stored,
when the same state distribution $\rho(s)$ is considered.
Thus, with sufficient capacity of function approximation and certain number of training,
we can expect a contraction of approximation loss for policy $\pi$ obtained by $\mathscr{P}_{\pi}$.

We use Definition \ref{defi:l_continuity} to characterize the local smoothness of approximation loss $f_{\theta}$ near policy $\pi$ with Lipschitz continuity.
Consider a typical PeVFA $V_{\theta}$ parameterized by an MLP with finite weights, biases and non-linear activation.
Such a $V_{\theta}$ is Lipschitz continuity with a bounded Lipschitz constant,
as it is made up of function transformations that individually have bounded Lipschitz constants, 
e.g., weight matrix $w$ of some layer has bounded Lipschitz constant to be the operator norm of matrix $w$ and ReLU activation has Lipschitz constant of 1.
Further, easily we have for any $\pi$ and $\pi^{\prime}$,
\begin{equation}
\begin{aligned}
    | f_{\theta}(\pi) - f_{\theta}(\pi^{\prime}) | \le \|\mathbb{V}_{\theta}(\pi) - \mathbb{V}_{\theta}(\pi^{\prime}) \| + \|V^{\pi} - V^{\pi^{\prime}} \|.
\end{aligned}
\end{equation}
As mentioned above, $V_{\theta}$ is Lipschitz continuity with a bounded Lipschitz constant;
and the norm of true value vector of two policies is also finite.
Thus, $f_{\theta}$ in this case can also have a bounded Lipschitz constant $L$.

\subsection{Proof of Lemma \ref{lem:gen_bound}}
\label{prf:lem1}
\begin{applemma}
For $\theta \xrightarrow{\mathscr{P}_{\pi_1}} \hat{\theta}$,
if $f_{\hat{\theta}}$ is $\hat{L}$-continuous at $\pi_1$ and $f_{\theta}(\pi_1) \le f_{\theta}(\pi_2)$, we have: 
$f_{\hat{{\theta}}}(\pi_2) \le \gamma f_{\theta}(\pi_2) + \mathcal{M}(\pi_1, \pi_2, \hat{L})$,
where $\mathcal{M}(\pi_1, \pi_2, \hat{L}) = \hat{L} \cdot d(\pi_1,\pi_2)$.
\end{applemma}
\begin{proof}
For the clarity, we also use $f$ and $\hat{f}$ as abbreviations of $f_{\theta}$ and $f_{\hat{\theta}}$ in the following.
Start from the $\hat{L}$-continuity of $\hat{f}(\theta)$ (recall Definition \ref{defi:l_continuity}), we have the upper bound of $\hat{f}(\pi_2)$ below:
\begin{equation}
\label{equation:lips_upper_bound}
    \hat{f}(\pi_2) \le 
    \hat{f}(\pi_1) + \hat{L} \cdot d(\pi,\pi^{\prime}).
\end{equation}
The second term in Eq.\ref{equation:lips_upper_bound} is decided by the two policies we considered and a Lipschitz constant $\hat{L}$.
Moreover, the constant $\hat{L}$ (i.e., locality property) is related to the parameters $\hat{\theta}$ of PeVFA.
In general, we denote the above term as $\mathcal{M}(\pi_1, \pi_2, \hat{L})$ called \textit{locality margin}.
The locality margin $\mathcal{M}(\pi_1, \pi_2, \hat{L})$ can have different forms that depends on the specific locality property,
for examples:
\begin{equation*}
    \mathcal{M}(\pi_1, \pi_2, \hat{L}) =
    \left\{
        \begin{aligned}
        & \hat{L} \cdot d(\pi_1, \pi_2) 
        & \text{\ding{172}}
        \\
        & \langle \hat{f}^{\prime}(\pi_1), \pi_2 - \pi_1 \rangle + \frac{1}{2} \hat{L} \cdot d(\pi_1, \pi_2)^2 
        & \text{\ding{173}}
        \\
        & 
        \langle \hat{f}^{\prime}(\pi_1), \pi_2 - \pi_1 \rangle
        + \frac{1}{2} \langle \hat{f}^{\prime \prime}(\pi_1) (\pi_2 - \pi_1), \pi_2 - \pi_1 \rangle + \frac{1}{6} \hat{L} \cdot d(\pi_1, \pi_2)^3
        & \text{\ding{174}}
        \end{aligned}
    \right.
\label{equation:lip_conds}
\end{equation*}
\ding{172}, \ding{173}, \ding{174} correspond to Lipschitz Continuous, Lipschitz Gradients and Lipschitz Hessian \citep{NesterovP06Cubic}, which are conisdered in previous works on generalization studies \citep{KeskarMNST17LargeBatch,Wang18GenPro}.

Further, apply the Definition \ref{defi:approx_mapping} and consider the case $f(\pi_1) \le f(\pi_2)$,
Eq.\ref{equation:lips_upper_bound} can be further transformed as follows:
\begin{equation}
\begin{aligned}
    \hat{f}(\pi_2) & \le 
    \hat{f}(\pi_1) + \mathcal{M}(\pi_1, \pi_2, \hat{L}) \\
    & \le \gamma f(\pi_1) + \mathcal{M}(\pi_1, \pi_2, \hat{L}) \\
    & \le \underbrace{\gamma f(\pi_2)}_{\text{generalized contraction}} + \ \ \underbrace{\mathcal{M}(\pi_1, \pi_2, \hat{L})}_{\text{locality margin}},
\end{aligned}
\label{equation:derivation_gen_upper_bound}
\end{equation}
which yields the generalization upper bound in Lemma \ref{lem:gen_bound}.
We note the first term of RHS of Eq.\ref{equation:derivation_gen_upper_bound} as generalized contraction term since it is from the contraction on $f(\pi_1)$ caused by the value approximation operator $\mathscr{P}_{\pi_1}$,
and the second term as locality margin since it is determined by specific local property.
\end{proof}
\begin{remark}
Since value approximation is only performed for $\pi_1$,
the condition $f_{\theta}(\pi_1) \le f_{\theta}(\pi_2)$ can usually exist after a certain number of training;
in turn,
the complementary case
$f_{\theta}(\pi_1) > f_{\theta}(\pi_2)$ 
is acceptable since the unoptimized approximation loss is already lower than the optimized one.
\end{remark}

\subsection{Proof of Corollary \ref{cor:gen_cond}}
\label{prf:cor1}
\begin{appCorollary}
$\mathscr{P}_{\pi_1}$
is $\gamma_{g}$-contraction ($\gamma_{g} \in [0,1)$) for $\pi_2$ when $f_{\theta}(\pi_2) > \frac{\hat{L} \cdot d(\pi_1, \pi_2)}{1 - \gamma}$.
\end{appCorollary}
\begin{proof}
Following Lemma \ref{lem:gen_bound}, consider Lipschitz continuity for a concrete locality property of $f_{\hat{\theta}}$, we have,
\begin{equation}
    \hat{f}(\pi_2) \le 
    \gamma f(\pi_2) + \hat{L} \cdot d(\pi_1, \pi_2).
\label{equation:lc_upper_bound1}
\end{equation}
Then we get the contraction condition of value generalization on $\pi_2$ in Corollary \ref{cor:gen_cond},
by letting the RHS of Eq.\ref{equation:lc_upper_bound1} be smaller than $f(\pi_2)$:
\begin{equation}
\begin{aligned}
    \gamma f(\pi_2) + \hat{L} \cdot d(\pi_1, \pi_2) & < f(\pi_2) \\
    (1 - \gamma) f(\pi_2) & > \hat{L} \cdot d(\pi_1, \pi_2) \\
    f(\pi_2) & > \frac{\hat{L} \cdot d(\pi_1, \pi_2)}{1 - \gamma} \ge 0.
\end{aligned}
\label{equation:derivation_gen_con}
\end{equation}
\end{proof}
\begin{remark}
\label{remark:gen_con_intuition}
From the generalization contraction condition provided in Corollary \ref{cor:gen_cond}, we can find that:
as \romannumeral1. $\gamma \rightarrow 0$, or \romannumeral2. $d(\pi_1, \pi_2) \rightarrow 0$, or \romannumeral3. $\hat{L} \rightarrow 0$, the contraction condition is easier to achieve (or the contraction gets tighter),
i.e., the generalization on unlearned policy $\pi_2$ is better.
\end{remark}
In another word, the tighter the contraction on learned policy $\pi_1$ is, the closer the two policies are, the smoother the approximation loss function $\hat{f}$ is, the generalization on unlearned policy $\pi_2$ is better.

Corollary \ref{cor:gen_cond} provides the generalization contraction condition on $f(\pi_2)$, under the assumptions that $\mathscr{P}_{\pi_1}$ is $\gamma$-contraction and $f(\pi_1) < f(\pi_2)$ (as in Lemma \ref{lem:gen_bound}). 
In below, we discuss a more general condition for generalization contraction on $f(\pi_2)$ which indicates more possible cases:
\begin{Corollary}
\label{cor:gen_cond_2}
For $\theta \xrightarrow{\mathscr{P}_{\pi_1}} \hat{\theta}$ and $f_{\hat{\theta}}$ is $\hat{L}$-continuous at $\pi_1$,
when $f(\pi_2) - \gamma f(\pi_1) > \hat{L} \cdot d(\pi_1, \pi_2)$, 
we have that $\mathscr{P}_{\pi_1}$
is also a $\gamma_{g}$-contraction for $\pi_2$,
i.e., $f_{\hat{\theta}}(\pi_2) \le \gamma_g f_{\theta}(\pi_2)$ with $\gamma_{g} \in [0,1)$.
\vspace{-0.05cm}
\end{Corollary}
\begin{proof}
From Eq.\ref{equation:derivation_gen_upper_bound}, we have
$\hat{f}(\pi_2) \le \gamma f(\pi_1) + \hat{L} \cdot d(\pi_1, \pi_2)$.
To yield the generalization contraction on $f(\pi_2)$, is to let
\begin{equation}
    \hat{f}(\pi_2) \le \gamma f(\pi_1) + \hat{L} \cdot d(\pi_1, \pi_2) < f(\pi_2),
\end{equation}
that is to let, 
\begin{equation}
\label{equation:gen_cond_2_condition}
    f(\pi_2) - \gamma f(\pi_1) > \hat{L} \cdot d(\pi_1, \pi_2).
\end{equation}
\end{proof}

Since $d(\pi_1, \pi_2)$ is constant in the two-policy case considered, 
the condition in Corollary \ref{cor:gen_cond_2} is associated to the value approximation losses on $\pi_1$ and $\pi_2$ before applying the value approximation operator $\mathscr{P}_{\pi}$, 
as well as the $\hat{L}$-continuity of $\hat{\theta}$ after applying $\mathscr{P}_{\pi}$.
We can find similar conclusions as mentioned in Remark \ref{remark:gen_con_intuition}.
However, Corollary \ref{cor:gen_cond_2} indicates some more cases that the condition of generalization contraction can be satisfied.
For example, it can happen in the complementary cases as we assumed in Lemma \ref{lem:gen_bound}, i.e., 1) when $f(\pi_1) > f(\pi_2)$, 
or 2) $\mathscr{P}_{\pi}$ is not a $\gamma$-contraction on $f(\pi_1)$.

\subsection{Proof of Lemma \ref{lem:recursive_relation}}
\label{prf:lem2}
\begin{applemma}
For $\theta_{-1} \xrightarrow{\mathscr{P}_{\pi_0}} \theta_{0} \xrightarrow{\mathscr{P}_{\pi_1}} \theta_{1} \xrightarrow{\mathscr{P}_{\pi_2}} \dots$ with $\gamma_t$ for each $\mathscr{P}_{\pi_t}$,
if $f_{\theta_t}$ is $L_{t}$-continuous at $\pi_t$ for any $t \ge 0$,
we have
$f_{\theta_t}(\pi_{t+1}) \le \gamma_t f_{\theta_{t-1}}(\pi_{t}) + \mathcal{M}_{t}$,
where $\mathcal{M}_{t} = L_{t} \cdot d(\pi_t, \pi_{t+1})$.
\end{applemma}
\begin{proof}
Consider any $t \ge 0$ and $\theta_{t-1} \xrightarrow{\mathscr{P}_{\pi_t}} \theta_{t}$,
due to $f_{\theta_t}$ is $L_{t}$-continuous at $\pi_{t}$,
we have,
\begin{equation}
    f_{\theta_{t}}(\pi_{t+1}) \le f_{\theta_{t}}(\pi_{t}) + L_{t} \cdot d(\pi_{t}, \pi_{t+1}),
\end{equation}
then due to the definition of the value approximation process $\mathscr{P}_{\pi_t}$,
\begin{equation}
\begin{aligned}
    f_{\theta_{t}}(\pi_{t+1}) & \le \gamma_t f_{\theta_{t-1}}(\pi_{t}) + L_{t} \cdot d(\pi_{t}, \pi_{t+1}), \\
    & = \gamma_t f_{\theta_{t-1}}(\pi_{t}) + \mathcal{M}_{t},
\end{aligned}
\end{equation}
where $\mathcal{M}_{t} = L_{t} \cdot d(\pi_{t}, \pi_{t+1})$.
\end{proof}
Intuitively, 
such a recursive relation between the generalized approximation loss of two consecutive steps, 
i.e., $f_{\theta_{t-1}}(\pi_{t})$ and $f_{\theta_{t}}(\pi_{t+1})$, 
are chained by the assumed continuity of the loss function $f_{\theta_t}$
and the definition of value approximation process.

\subsection{Proof of Corollary \ref{cor:local_gen_induction}}
\label{prf:cor2}
\begin{appCorollary}
By induction, we have
$f_{\theta_t}(\pi_{t+1}) \le \prod_{i=0}^t \gamma_t f_{\theta_{-1}}(\pi_{0}) + \sum_{i=0}^{t-1} \prod_{j=i+1}^{t} \gamma_j \mathcal{M}_{i} + \mathcal{M}_{t}$ .
\end{appCorollary}
\begin{proof}
Consider the consecutive value approximation process $\theta_{-1} \xrightarrow{\mathscr{P}_{\pi_0}} \theta_{0} \xrightarrow{\mathscr{P}_{\pi_1}} \dots 
\xrightarrow{\mathscr{P}_{\pi_{t-1}}} \theta_{t-1}
\xrightarrow{\mathscr{P}_{\pi_t}} \theta_{t}
\xrightarrow{\mathscr{P}_{\pi_{t+1}}} \dots
$,
following the recursive relation in Lemma \ref{lem:recursive_relation},
we have the inequality below by induction,
\begin{equation}
\begin{aligned}
    f_{\theta_{t}}(\pi_{t+1}) & \le \gamma_t f_{\theta_{t-1}}(\pi_{t}) + \mathcal{M}_{t}, \\
    & \le \dots \\
    & \le \gamma_t \left(\gamma_{t-1} 
    \left( \dots 
    \left( 
    \gamma_0 f_{\theta_{-1}}(\pi_{0}) + \mathcal{M}_{0} 
    \right) 
    \dots \right) 
    \mathcal{M}_{t-1} \right) + \mathcal{M}_{t}, \\
    & = \underbrace{\left( \prod_{i=0}^t \gamma_t \right) f_{\theta_{-1}}(\pi_{0})}_{\text{\ding{182}}}
    + \underbrace{\sum_{i=0}^{t-1} 
    \left(  \prod_{j=i+1}^{t} \gamma_j \right) 
    \mathcal{M}_{i} + \mathcal{M}_{t}}_{\text{\ding{183}}}.
\end{aligned}
\end{equation}
where $\mathcal{M}_{t} = L_{t} \cdot d(\pi_{t}, \pi_{t+1})$.
We use \ding{182} to denote the term for accumulated generalized contraction of initial approximation loss and use \ding{183} to denote the term for accumulated locality margin.
\end{proof}

Towards the infinity case i.e., $t \rightarrow \infty$, if we assume that \textit{(i)} $\max_{t} d(\pi_{t}, \pi_{t+1}) < \infty$ 
and \textit{(ii)} $\prod_{k=h_1}^{h_2}\gamma_{k} = O(\frac{1}{{(h_2-h_1+1)}^{1+\varepsilon}})$, $\forall  0 < h_1 \leq  h_2$ with  some $\varepsilon > 0$, then $\lim_{t\rightarrow \infty} f_{\theta_{t}}(\pi_{t+1}) < \infty$. 
That is because
the sequence
$\{\mathcal{M}_{i}\}_{i=0}^{t}$ has a public upper bound $\mathcal{M}_{\text{max}} = L_{\text{max}} \cdot \max_{t} d(\pi_{t}, \pi_{t+1})$ where $L_{\text{max}}$ denotes the upper bound of Lipschitz constant (recall the discussion in Appendix \ref{disc:defini}), 
and by (ii) $\sum_{i=0}^{t-1} \left(  \prod_{j=i+1}^{t} \gamma_j \right) = O(\sum_{i=0}^{t-1} \frac{1}{{(t-i+1)}^{1+\varepsilon}}) < \infty $.

Note that we consider a really loose bound in the infinity case above with $\mathcal{M}_{\text{max}}$, therefore the condition \textit{(ii)} may be unnecessarily strict when the dynamics of $L_{t}$ and $d(\pi_t, \pi_{t+1})$ are considered.
Intuitively, the evolvement of $L_{t}$ during learning process is related to function family and optimization method of $\theta_t$;
and for $d(\pi_t, \pi_{t+1})$, this is related to value approximation error ($f_{\theta_{t}(\pi_t}$) and policy improvement method (i.e., how $\pi_t$ is improved to be $\pi_{t+1}$).
We leave these further analysis for future work.

\subsection{Proof of Theorem \ref{thm:closer_target}}
\label{prf:thm1}
\begin{theorem}
During $\theta_{-1} \xrightarrow{\mathscr{P}_{\pi_0}} \theta_{0} \xrightarrow{\mathscr{P}_{\pi_1}} \theta_{1} \xrightarrow{\mathscr{P}_{\pi_2}} \dots$,
for any $t \ge 0$, if $f_{\theta_t}(\pi_t) + f_{\theta_t}(\pi_{t+1}) \le 
\|V^{\pi_t} - V^{\pi_{t+1}}\|$,
then $f_{\theta_t}(\pi_{t+1}) \le \|\mathbb{V}_{\theta_{t}}(\pi_t) - V^{\pi_{t+1}}\| $.
\end{theorem}
\begin{proof}
By the condition in Theorem \ref{thm:closer_target}, we have
\begin{equation}
\label{equation:closer_target_precond}
\begin{aligned}
   &~~~ f_{\theta_t}(\pi_t) + f_{\theta_t}(\pi_{t+1}) 
     \le \|V^{\pi_{t}} - V^{\pi_{t+1}}\| \\
    & \leq     \|\mathbb{V}_{\theta_{t}}(\pi_t) - V^{\pi_{t}}\| + \|\mathbb{V}_{\theta_{t}}(\pi_t)  - V^{\pi_{t+1}}\| 
     =     f_{\theta_t}(\pi_t) + \|\mathbb{V}_{\theta_{t}}(\pi_t) - V^{\pi_{t+1}}\| ,
\end{aligned}
\end{equation}
where the second inequality comes from \emph{Triangle Inequality}. Then it is straightforward that
\begin{equation}
\label{equation:derivation_closer_target}
\begin{aligned}
    \underbrace{f_{\theta_t}(\pi_{t+1})}_{\text{generalizated VAD with PeVFA}} 
    \le \underbrace{\|\mathbb{V}_{\theta_{t}}(\pi_t) - V^{\pi_{t+1}}\|}_{\text{conventional VAD}},
\end{aligned}
\end{equation}
which means that with local generalization of values for successive policy $\pi_{t+1}$,
the value approximation distance (VAD) can be closer in contrast to the conventional one (RHS of Eq.\ref{equation:derivation_closer_target}).
\end{proof}

In practice,
we consider that it is also possible for farther distance to exist, e.g., the condition in above Theorem \ref{thm:closer_target} is not satisfied.
Moreover, under nonlinear function approximation,
it is not necessary that a closer approximation distance (induced by Theorem \ref{thm:closer_target}) ensures easier approximation or optimization process.
This can be associated to many factors, e.g., the underlying function space, the optimization landscape, the learning algorithm used and etc.
In this paper, we provide a condition for potentially beneficial local generalization 
and we resort to empirical examination as shown in Sec. \ref{subsec:demonstrative_exp}.
Further investigation on the interplay between value generalization and policy learning especially under nonlinear function approximation is planned for future work.

\section{Details of Empirical Evidence of Two Kinds of Generalization}
\label{app:demon_details}

\subsection{Global Generalization in 2D Point Walker}
\label{app:2d_walker}

Global generalization denotes the generalization scenario that values can generalize to unlearned policies ($\pi^{\prime} \in \Pi_{1}$) from already learned policies ($\pi \in \Pi_{0}$).
We conduct the following experiments to demonstrate global generalization in a 2D continuous Point Walker environment with synthetic simple policies.

\paragraph{Environment.}
We consider a point walker on a 2D continuous plane with:
\begin{itemize}
    \item state: $(x, y, \sin(\theta), \cos(\theta), \cos(x), \cos(y))$, where $\theta$ is the angle of the polar coordinates,
    \item action: 2D displacement, $a \in \mathbb{R}^{2}_{[-1,1]}$,
    \item a deterministic transition function that describes the locomotion of the point walker, depending on the current position and displacement issued by agent,
    i.e., $\langle x^{\prime}, y^{\prime} \rangle = \langle x, y \rangle + a$,
    \item a reward function: $r_t = \frac{u_{t+1} - u_{t}}{10}$ with utility $u_t = x_t^2 - y_t^2$, as illustrated in Fig.\ref{figure:utility_heatmap}. 
\end{itemize}

\begin{figure}[ht]
\centering
\hspace{-0.8cm}
\subfigure[Utility function heat map]{
\includegraphics[width=0.43\textwidth]{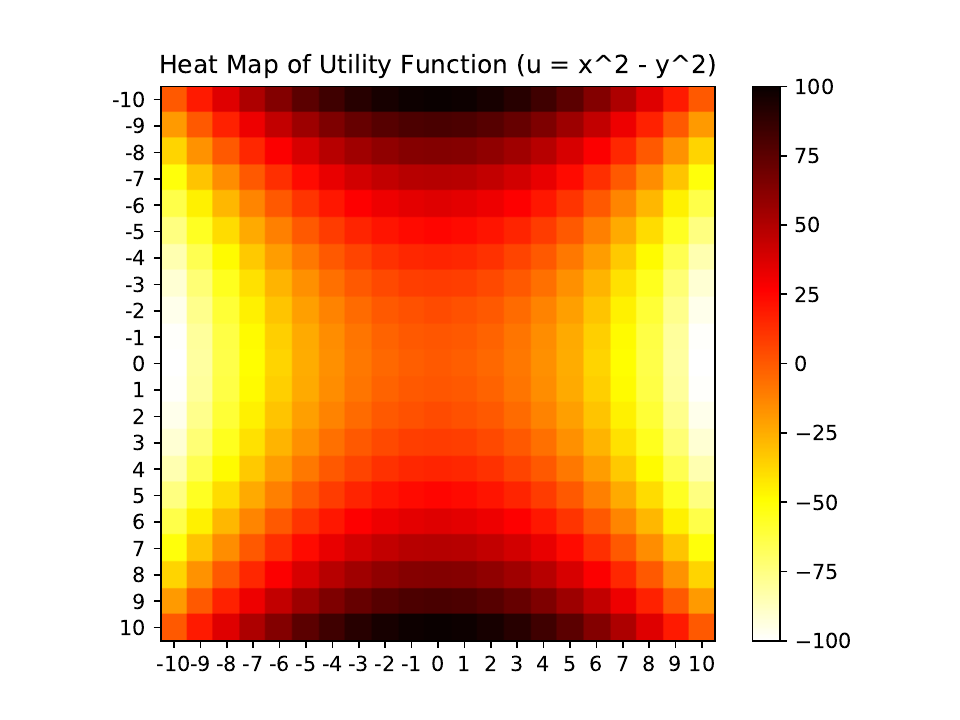}
\label{figure:utility_heatmap}
}
\hspace{-0.7cm}
\subfigure[Examples of synthetic policy population]{
\includegraphics[width=0.62\textwidth]{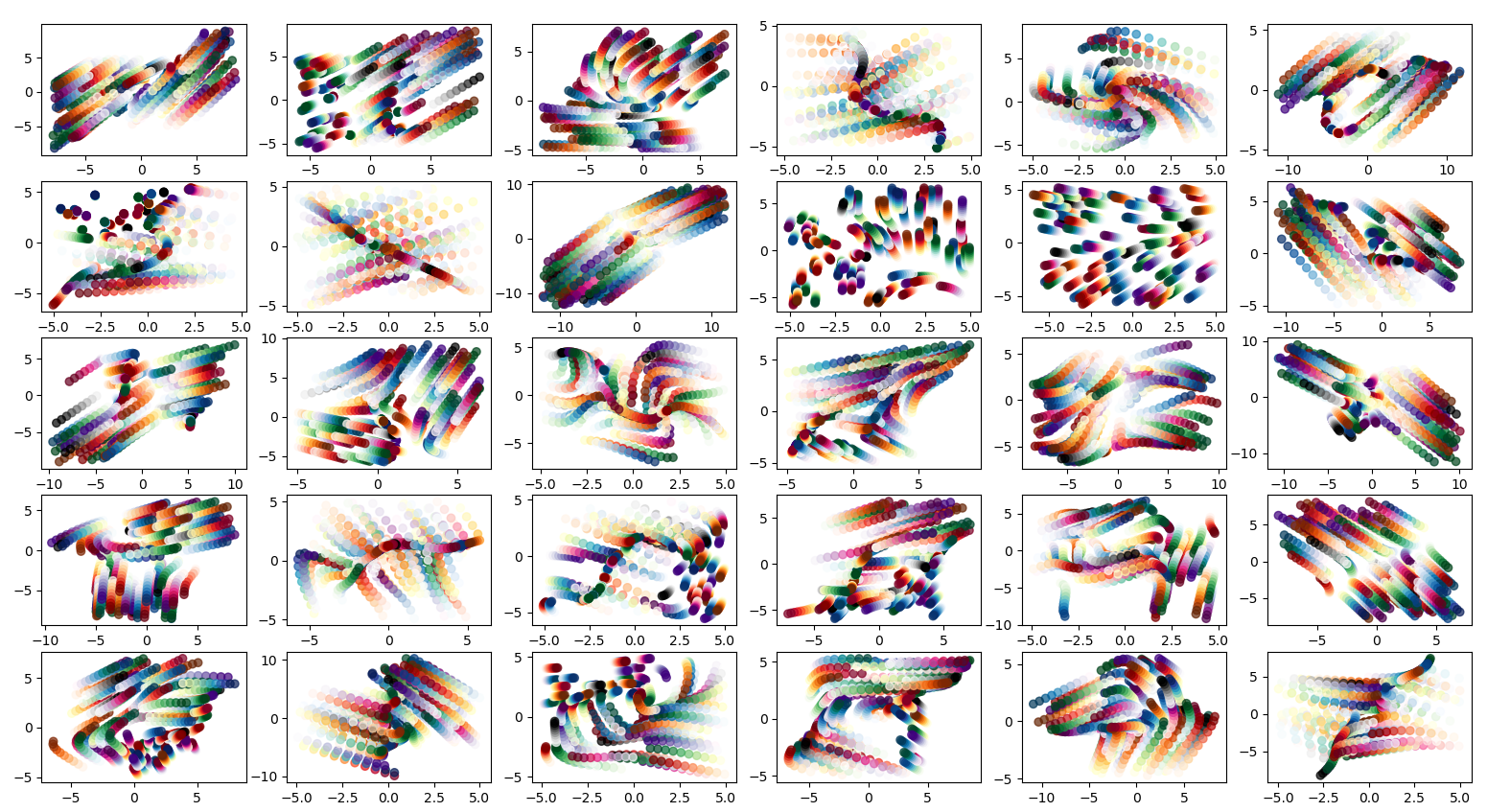}
\label{figure:syn_popu}
}

\caption{
2D Point Walker. 
(a) The heat map of the utility function of the 2D plane. 
The darker regions have higher utilities.
(b) Demonstrative illustrations of trajectories generated by 30 synthetic policies, showing diverse behaviors and patterns.
Each subplot illustrates the trajectories generated in 50 episodes by a randomly synthetic policy, with different colors as separation.
For each trajectory (the same color in one subplot), 
transparency represents the dynamics along timesteps, i.e., fully transparent and non-transparent denotes the positions at first and last timesteps.
}
\label{figure:2d_walker_env}
\end{figure}

\paragraph{Synthetic Policy.}
We build the policy sets $\Pi = \Pi_0 \cup \Pi_1$ and $\Pi_0 \cap \Pi_1 = \emptyset$ with synthetic policies.
Each synthetic policy is a 2-layer \emph{tanh}-activated neural policy network with 2 nodes for each layer.
The weights are initialized by sampling from a uniform distribution $U(-1,1)$ and the biases are initialized by $U(-0.2,0.2)$.
Each policy is deterministic, taking an environmental state as input and outputting a displacement in the plane.
We find that the synthetic population generated by such a simple way can show diverse behaviors.
Fig.\ref{figure:syn_popu} shows the motion patterns of an example of such a synthetic population.
Note that the synthetic policies are not trained in this experiment.


\paragraph{Policy Dataset.}
We rollout each policy in environment to collect trajectories $\mathcal{T}=\{\tau_i\}_{i=0}^k$.
For such small synthetic policies, it is convenient to obtain policy representation.
Here we use the concatenation of all weights and biases of the policy network (26 in total) as representation $\chi_{\pi}$ for each policy $\pi$,
called \emph{raw} policy representation (RPR).
Therefore, combined with the trajectories collected, we obtain the policy dataset, i.e., $\{(\chi_{\pi_j}, \mathcal{T}_{\pi_j})\}_{j=0}^n$.
In total, 20k policies are synthesized in our experiments and we collected 50 trajectories with horizon 10 for each policy.

We separate the synthetic policies into training set (i.e., unknown policies $\Pi_0$) and testing set (i.e., unseen policies $\Pi_1$)
in a proportion of $8:2$.
We set a PeVFA network $\mathbb{V}_{\theta}(s, \chi_{\pi})$ to approximate the values of training policies (i.e., $\pi \in \Pi_0$),
and then conduct evaluation on testing policies (i.e., $\pi \in \Pi_1$).
We use Monte Carlo return \citep{SuttonB98} of collected trajectories as approximation target (true value of policies) in this experiment.
The network architecture of $\mathbb{V}_{\theta}(s, \chi_{\pi})$ is illustrated in Fig.\ref{figure:pevfa_architecture}.
The learning rate is 0.005, batch size is 256.
K-fold validation is performed through shuffling training and testing sets.

Fig.\ref{figure:global_gen_loss_curves} shows the curves of training loss and testing loss.
The average losses on training and testing set are 1.782 and 2.071 over 6 trials.
Fig.\ref{figure:global_gen} plots the value predictions for policies from training and testing set (100 for each). 
This demonstrates that a PeVFA trained with data collected by training set $\Pi_0$ achieves reasonable value prediction of unseen testing policies in $\Pi_1$.
Our results indicate that
value generalization can exist among policy space with a properly trained PeVFA.
RPR can also be one alternative of policy representation when policy network is of small scale.

\begin{figure}[ht]
\centering
\subfigure[Structure of PeVFA network]{
\includegraphics[width=0.3\textwidth]{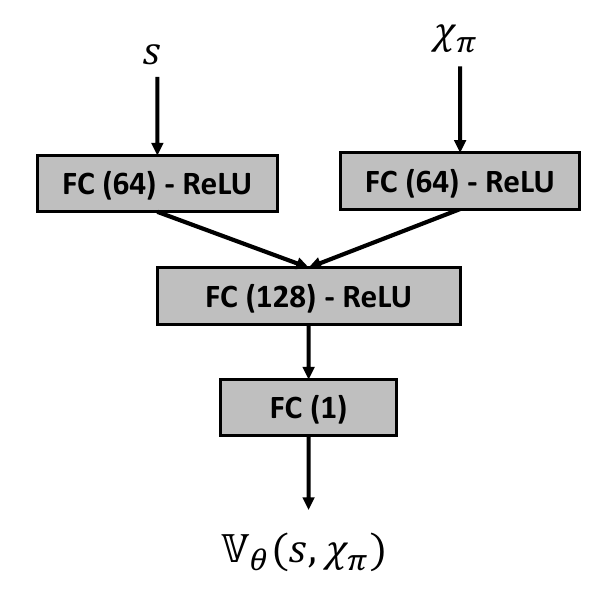}
\label{figure:pevfa_architecture}
}
\hspace{0.5cm}
\subfigure[Global Generalization on 2D Point Walker]{
\includegraphics[width=0.4\textwidth]{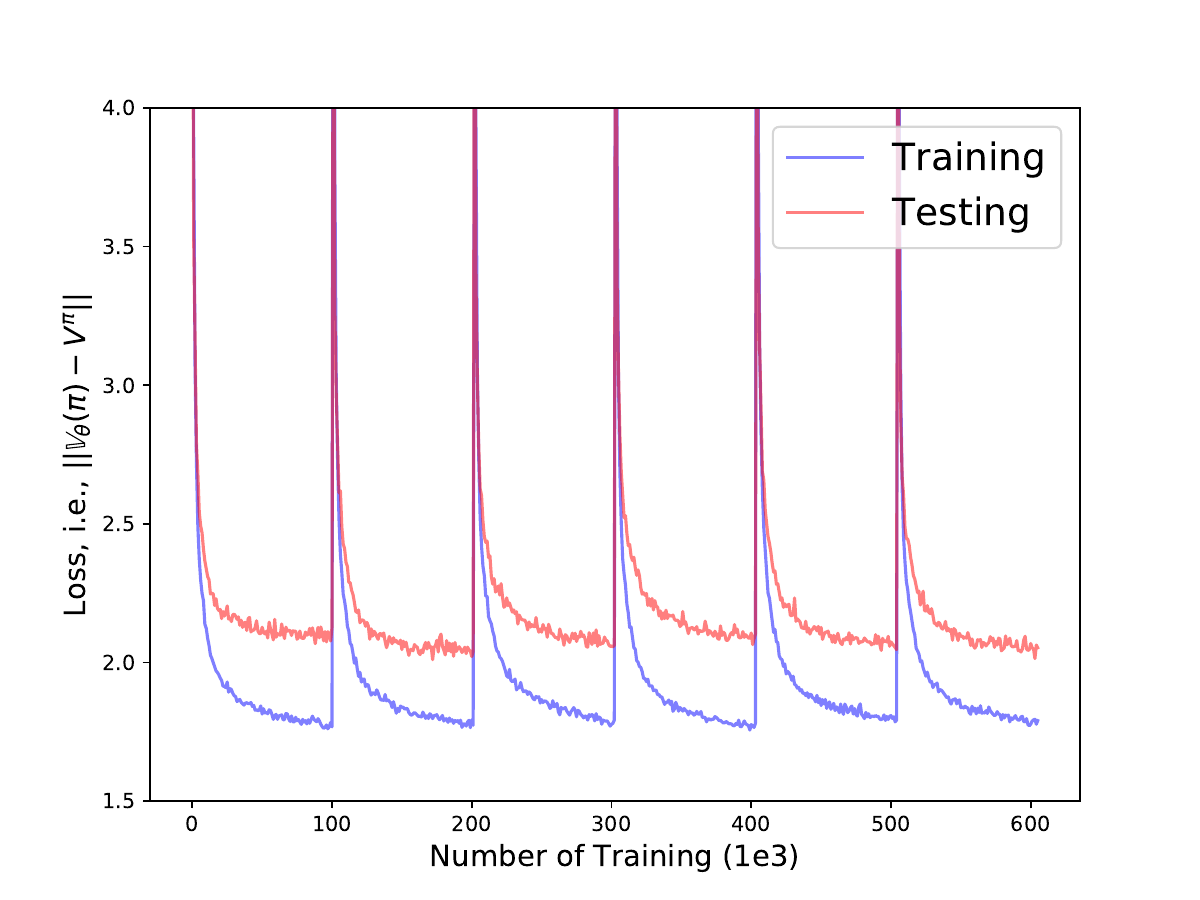}
\label{figure:global_gen_loss_curves}
}
\caption{
Global generalization of PeVFA on 2D Point Walker. 
(a) An illustration of architecture of PeVFA network.
FC is abbreviation for Fully-connected layer. 
(b) Training and testing losses.
Data shuffle sand network re-initialization are performed per 100 steps, i.e., 1e5 training times.
}
\end{figure}

\begin{figure}
\centering
\hspace{-0.2cm}
\subfigure[InvertedPendulum-v1]{
\includegraphics[width=0.32\textwidth]{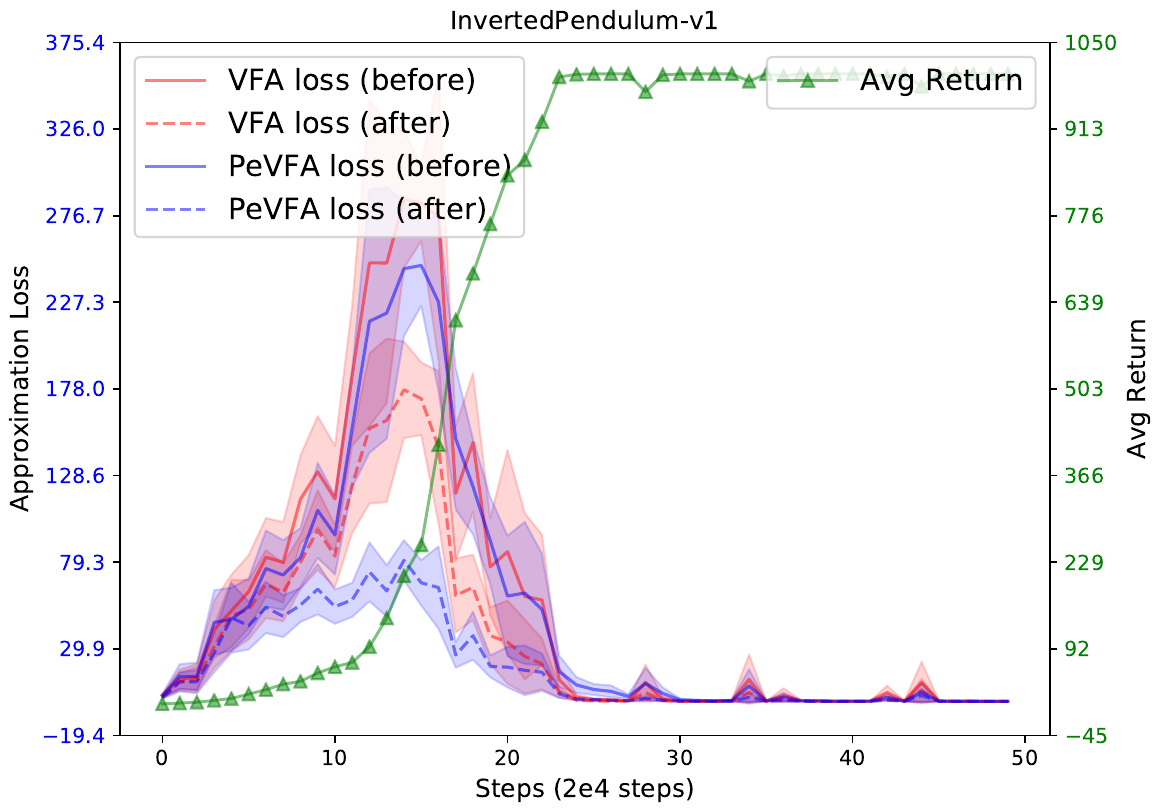}
}
\hspace{-0.35cm}
\subfigure[InvertedDoublePendulum-v1]{
\includegraphics[width=0.32\textwidth]{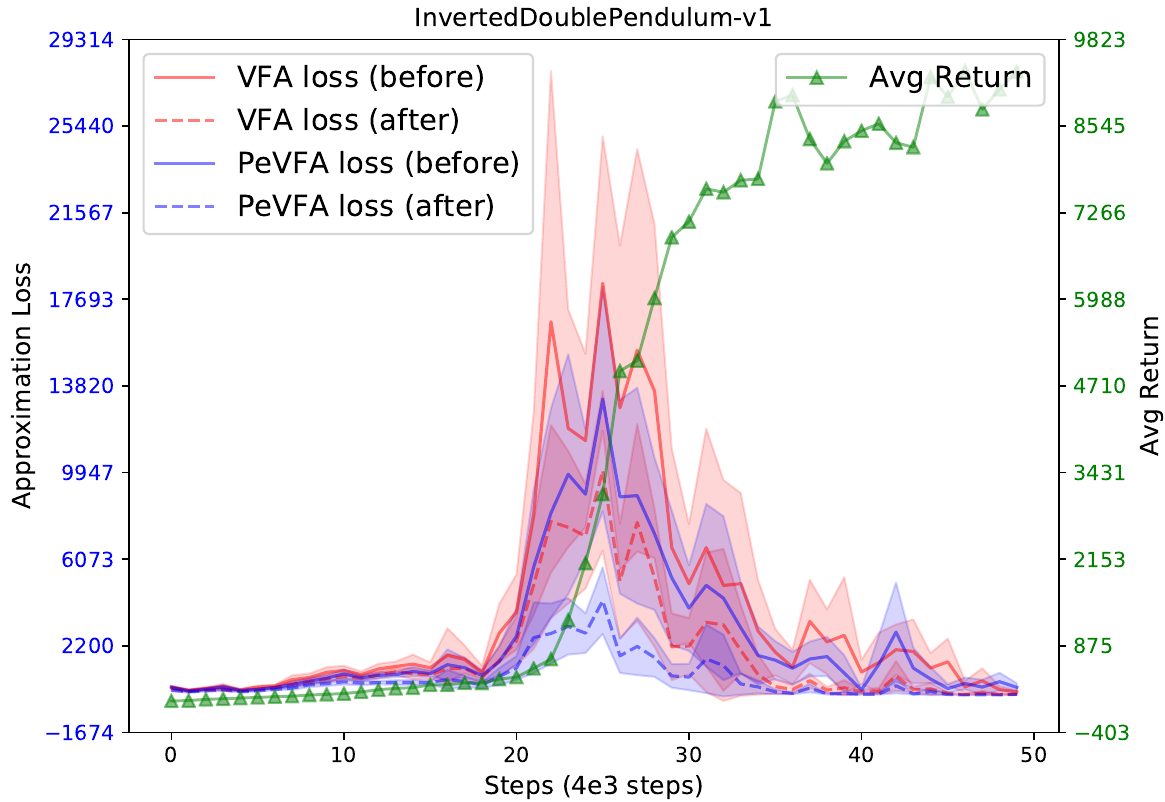}
}
\hspace{-0.35cm}
\subfigure[LunarLanderContinuous-v2]{
\includegraphics[width=0.32\textwidth]{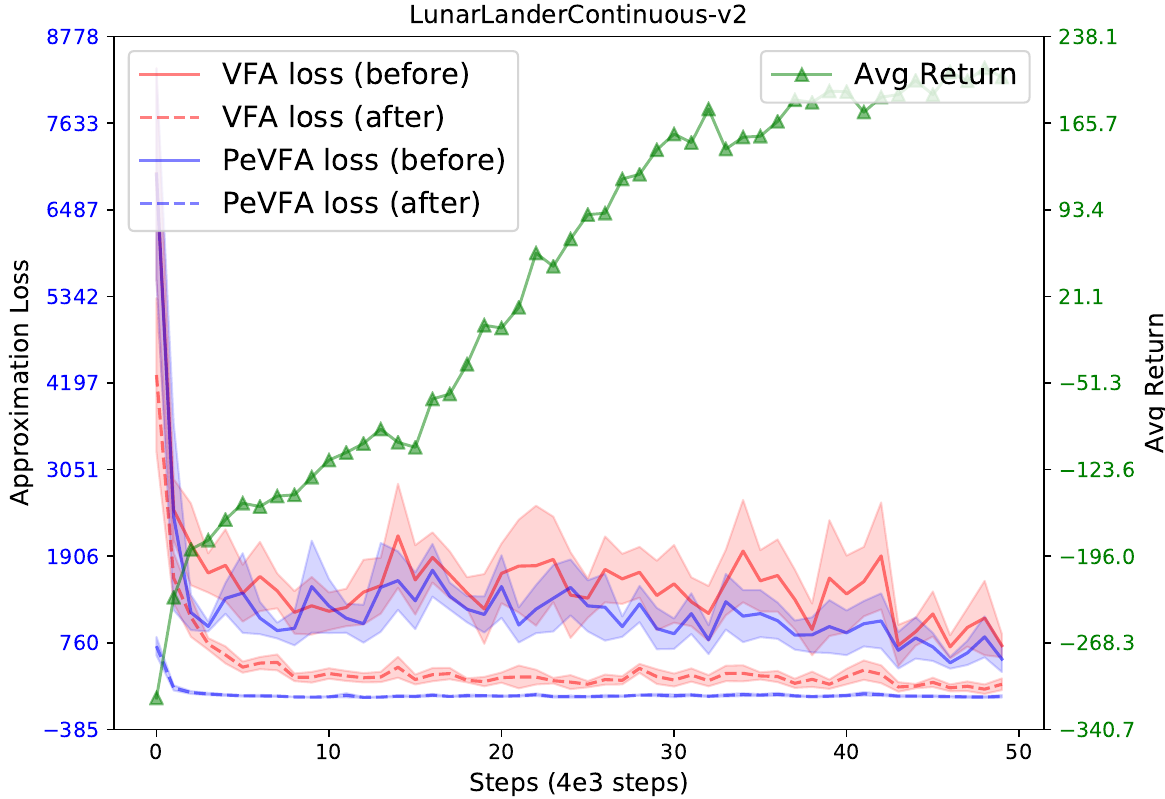}
}
\subfigure[HalfCheetah-v1]{
\includegraphics[width=0.32\textwidth]{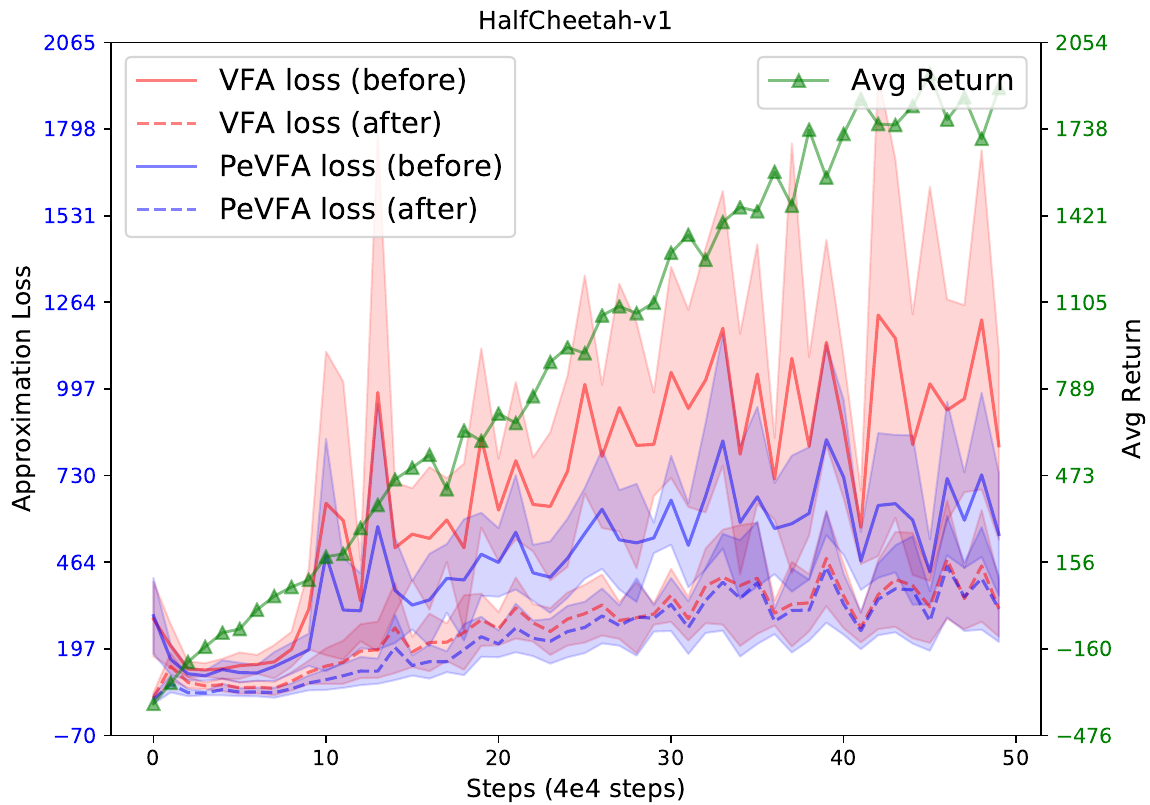}
}
\hspace{-0.35cm}
\subfigure[Hopper-v1]{
\includegraphics[width=0.32\textwidth]{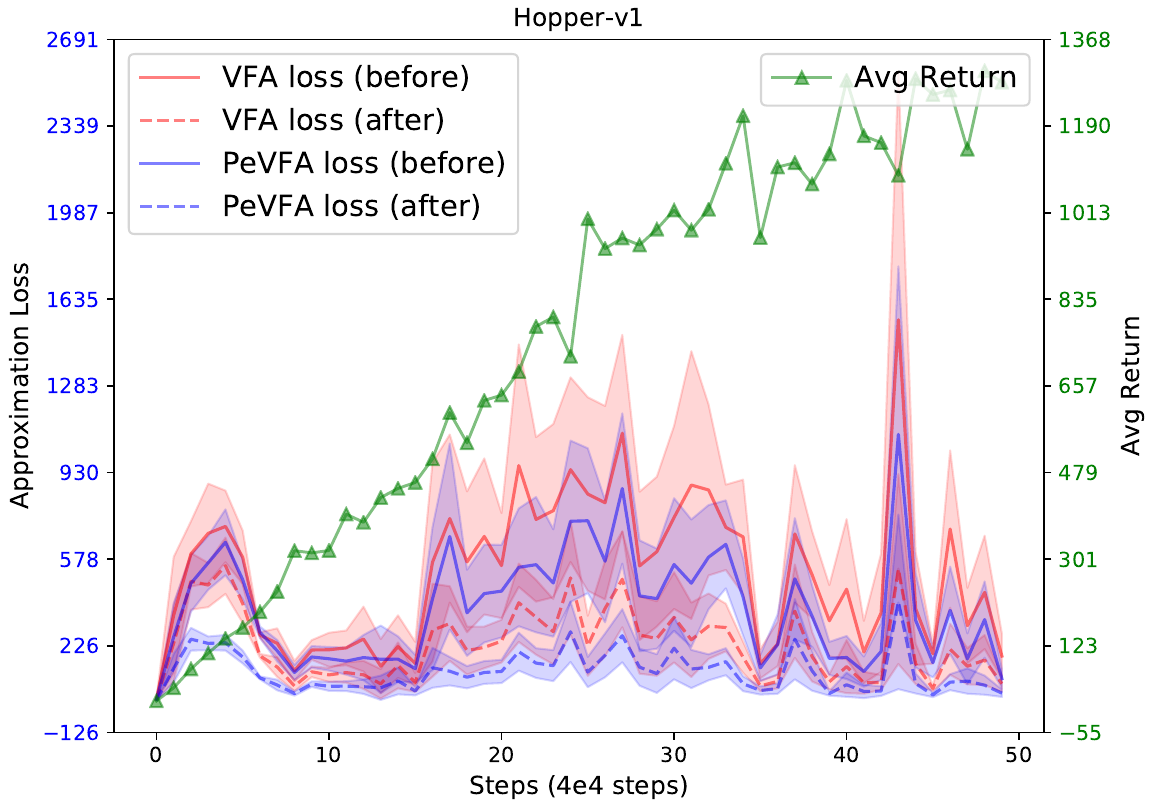}
}
\hspace{-0.35cm}
\subfigure[Walker2d-v1]{
\includegraphics[width=0.32\textwidth]{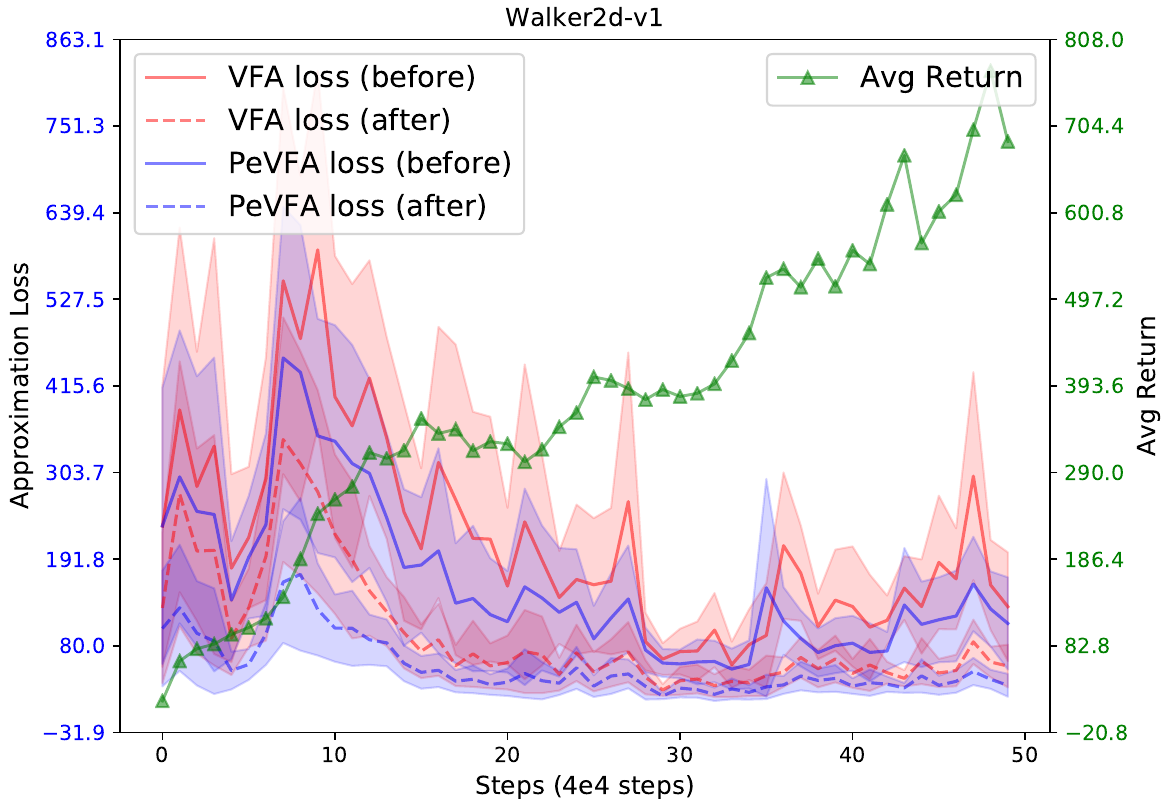}
}
\hspace{-0.35cm}
\subfigure[Ant-v1]{
\includegraphics[width=0.32\textwidth]{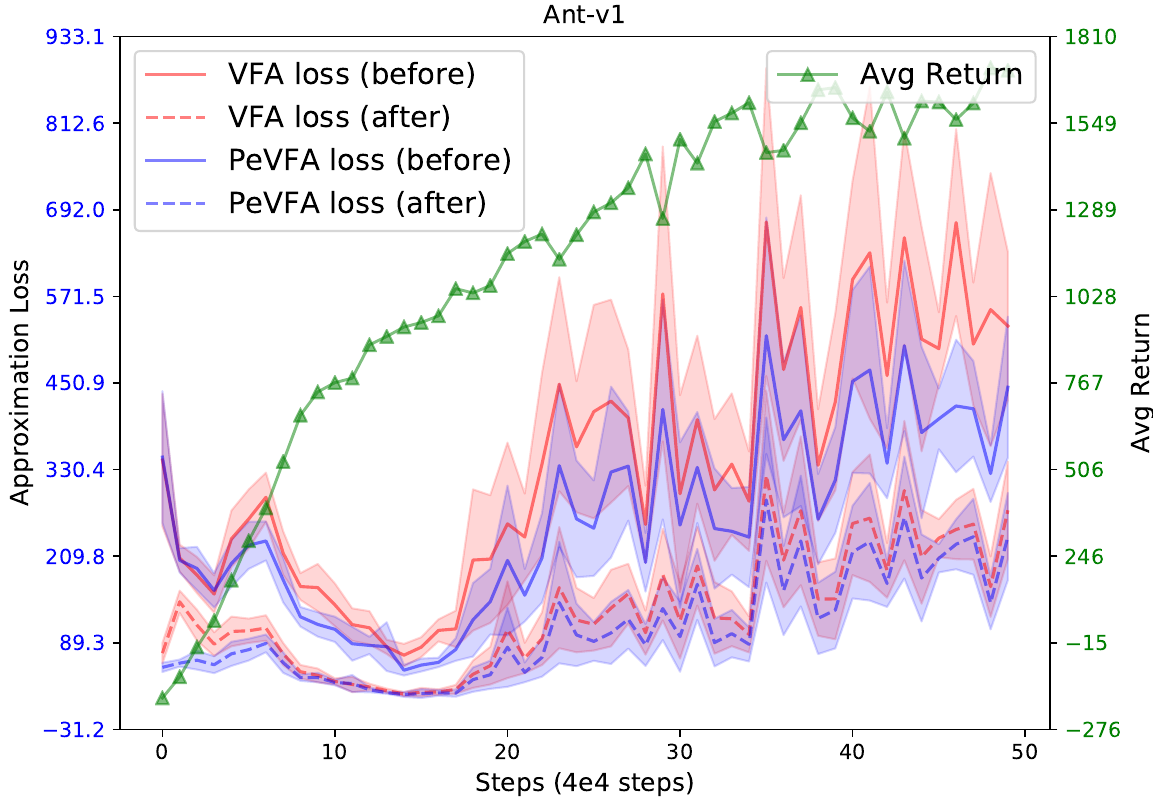}
}
\caption{
Complete empirical evidence of local generalization of PeVFA across 7 MuJoCo tasks.
The learning rate of policy and value function approximators are 0.0001 and 0.001 respectively.
Each plot has two vertical axes, the left one for approximation error (red and blue curves) and the right one for average return (green curves).
Red and blue denotes the approximation error of conventional VFA ($V_{\phi}(s)$) and of PeVFA ($\mathbb{V}_{\theta}(s, \chi_{\pi})$) respectively;
solid and dashed curves denote the approximation error before and after the training for values of successive policy (i.e., policy evaluation) with conventional VFA and PeVFA, averaged over 6 trials.
The shaded region denotes half a standard deviation of average evaluation.
PeVFA consistently shows lower losses (i.e., closer to approximation target) across all tasks than conventional VFA before and after policy evaluation along policy improvement path, which demonstrates Theorem \ref{thm:closer_target}.
}
\label{figure:local_gen_mujoco_complete}
\end{figure}

\begin{figure}[ht]
\centering
\hspace{-0.2cm}
\subfigure[InvPend-v1 (lr = $1e^{-4}$)]{
\includegraphics[width=0.32\textwidth]{appendix_figs/ppo_InvertedPendulum-v1_pt2_0928_fixed.pdf}
}
\hspace{-0.35cm}
\subfigure[InvPend-v1 (lr = $1e^{-3}$)]{
\includegraphics[width=0.32\textwidth]{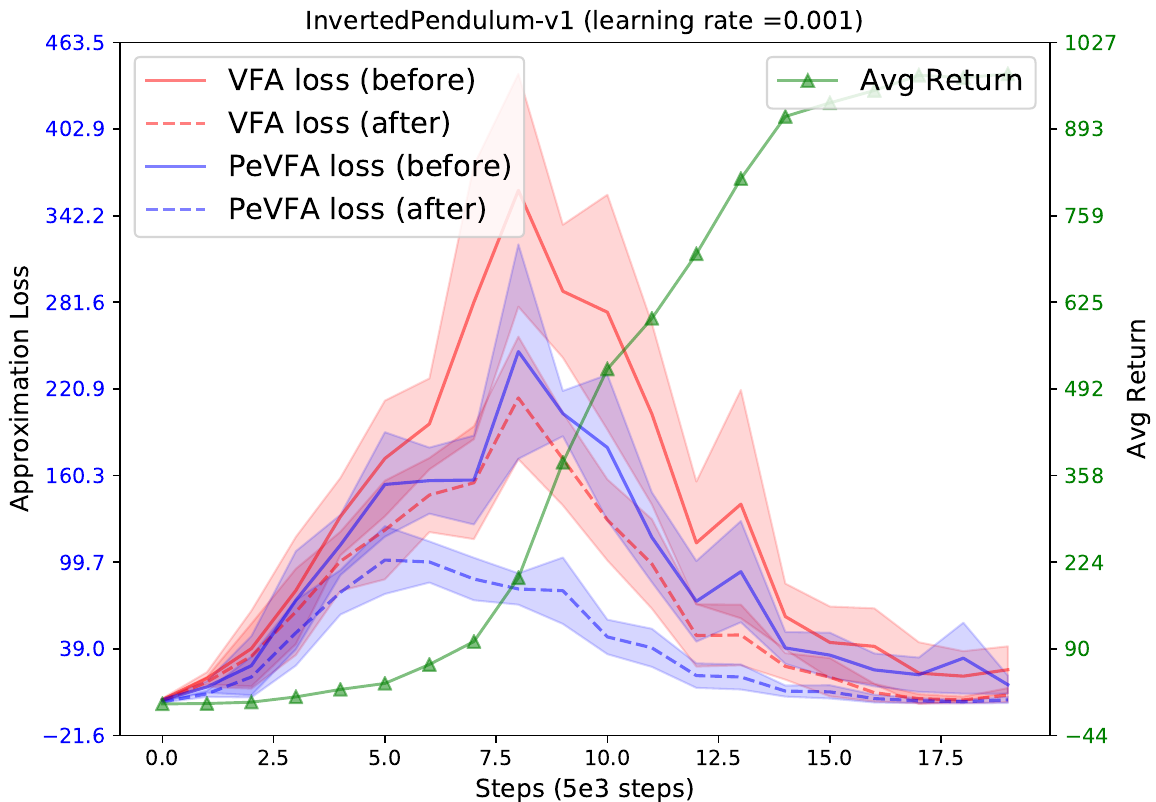}
}
\hspace{-0.35cm}
\subfigure[InvPend-v1 (lr = $5e^{-3}$)]{
\includegraphics[width=0.32\textwidth]{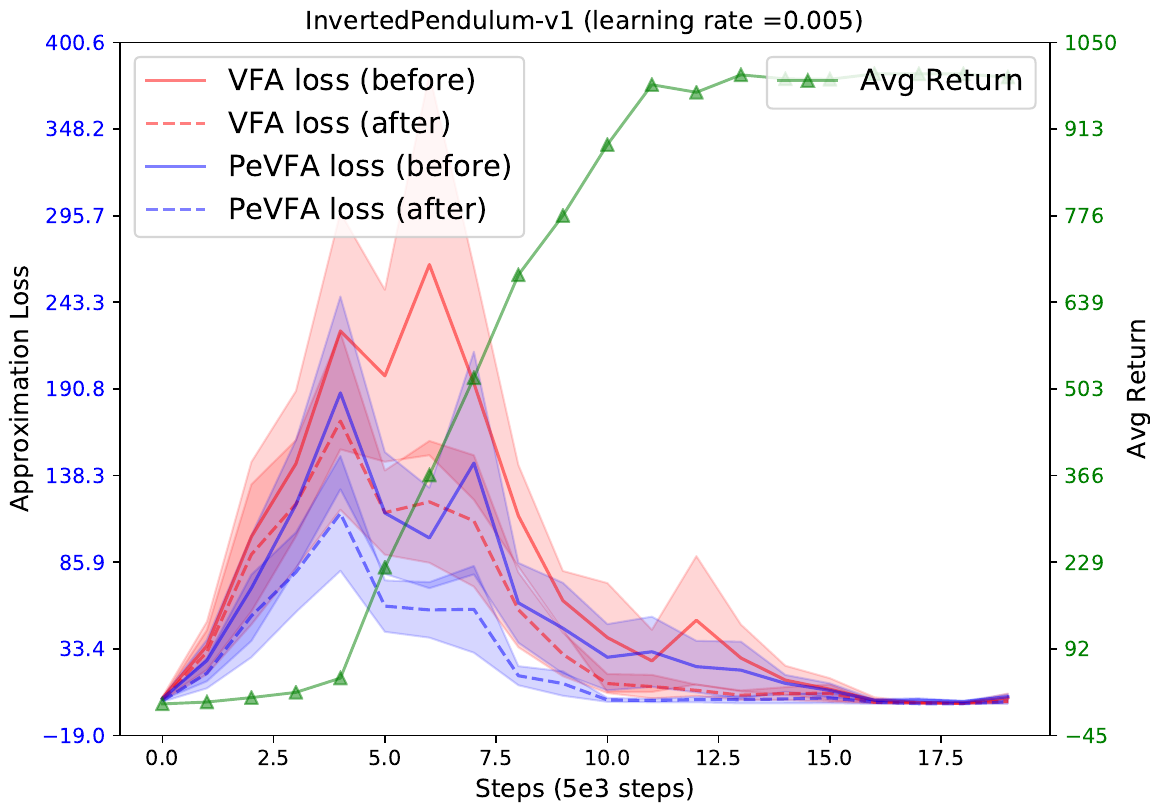}
}
\subfigure[Ant-v1 (lr = $1e^{-4}$)]{
\includegraphics[width=0.32\textwidth]{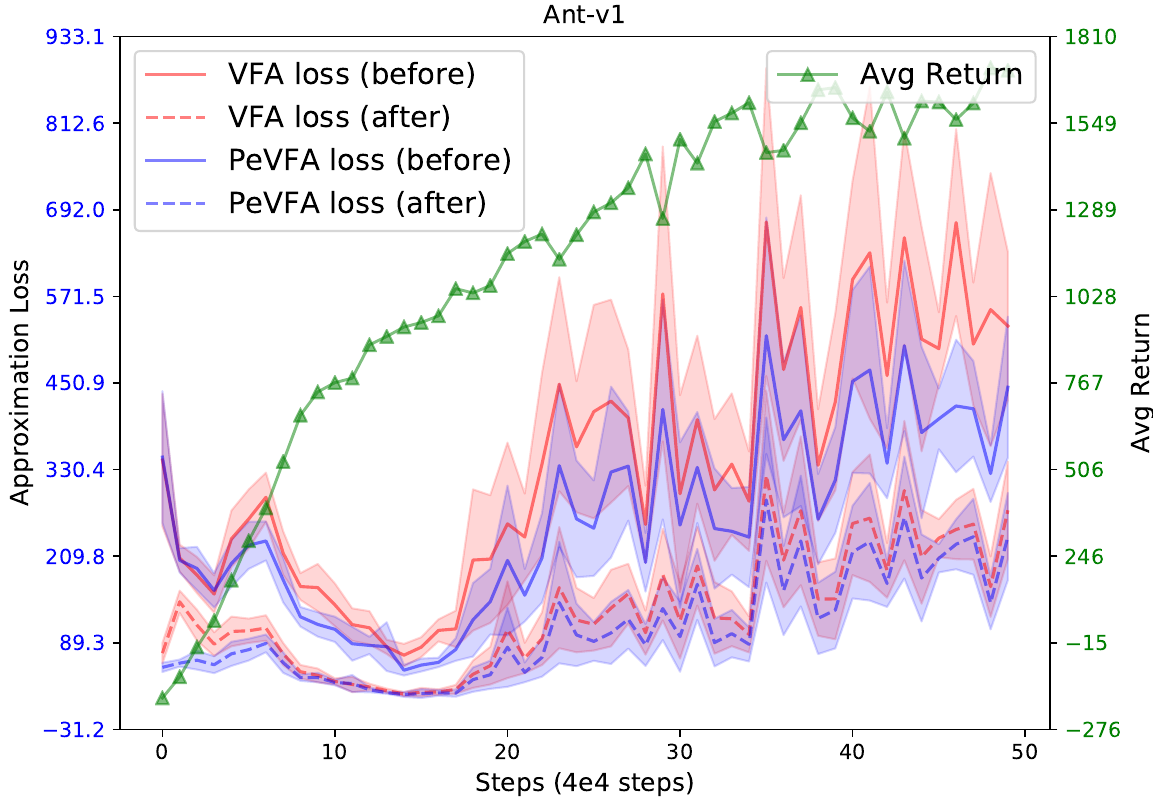}
}
\hspace{-0.35cm}
\subfigure[Ant-v1 (lr = $1e^{-3}$)]{
\includegraphics[width=0.32\textwidth]{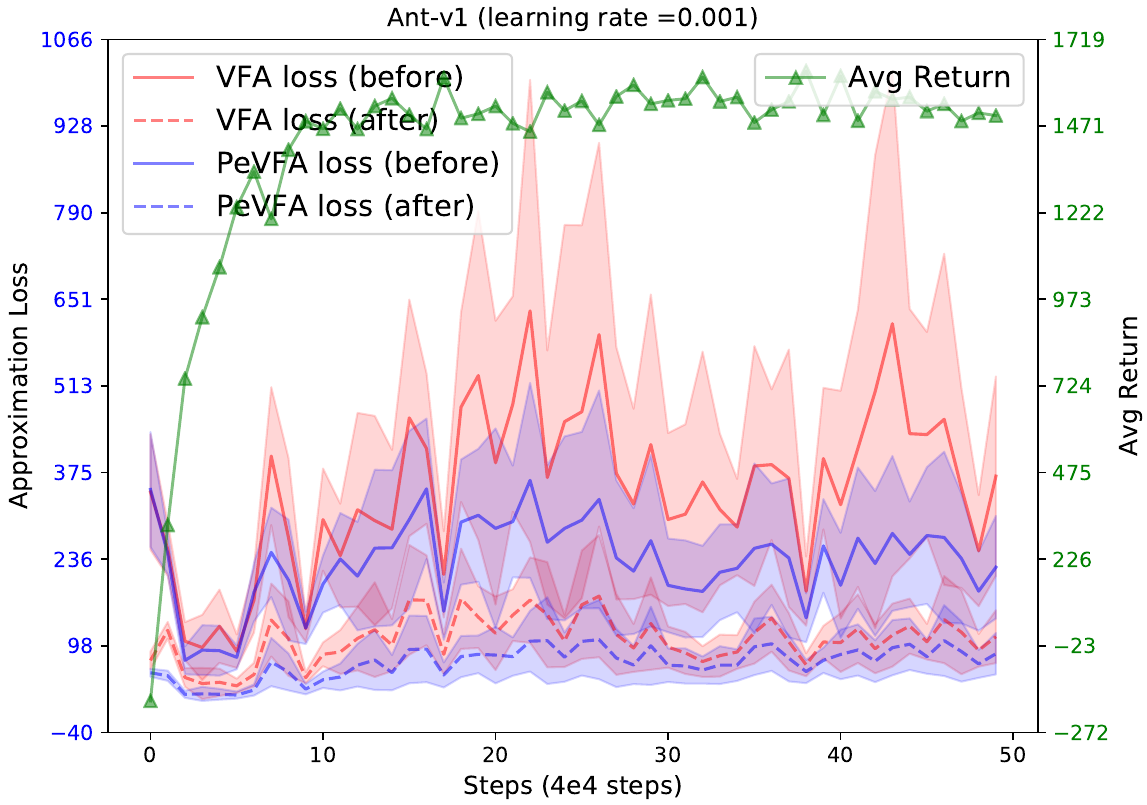}
}
\hspace{-0.35cm}
\subfigure[Ant-v1 (lr = $5e^{-3}$)]{
\includegraphics[width=0.32\textwidth]{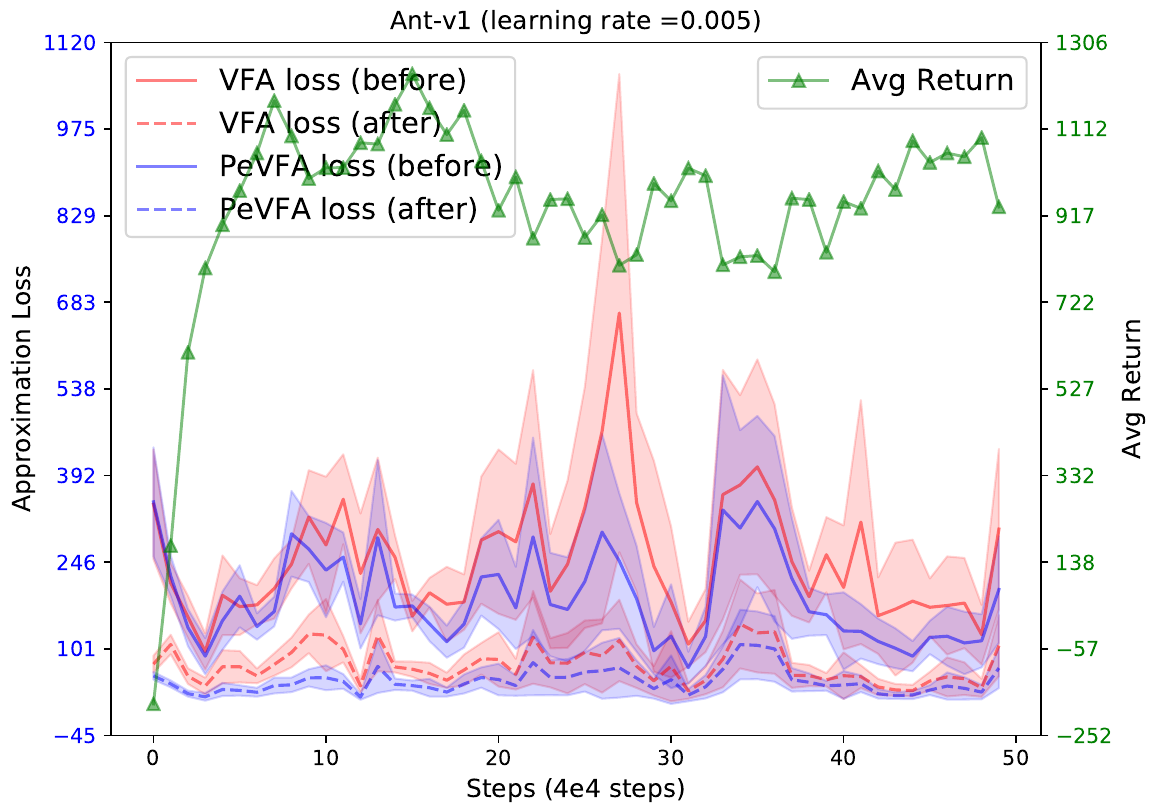}
}
\caption{
Empirical evidence of local generalization of PeVFA on InvertedPendulum-v1 and Ant-v1 with different learning rates of policy, i.e., $\{0.0001, 0.001, 0.005\}$.
Results are averaged over 6 trials.
}
\label{figure:local_gen_mujoco_high_lr}
\end{figure}

\subsection{Local Generalization in MuJoCo Continuous Control Tasks}
\label{app:local_gen_mujoco}
We demonstrate local generalization of PeVFA, especially to examine the existence of Theorem \ref{thm:closer_target},
i.e., PeVFA can induce closer approximation distance (i.e., lower approximation error) than conventional VFA along the policy improvement path.

We use a 2-layer 8-unit policy network trained by PPO \citep{Schulman2017PPO} algorithm in OpenAI MuJoCo continuous control tasks.
As in previous section, using a very small policy network is for the convenience of training and acquisition of policy representation in this demonstrative experiment.
We use all weights and biases of the small policy network (also called \emph{raw} policy representation, RPR), 
whose number is about 10 to 100 in our experiments, 
depending on the specific environment (i.e., the state and action dimensions).
We train the small policy network as commonly done with PPO \citep{Schulman2017PPO} and GAE \citep{Schulman2016GAE}.
The conventional value network $V_{\phi}(s)$ (VFA), 
is a 2-layer 128-unit ReLU-activated MLP with state as input and value as output. 
Parallel to the conventional VFA in PPO, we set a PeVFA network $\mathbb{V}_{\theta}(s,\chi_{\pi})$ with RPR as additional input.
The structure of PeVFA differs at the first hidden layer which has two input streams and each of them has 64 units, as illustrated in Fig.\ref{figure:pevfa_architecture},
so that making VFA and PeVFA have similar scales of parameter number.
In contrast to conventional VFA $V_{\phi}$ which approximates the value of current policy (e.g., Algorithm \ref{algorithm:PG_GPI}),
PeVFA $\mathbb{V}_{\theta}(s,\chi_{\pi})$ has the capacity to preserve values of multiple policies 
and thus is additionally trained to approximate the values of 
all historical policies ($\{\pi_i\}_{i=0}^{t}$) along the policy improvement path (e.g., Algorithm \ref{algorithm:PG_PeVFA}).
The learning rate of policy is 0.0001 and the learning rate of value function approximators ($V_{\phi}(s)$ and $\mathbb{V}_{\theta}(s,\chi_{\pi})$) is 0.001.
The training scheme of PPO policy here is the same as that described in Appendix \ref{app:experimental_details} and Tab.\ref{table:common_hyperparameter}.


Note that $\mathbb{V}_{\theta}(\chi_{\pi})$ does not interfere with PPO training here,
and is only referred as a comparison with $V_{\phi}$ on the approximation error to the true values of successive policy $\pi_{t+1}$.
We use the MC returns of on-policy data (i.e., trajectories) collected by current successive policy as unbiased estimates of true values, similarly done in \citep{Hasselt10DoubleQ,Fujimoto2018TD3}.
Then we calculate the approximation error for VFA $V_{\phi}$ and PeVFA $\mathbb{V}_{\theta}(\chi_{\pi})$ 
to the approximation target before and after value network training of current iteration.
Finally, we compare the approximation error between VFA and PeVFA to approximately examine local generalization and closer approximation target in Theorem \ref{thm:closer_target}.
Complete results of local generalization across all 7 MuJoCo tasks are show in Fig.\ref{figure:local_gen_mujoco_complete}.
The results show that PeVFA consistently shows lower losses (i.e., closer to approximation target) across all tasks than conventional VFA before and after policy evaluation along policy improvement path, which demonstrates Theorem \ref{thm:closer_target}.
Moreover, we also provide similar empirical evidence when policy is updated with larger learning rates in $\{ 0.0001, 0.001, 0.005 \}$, as in Fig.\ref{figure:local_gen_mujoco_high_lr}.

A common observation across almost all results in Fig.\ref{figure:local_gen_mujoco_complete} and in Fig.\ref{figure:local_gen_mujoco_high_lr} is that the larger the extent of policy change (see the regions with a sheer slope on green curves), 
the higher the losses of conventional VFA tend to be (see the peaks of red curves), where the generalization tends to be better and more significant (see the blue curves). 
Since InvertedPendulum-v1 is a simple task while the complexity of the solution for Ant-v1 is higher, the difference between value approximation losses of PeVFA and VFA is more significant at the regions with fast policy improvement.
Besides, the Raw Policy Representation (RPR) we used here does not necessarily induce a smooth and efficient policy representation space, 
among which policy values are easy to generalize and optimize. 
Thus, RPR may be sufficient for a good generalization in InvertedPendulum-v1 but may be not in Ant-v1.
Overall, we think that the quantity of value approximation loss is related to several factors of the environment such as the reward scale, the extent of policy change, the complexity of underlying solution (e.g., value function space) and some others.
A further investigation on this can be interesting.

\section{Generalized Policy Iteration with PeVFA}
\label{app:GPI_with_PeVFA}

\subsection{Comparison between Conventional GPI and GPI with PeVFA}

A graphical comparison of conventional GPI and GPI with PeVFA is shown in Fig.\ref{figure:GPI}. 
Here we provide another comparison with pseudo-codes.

From the lens of Generalized Policy Iteration \citep{SuttonB98},
for most model-free policy-based RL algorithms, the approximation of value function and the update of policy through policy gradient theorem are usually conducted iteratively.
Representative examples are REINFORCE \citep{SuttonB98}, Advantage Actor-Critic \citep{Mnih2016AC}, Deterministic Policy Gradient (DPG) \citep{Silver2014DPG} and Proximal Policy Optimization (PPO) \citep{Schulman2017PPO}.
With conventional value function (approximator),
policy evaluation is usually performed in an on-policy or off-policy fashion.
We provide a general GPI description of model-free policy-based RL algorithm with conventional value functions in Algorithm \ref{algorithm:PG_GPI}.

\begin{algorithm}
  \caption{Generalized policy iteration for model-free policy-based RL algorithm with conventional value functions ($V^{\pi}(s)$ or $Q^{\pi}(s,a)$) }
  \begin{algorithmic}[1]
    \State Initialize policy $\pi_{0}$ and $V_{-1}^{\pi}(s)$ or $Q_{-1}^{\pi}(s,a)$ 
    \State Initialize experience buffer $\mathcal{D}$ 
	\For{iteration t $= 0, 1, 2, \dots$}
        \State Rollout policy $\pi_{t}$ in the environment and obtain trajectories $\mathcal{T}_{t} = \{\tau_{i}\}_{i=0}^{k}$
        \State Add experiences $\mathcal{T}_{t}$ in buffer $\mathcal{D}$
        \If {\emph{on-policy update}} 
            \State Prepare training samples from rollout trajectories $\mathcal{T}_{t}$
        \ElsIf {\emph{off-policy update}}
            \State Prepare training samples by sampling from buffer $\mathcal{D}$
        \EndIf
        \State Calculate approximation target $\{y_i\}_{i}$ from training samples (e.g., with MC or TD)
        \State \textcolor{gray}{\# Generalized Policy Evaluation}
        \State Update $V_{t-1}^{\pi}(s)$ or $Q_{t-1}^{\pi}(s,a)$ with $\{(s_i, y_i)\}_{i}$ or $\{(s_i, a_i, y_i)\}_{i}$,
        i.e.,
        $V^{\pi}_{t} \longleftarrow V^{\pi}_{t-1}$ or $Q^{\pi}_{t} \longleftarrow Q^{\pi}_{t-1}$
        \State \textcolor{gray}{\# Generalized Policy Improvement}
        \State Update policy $\pi_t$ with regard to $V_{t}^{\pi}(s)$ or $Q_{t}^{\pi}(s,a)$ through some policy gradient theorem, 
        i.e., 
        $\pi_{t+1} \longleftarrow \pi_{t}$
	\EndFor
  \end{algorithmic}
\label{algorithm:PG_GPI}
\end{algorithm}
Note that we use subscript $t-1 \rightarrow t$ (Line 13 in Algorithm \ref{algorithm:PG_GPI}) to let the updated value functions to correspond to the evaluated policy $\pi_t$ during policy evaluation process in current iteration.

As a comparison, a new form of GPI with PeVFA is shown in Algorithm \ref{algorithm:PG_PeVFA}.
Except for the different parameterization of value function,
PeVFA can perform additionally training on historical policy experiences at each iteration (Line 7-8).
This is naturally compatible with PeVFA since it develops the capacity of conventional value function to preserve the values of multiple policies.
Such a training is to improve the value generalization of PeVFA among a policy set or policy space. 
Note that for value approximation of current policy $\pi_t$ (Line 10-14),
the start points are generalized values of $\pi_t$ from historical approximation, i.e., $\mathbb{V}_{t-1}(s,\chi_{\pi_t})$ and $\mathbb{Q}_{t-1}(s,a,\chi_{\pi_t})$.
In another word, this is the place where local generalization steps (illustrated in Fig.\ref{figure:local_gen}) are.
One may compare with conventional start points ($V_{t-1}^{\pi}(s)$ and $Q_{t-1}^{\pi}(s,a)$, Line 13 in Algorithm \ref{algorithm:PG_GPI})
and see the difference,
e.g., $V_{t-1}^{\pi}(s) \Leftrightarrow V^{\pi_{{t-1}}}(s) \Leftrightarrow \mathbb{V}_{t-1}(s, \chi_{\pi_{t-1}})$ is different with $\mathbb{V}_{t-1}(s, \chi_{\pi_{t}})$,
where $\Leftrightarrow$ is used to denote an equivalence in definition.
As discussed in Sec. \ref{subsec:demonstrative_exp} and \ref{subsec:PeVFA_GPI},
we suggest that such local generalization steps help to reduce approximation error and thus improve efficiency during the learning process.

\begin{algorithm}
  \caption{Generalized policy iteration of model-free policy-based RL algorithm with PeVFAs ($\mathbb{V}(s,\chi_{\pi})$ or $\mathbb{Q}(s,a,\chi_{\pi})$) }
  \begin{algorithmic}[1]
    \State Initialize policy $\pi_{0}$ and PeVFA $\mathbb{V}_{-1}(s,\chi_{\pi})$ or $\mathbb{Q}_{-1}(s,a,\chi_{\pi})$ 
    \State Initialize experience buffer $\mathcal{D}$ 
	\For{iteration t $= 0, 1, 2, \dots$}
        \State Rollout policy $\pi_{t}$ in the environment and obtain trajectories $\mathcal{T}_{t} = \{\tau_{i}\}_{i=0}^{k}$
        \State Get the policy representation $\chi_{\pi_t}$ for policy $\pi_t$ (from policy network parameters or policy rollout experiences)
        \State Add experiences $(\chi_{\pi_t}, \mathcal{T}_{t})$ in buffer $\mathcal{D}$
        \State \textcolor{gray}{{\# Value approximation training for historical policies $\{\pi_i\}_{i=0}^{t-1}$}}
        \State Update PeVFA $\mathbb{V}_{t-1}(s,\chi_{\pi_i})$ or $\mathbb{Q}_{t-1}(s,a,\chi_{\pi_i})$ with all historical policy experiences $\{(\chi_{\pi_i}, \mathcal{T}_{i})\}_{i=0}^{t-1}$
        \State \textcolor{gray}{{\# Conventional value approximation training for current policy $\pi_t$}}
        \If {\emph{on-policy update}} 
            \State Update PeVFA $\mathbb{V}_{t-1}(s,\chi_{\pi_t})$ or $\mathbb{Q}_{t-1}(s,a,\chi_{\pi_t})$ for $\pi_t$ with on-policy experiences $(\chi_{\pi_t}, \mathcal{T}_{t})$
        \ElsIf {\emph{off-policy update}}
            \State Update PeVFA $\mathbb{V}_{t-1}(s,\chi_{\pi_t})$ or $\mathbb{Q}_{t-1}(s,a,\chi_{\pi_t})$ for $\pi_t$ with off-policy experiences $\chi_{\pi_t}$ and $\{\mathcal{T}_{i}\}_{i=0}^{t}$
            from experience buffer $\mathcal{D}$
        \EndIf
        \State $\mathbb{V}_{t} \longleftarrow \mathbb{V}_{t-1}$ or $\mathbb{Q}_{t} \longleftarrow \mathbb{Q}_{t-1}$
        \State Update policy $\pi_t$ with regard to $\mathbb{V}_{t}(s,\chi_{\pi_t})$ or $\mathbb{Q}_{t}(s,a,\chi_{\pi_t})$ through some policy gradient theorem,
        i.e., $\pi_{t+1} \longleftarrow \pi_{t}$
	\EndFor
  \end{algorithmic}
\label{algorithm:PG_PeVFA}
\end{algorithm}

\subsection{More Discussions on GPI with PeVFA}
\label{app:discussion_on_GPI_PeVFA}

\paragraph{Off-Policy Learning.}
Off-policy Value Estimation \citep{SuttonB98} denotes to evaluate the values of some target policy from data collected by some behave policy.
As commonly seen in RL (also shown in Line 6-10 in Algorithm \ref{algorithm:PG_GPI}),
different algorithms adopt on-policy or off-policy methods.
For GPI with PeVFA, especially for the value estimation of historical policies (Line 8 in Algorithm \ref{algorithm:PG_PeVFA}),
on-policy and off-policy methods can also be considered here.
One interesting thing is, in off-policy case, one can use experiences from any policy for the learning of another one,
which can be appealing since the high data efficiency of value estimation of each policy can strengthen value generalization among themselves with PeVFA, which further improve the value estimation process.

\paragraph{Convergence of GPI with PeVFA.}
Convergence of GPI is usually discussed in some ideal cases, e.g., with small and finite state action spaces and with sufficient function approximation ability.
In this paper, we focus on the comparison between conventional VFA and PeVFA in value estimation, i.e., Policy Evaluation,
and we make no assumption on the Policy Improvement part.
We conjecture that with the same policy improvement algorithm and sufficient function approximation ability, GPI with conventional VFA and GPI with PeVFA finally converge to the same policy.
Moreover, based on Theorem \ref{thm:closer_target} and our empirical evidence in Sec. \ref{subsec:demonstrative_exp}, GPI with PeVFA can be more efficient in some cases:
with local generalization, it could take less experiences (training) for PeVFA to reach the same level of approximation error than conventional VFA,
or with the same amount of experience (training), PeVFA could achieve lower approximation error than conventional VFA.
We believe that a deeper dive in convergence analysis is worth further investigation.

\paragraph{PeVFA with TD Value Estimation.}
In this paper, we propose PPO-PeVFA as a representative instance of re-implementing DRL algorithms with PeVFA. 
Our theoretical results and algorithm \ref{algorithm:PG_PeVFA} proposed under the general policy iteration (GPI) paradigm 
are suitable for TD value estimation as well in principle. 
One potential thing that deserves further investigation is that, 
it can be a more complex generalization problem since the approximation target of TD learning is moving (in contrast to the stationary target when unbiased Monte Carlo estimates are used). 
The non-stationarity induced by TD is recognized to hamper the generalization performance in RL as pointed out in recent work \citep{Igl20Impact}. 
Further study on PeVFA with TD learning (e.g., TD3 and SAC) is planned in the future as mentioned in Sec. \ref{sec:discussion}.

\paragraph{PeVFA and value-based RL.}
We found it is interesting to consider the integration of PeVFA and value-based RL, e.g., Q-Learning, since no explicit policy exists in these cases
The main two different things to consider when we try to apply PeVFA with value-based RL, are 1) the policy representation of an implicit policy and 2) the backup of bellman optimality equation.
We can handle the former with SPR in principle. For the later, due to bellman optimality equation, the LHS and RHS of the typical Q-Learning backup may need to have different policy representations in PeVFA version.
This may be intriguing and more complex especially when bootstrapping is involved. We leave it for future work.

\subsection{More Discussions on PeVFA, Training VFA of PPO More and PPG}
\label{app:discussion_PPG}
One key characteristic of PeVFA is that it performs on-policy value training for the policies along policy improvement path that are previous to the current one, with corresponding historical data collected.

In principle, we can do \textit{more training on on-policy data} in conventional PPO by increase the \textit{epoch} of VFA training (10 commonly used for MuJoCo tasks).
Actually, we had made a study on this and we provide the results in Figure \ref{figure:epoch_MuJoCo} below.
Fewer or more training of VFA in conventional PPO is not beneficial. 

Besides, we can also do \textit{more training on off-policy data (i.e., historical experience)} in conventional PPO.
This can be approximately evaluated according to PPO-PeVFA with Random Policy Representation (Ran PR) in our experiments, since the meaningless policy representation can be viewed as noise input.

In Phasic Policy Gradient (PPG) \cite{CobbeHKS21PPG}, it makes use of \textit{near on-policy data for more training of VFA}. In their paper, this is also demonstrated to be not necessary and can be achieved by increasing the epoch number for VFA (see their Figure 9 and the sentence `Although it is convenient to optimize $L^{value}$ during the auxiliary phase as well, it is not strictly necessary').

\begin{figure}
\centering
\subfigure{
\includegraphics[width=0.35\textwidth]{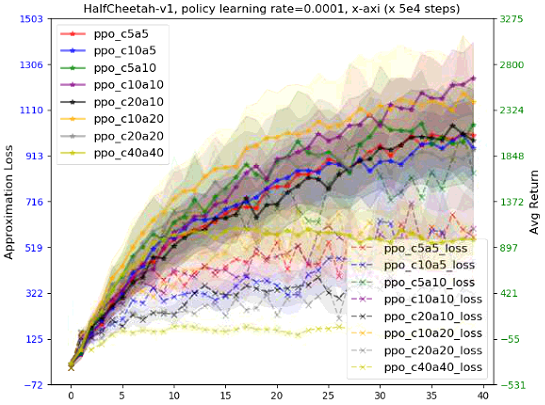}
}
\hspace{0.4cm}
\subfigure{
\includegraphics[width=0.35\textwidth]{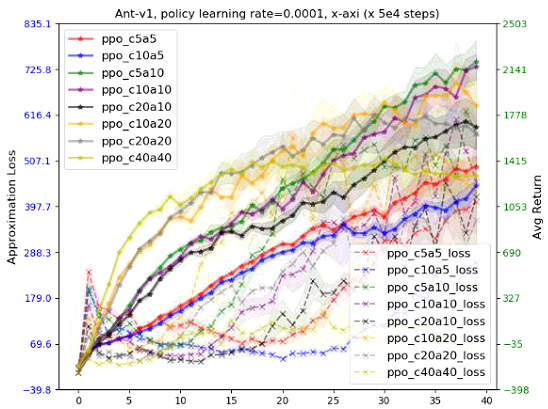}
}
\caption{PPO with different epoch settings. For example, c20a10 denotes 20 epochs for VFA and 10 epochs for policy training at each iteration.
The number of epochs for VFA is set to 10 commonly for MuJoCo tasks. Fewer or more training of VFA in conventional PPO is not beneficial. 
}
\label{figure:epoch_MuJoCo}
\end{figure}

\section{Policy Representation Learning Details}
\label{app:PR}

\subsection{Policy Geometry}
\label{app:policy_geom}
A policy $\pi \in \Pi = \mathcal{P}(\mathcal{A})^{\mathcal{S}}$, defines the behavior (action distribution) of the agent under each state.
For a more intuitive view, we consider the geometrical shape of a policy:
all state $s \in \mathcal{S}$ and all action $a \in \mathcal{A}$ are arranged along the $x$-axis and $y$-axis of a 2-dimensional plane,
and the probability (density) $\pi(a|s)$ is the value of $z$-axis over the 2-dimensional plane.
Note that for finite state space and finite action space (discrete action space), the policy can be viewed as a $|\mathcal{S}|\times|\mathcal{A}|$ table with each entry in it is the probability of the corresponding state-action case.
Without loss of generality, we consider the continuous state and action space and the policy geometry here.
Illustrations of policy geometry examples are shown in Fig.\ref{figure:policy_geom}.

Fig.\ref{figure:arbitary_p} shows the policy geometry in a general case, where the policy can be defined arbitrarily.
Generally, the policy geometry can be any possible geometrical shape (s.t. $\forall s \in \mathcal{S}, \sum_{a \in \mathcal{A}}\pi(a|s) = 1$).
This means that the policy geometry is not necessarily continuous or differentiable in a general case.
Specially, one can imagine that the geometry of a deterministic policy consists of peak points ($z = 1$) for each state and other flat regions ($z = 0$).
Fig.\ref{figure:continuous_p} shows an example of synthetic continuous policy which can be viewed as a 3D curved surface.
In Deep RL, a policy may usually be modeled as a deep neural network.
Assume that the neural policy is a function that is almost continuous and differentiable everywhere,
the geometry of such a neural policy can also be continuous and differentiable almost everywhere.
As shown in Fig.\ref{figure:smoothed_p}, 
we provide a demo of neural policy by smoothing an arbitrary policy along both state and action axes.

\begin{figure}[ht]
\centering
\hspace{-0.5cm}
\subfigure[Arbitrary policy]{
\includegraphics[width=0.34\textwidth]{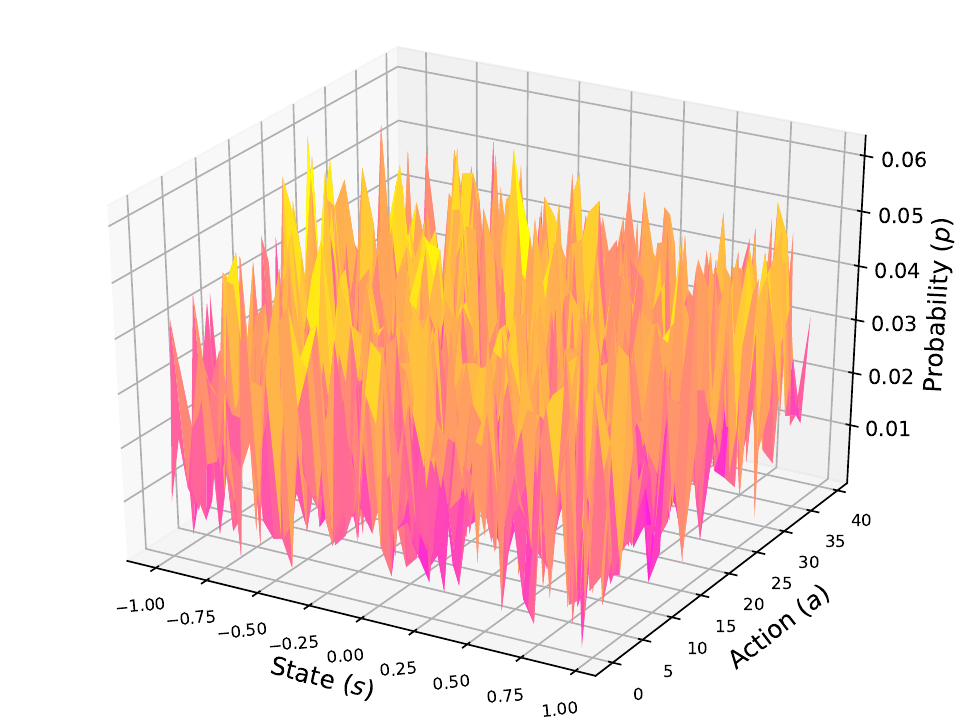}
\label{figure:arbitary_p}
}
\hspace{-0.4cm}
\subfigure[Continuous policy]{
\includegraphics[width=0.34\textwidth]{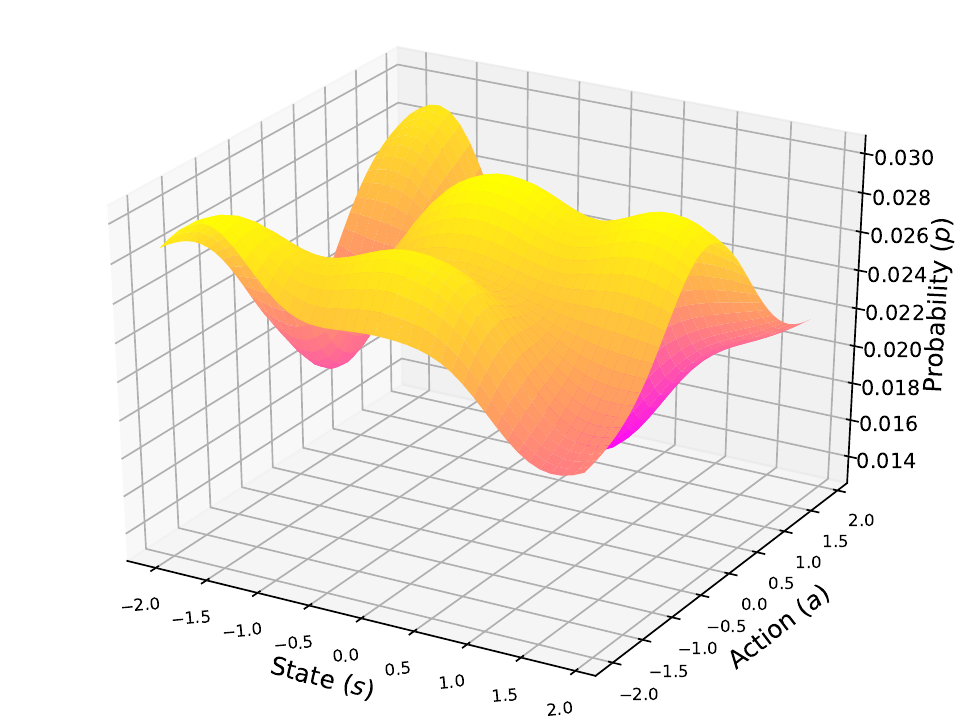}
\label{figure:continuous_p}
}
\hspace{-0.4cm}
\subfigure[Demo of neural policy]{
\includegraphics[width=0.34\textwidth]{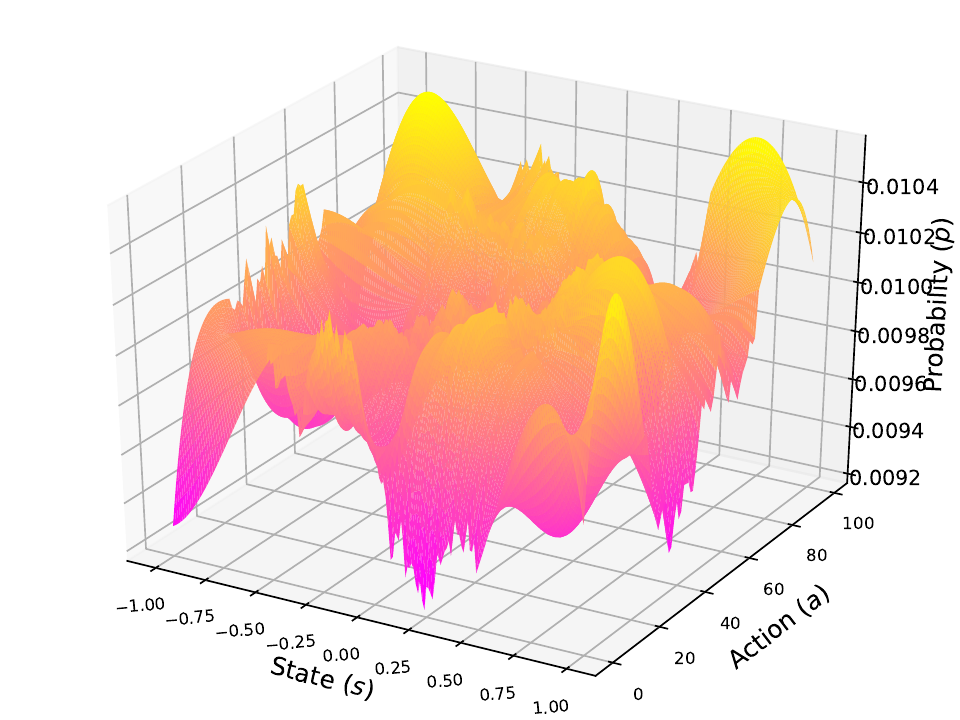}
\label{figure:smoothed_p}
}
\caption{
Examples of policy geometry. 
(a) An arbitrary policy, where $p(s,a)$ is sampled from $\mathcal{N}(0,1)$ for a joint space of 40 states and 40 actions and then normalized along action axis.
States are squeezed into the range of $[-1, 1]$ for clarity.
(b) A synthetic continuous policy with $p(s,a) = (1 - a^5 + s^5) \exp(-s^2 - a^2)$ for a joint space of $s \in [-2,2]$ and $a \in [-2, 2]$ (each of which are discretized into 40 ones) and then normalized along action axis.
(c) A general demo of neural network policy, generated from an arbitrary policy (as in (a)) over a joint space of 200 states and 100 actions with some smoothing skill.
States are squeezed into the range of $[-1, 1]$ for clarity and the probability masses of actions under each state are normalized to sum into 1.
}
\label{figure:policy_geom}
\end{figure}

\subsection{Implementation Details of Surface Policy Representation (SPR) and Origin Policy Representation (OPR)}
\label{app:opr_details}
Here we provide a detailed description of how to encode different policy data for Surface Policy Representation (SPR) and Origin Policy Representation (OPR) we introduced in Sec. \ref{sec:policy_repres}.

\paragraph{Encoding of State-action Pairs for SPR.}
Given a set of state-action pairs $\{s_i, a_i\}_{i=1}^{n}$ (with size $[n, s\_dim + a\_dim]$) generated by policy $\pi$ (i.e., $a_i \sim \pi(\cdot|s_i)$),
we concatenate each state-action pair and obtain an embedding of it by feeding it into an MLP, resulting in a stack of state-action embedding with size $[n, e\_dim]$.
After this, we perform a mean-reduce operator on the stack and obtain an SPR with size $[1, e\_dim]$.
A similar permutation-invariant transformation is previously adopted to encode trajectories in \citep{GroverAGBE18MAPR}.

\paragraph{Encoding of Network Parameters for OPR.}
We propose a novel way to learn low-dimensional embedding from policy network parameters directly.
To our knowledge, we are the first to learn policy embedding from neural network parameters in RL.
Note that latent space of neural networks are also studied in differentiable Network Architecture Search (NAS) \citep{LiuSY19DARTS,LuoTQCL18NAO},
where architecture-level embedding are usually considered.
In contrast, OPR cares about parameter-level embedding with a given architecture.

Consider a policy network to be an MLP with well-represented state (e.g., CNN for pixels, LSTM for sequences) as input and deterministic or stochastic policy output.
We compress all the weights and biases of the MLP to obtain an OPR that represents the decision function.
The encoding process of an MLP with two hidden layers is illustrated in Fig.\ref{figure:opr_encoder}.
The main idea is to perform permutation-invariant transformation for inner-layer weights and biases for each layer first.
For each unit of some layer, we view the unit as a non-linear function of all outputs, determined by weights, a bias term and activation function.
Thus, the whole layer can be viewed as a batch of operations of previous outputs, e.g., with the shape $[h_{t}, h_{t-1} + 1]$ for $t \ge 1$ and $t=0$ is also for the input layer.
Note that we neglect activation function in the encoding since we consider the policy network structure is given.
That is also why we call OPR as parameter-level embedding in contrast to architecture-level enbedding in NAS (mentioned in the last paragraph).
We then feed the operation batch in an MLP and perform mean-reduce to outputs.
Finally we concatenate encoding of layers and obtain the OPR.

We use permutation-invariant transformation for OPR because that we suggest the operation batch introduced in the last paragraph can be permutation-invariant.
Actually, our encoding shown in Fig.\ref{figure:opr_encoder} is not strict to obtain permutation-invariant representation since inter-layer dependencies are not included during the encoding process.
We also tried to incorporate the order information during OPR encoding and we found similar results with the way we present in Fig.\ref{figure:opr_encoder}, which we adopt in our experiments. 

\begin{figure}[ht]
\centering
\includegraphics[width=0.8\textwidth]{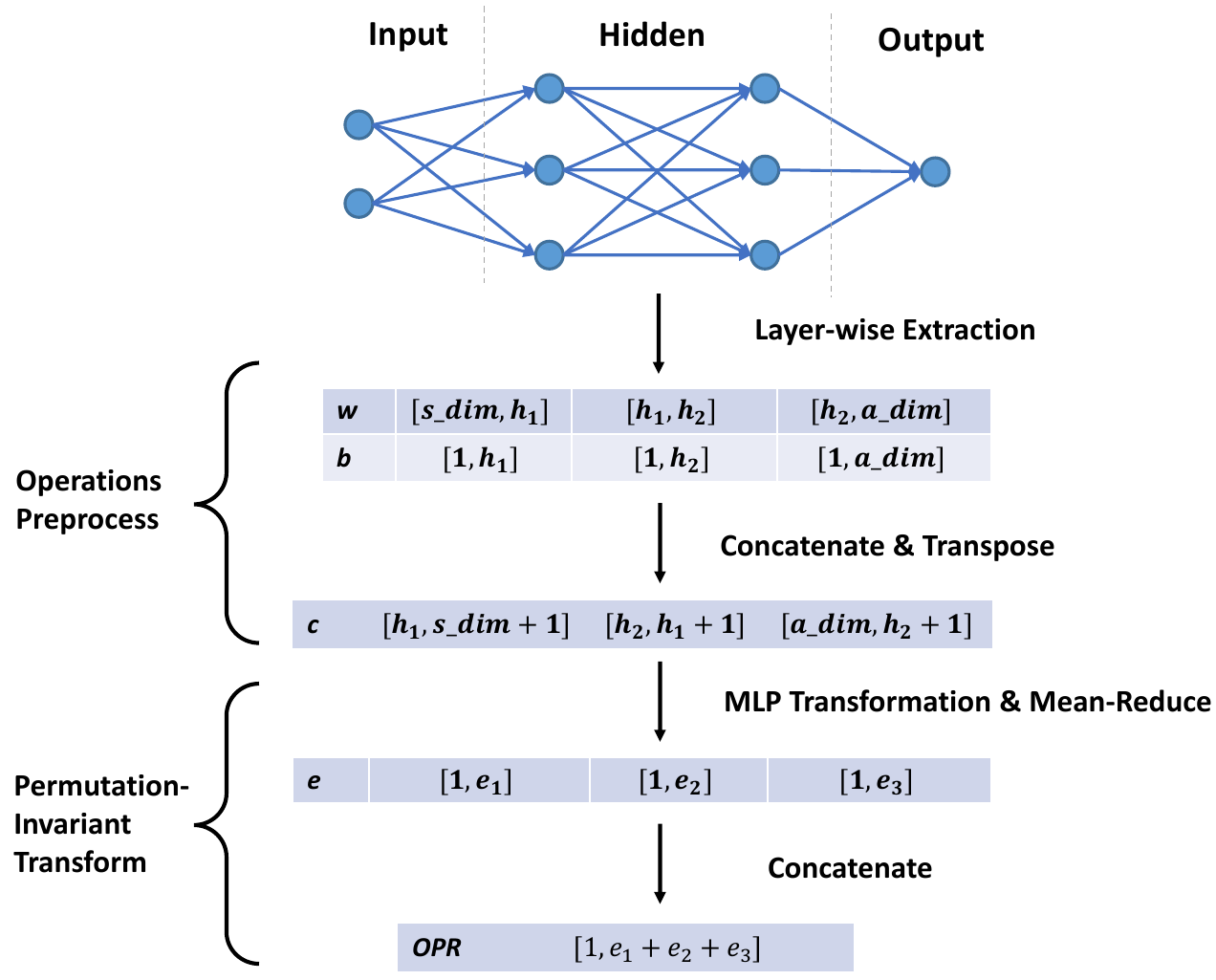}
\caption{
An illustration for policy encoder of Origin Policy Representation (OPR) for a two-layer MLP.
$h_1, h_2$ denotes the numbers of hidden units for the first and second hidden layers respectively.
The main idea is to perform permutation-invariant transformation for inner-layer weights and biases for each layer first and then concatenate encoding of layers.
}
\label{figure:opr_encoder}
\end{figure}

\paragraph{Towards more sophisticated RL policy that operates images.}
Our proposed two policy representations (i.e., OPR and SPR) can basically be applied to encode policies that operate images, 
with the support of advanced image-based state representation. 
For OPR, a policy network with image input usually has a pixel feature extractor like Convolutional Neural Networks (CNNs) followed by a decision model (e.g., an MLP). 
With effective features extracted, the decision model can be of moderate (or relatively small) scale. 
Recent works on unsupervised representation learning like MoCo \citep{He0WXG20MoCo}, SimCLR \citep{Chen20SimCLR}, CURL \citep{Srinivas20CURL} 
also show that a linear classifier or a simple MLP which takes compact representation of images learned in an unsupervised fashion 
is capable of solving image classification and image-based continuous control tasks. 
In another direction, it is promising to develop more efficient even gradient-free OPR, 
for example using the statistics of network parameters in some way instead of all parameters as similarly considered in \citep{Unterthiner20predict}.

For SPR, to encode state-action pairs (or sequences) with image states can be converted to the encoding in the latent space. 
The construction of latent space usually involves self-supervised representation learning, e.g., image reconstruction, dynamics prediction. 
A similar scenario can be found in recent model-based RL like Dreamer \citep{HafnerLB020Dream}, where the imagination is efficiently carried out in the latent state space rather than among original image observations. 

Overall, we believe that there remain more effective approaches to represent RL policy to be developed in the future 
in a general direction of OPR and SPR, which are expected to induce better value generalization in a different RL problems.

\subsection{Data Augmentation for SPR and OPR in Contrastive Learning}
Data augmentation is studied to be an important component in contrastive learning in deep RL recently \citep{Kostrikov20DrQ,Laskin20RAD}.
Contrastive learning usually resorts to data augmentation to build positive samples.
Data augmentation is typically performed on pixel inputs (e.g., images) problems \citep{He0WXG20MoCo,Chen20SimCLR}.
In our work, we train policy representation with contrastive learning where data augmentation is performed on policy data.
For SPR, i.e., state-action pairs as policy data, there is no need to perform data augmentation since different batches of randomly sampled state-action pairs naturally forms positive samples,
since they all reflect the behavior of the same policy.
A similar idea can also be found in \citep{Fu20CMRL} when dealing with task context in Meta-RL.

For OPR, i.e., policy network parameters as policy data, it is unclear how to perform data augmentation on them.
In this work, we consider two kinds of data augmentation for policy network parameters as shown in Fig.\ref{figure:opr_data_aug}.
We found similar results for both random mask and noise corruption, 
and we use random mask as default data augmentation in our experiments. 

\begin{figure}[ht]
\centering
\includegraphics[width=0.95\textwidth]{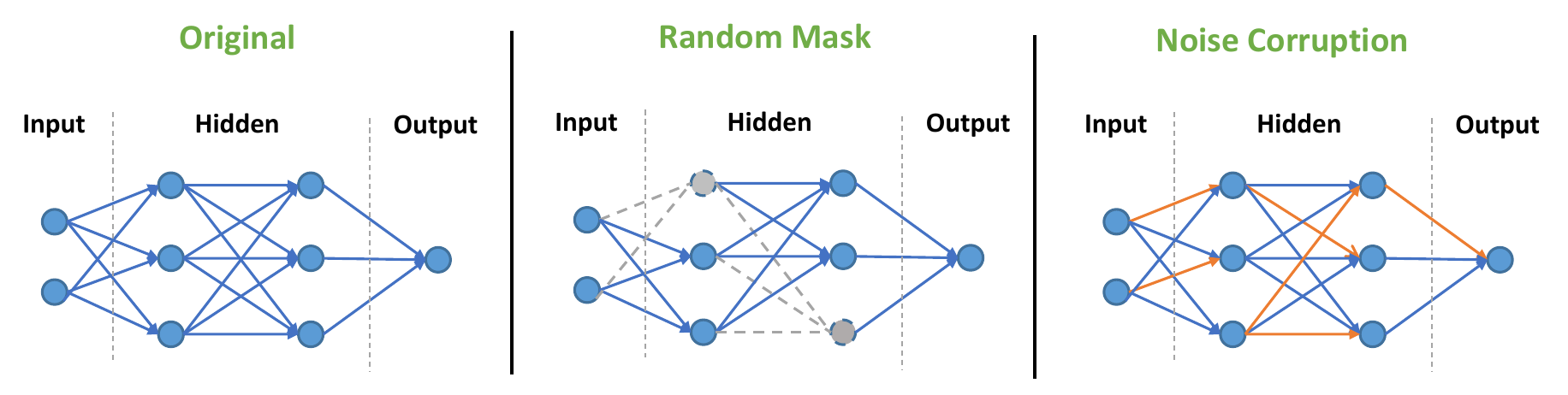}
\caption{
Examples of data augmentation on policy network parameters for Origin Policy Representation (OPR).
\emph{Left}: an example of original policy network.
\emph{Middle}: dropout-like random masks are performed on original policy network, where gray dashed lines represent the weights masked out.
\emph{Right}: randomly selected weights are corrupted by random noises, denoted by orange lines. 
}
\label{figure:opr_data_aug}
\end{figure}

As an unsupervised representation learning method, 
contrastive Learning
encourages policies to be close to similar ones (i.e., positive samples $\pi^{+}$)
and to be apart from different ones (i.e., negative samples $\pi^{-}$) in policy representation space.
The policy representation network is then trained with InfoNCE loss \citep{Oord18CPC},
i.e., to minimize the cross-entropy loss below:
\begin{equation*}
    \mathcal{L}_{\text{CL}} = 
    - \mathbb{E} \left[ \log \frac{\exp(\chi_{\pi}^{T} W \chi_{\pi^{+}})}{\exp(\chi_{\pi}^{T} W \chi_{\pi^{+}}) + \sum_{\pi^{-}} \exp(\chi_{\pi}^{T} W \chi_{\pi^{-}})}
    \right]
\end{equation*}

\subsection{Pseudo-code of Policy Representation Learning Framework}
\label{app:overall_view_pr}

The pseudo-code of the overall framework of policy representation learning is in Algorithm \ref{algorithm:overall_view_pr}.
The policy representation learning is conducted base on a policy dataset, which stores the policy data, i.e., interaction trajectories generated by policies and the parameters of policy networks.
Such a dataset can be obtained in different ways, e.g., pre-collected, online collected and etc.
In our experiments, we do not assume the access to pre-collected or given policy data; instead, we use the data of all historical policies met along the policy improvement path during the online learning process.

Different kinds of policy data (i.e., state-action pairs or policy parameters) are used depending on the policy representation adopted (i.e., SPR or OPR). 
For policy representation learning, the value function approximation loss (E2E) is used as a default choice of training loss in our framework.
In addition, the auxiliary loss (AUX) of policy recovery and contrastive learning (CL) serve as another two options to be optimized for representation learning.
Note that in Line 21, the positive samples $\chi_{\pi_i^+}$ is obtained from a momentum policy encoder \citep{He0WXG20MoCo} with another augmentation for corresponding policy data,
while negative samples $\chi_{\pi_i^-}$ are other policy embeddings in the same batch, i.e., $\chi_{\pi_i^-} \in B \backslash \{\chi_{\pi_i}\}$.

\begin{algorithm}
  \caption{A Framework of Policy Representation Learning}
  \textbf{Input:} policy dataset $\mathbb{D} = \{(\pi_i, \omega_i, \mathcal{D}_{\pi_i})\}_{i=1}^n$, 
  consisting of policy $\pi_i$, policy parameters $\omega_i$ and state-action pairs $\mathcal{D}_{\pi_i} = \{(s_j, a_j)\}_{j=1}^m$
  \begin{algorithmic}[1]
    \State Initialize the policy encoder $g_{\alpha}$ with parameters $\alpha$
    \State Initialize the policy decoder (or master policy) (network) $\Bar{\pi}_{\beta}(a|s,\chi_{\pi})$ for SPR and the weight matrix $W$ for ORP respectively
	\For{iteration $i$ $= 0, 1, 2, \dots$}
	    \State Sample a mini-batch of policy data $\mathcal{B}$ from $\mathbb{D}$
	    \State \textcolor{gray}{\# Encode and obtain the policy embedding $\chi_{\pi_i}$ with SPR or OPR}
	    \If {Use OPR}
	        \If {Use Contrastive Learning}
	            \State Perform data augmentation on each $w_i \in \mathcal{B}$ 
	        \EndIf
	        \State $\chi_{\pi_i} = g_{\alpha}^{\text{OPR}}(\omega_{\pi_i})$ for each $(\pi_i, \omega_i, \cdot) \in \mathcal{B}$ 
	    \ElsIf{Use SPR}
	        \State $\chi_{\pi_i} = g_{\alpha}^{\text{SPR}}(B_{i})$ where $B_{i}$ is a mini-batch of state-action pairs sampled from $\mathcal{D}_{\pi_i}$,
	        for each $(\pi_i, \cdot, \mathcal{D}_{\pi_i}) \in \mathcal{B}$ 
	    \EndIf
	    \State \textcolor{gray}{\# Train policy encoder $g_{\alpha}$ in different ways (i.e., AUX or CL)}
	    \If{Use Auxiliary Loss (AUX)}
	        \State Sample a mini-batch of state-action pairs $B = {(s_i, a_i)}_{i=1}^b$ from $\mathcal{D}_{\pi_{i}}$ for each $\pi_i$
	        \State Compute the auxiliary loss, $\mathcal{L}_{\text{Aux}} = - \sum_{(s_i,a_i) \in B} \log \Bar{\pi}_{\alpha}(a_i|s_i,\chi_{\pi_i})$
	        \State Update parameters $\alpha,\beta$ to minimize $\mathcal{L}^{\text{Aux}}$
	    \EndIf
	    \If{Use Contrastive Learning (CL)}
	        \State Calculate contrastive loss, 
	        $\mathcal{L}_{\text{CL}} = - \sum_{\chi_{\pi_i} \in B} \log \frac{\exp(\chi_{\pi_i}^{T} W \chi_{\pi_i^{+}})}{\exp(\chi_{\pi_i}^{T} W \chi_{\pi_i^{+}}) + \sum_{\pi_i^{-}} \exp(\chi_{\pi_i}^{T} W \chi_{\pi_i^{-}})}$,
	        where $\chi_{\pi_i^+},\chi_{\pi_i^-}$ are positive and negative samples
	        \State Update parameters $\alpha, W$ to minimize $\mathcal{L}^{\text{CL}}$
	    \EndIf
	    \State \textcolor{gray}{\# Train policy encoder $g_{\alpha}$ with the PeVFA approximation loss (E2E)}
	    \State Calculate the value approximation loss of PeVFA, $\mathcal{L}_{\text{Val}}$
	    \State Update parameters $\alpha$ to minimize $\mathcal{L}_{\text{Val}}$
	\EndFor
  \end{algorithmic}
\label{algorithm:overall_view_pr}
\end{algorithm}

\subsection{Criteria of A Good Policy Representation}
\label{app:critira}

To answer the question: what is a good representation for RL policy ought to be? 
We assume the following criteria:
\begin{itemize}
    \item 
    \emph{Dynamics}. 
    Intuitively, a good policy representation should contain the information of how the policy influences the environment (dynamics and rewards).
    \item 
    \emph{Consistency}.
    A good policy representation should keep the consistency among both policy space and presentation space.
    Concretely, the policy representation should be distinguishable, i.e., different policies also differ among their representation.
    In contrast, the representation of similar polices should lay on the close place in the representation space.
    \item 
    \emph{Geometry}.
    Additionally, from the lens of policy geometry as shown in Appendix \ref{app:policy_geom},
    a good policy representation should be an reflection of policy geometry.
    It should show a connection to the policy geometry or be interpretable from the geometric view.
\end{itemize}
From the perspective of above criteria, 
SPR follows \textit{Dynamics} and \textit{Geometry} 
while OPR may render them in an implicit way since network parameters determine the nonlinear function of policy.
Auxiliary loss for policy recovery (AUX) is a learning objective to acquire \textit{Dynamics};
Contrastive Learning (CL) is used to impose \textit{Consistency}.

Based on the above thoughts, we hypothesize the reasons of several findings as shown in the comparison in Tab.\ref{table:overall_evaluations}.
First, AUX naturally overlaps with SPR and OPR to some degree for \textit{Dynamics} while CL is relatively complementary to SPR and OPR for \textit{Consistency}.
This may be the reason why CL improves the E2E representation more than AUX in an overall view.
Second, the noise of state-action samples for SPR may be the reason to OPR's slightly better overall performance than that of SPR (similar results are also found in our visualizations as in Fig.\ref{figure:visual_spr}).

Moreover, the above criteria are mainly considered from an unsupervised or self-supervised perspective.
However, a sufficiently good representation of all the above properties may not be necessary for a specific downstream generalization or learning problem which utilizes the policy representation.
A problem-specific learning signal, e.g., the value approximation loss in our paper (E2E representation), can be efficient since it is to extract the most relevant information in policy representation for the problem.
A recent work \citep{Tsai20Demystify} also studies the relation between self-supervised representation and downstream tasks from the lens of mutual information.
Therefore, we suggest that a trade-off between good unsupervised properties and efficient problem-specific information of policy representation should be considered when using policy representation in a specific problem.

\section{Complete Background and Detailed Related Work}
\label{app:bg_more}

\subsection{Reinforcement Learning}
\label{app:detailed_mdp_description}
\paragraph{Markov Decision Process.}
We consider a Markov Decision Process (MDP) 
defined as
$\langle \mathcal{S}, \mathcal{A}, r, \mathcal{P}, \gamma , \rho_0 \rangle$ with $\mathcal{S}$ the state space, $\mathcal{A}$ the
action space, $r$ the reward function, $\mathcal{P}$ the transition
function, $\gamma \in [0,1)$ the discount factor
and $\rho_0$ the initial state distribution.
A policy $\pi \in \Pi = P(\mathcal{A})^{|\mathcal{S}|}$, defines the distribution over all actions for each state.
The agent interacts with the environment with its policy,
generating the trajectory $s_0, a_0, r_1, s_1, a_1, r_2, ..., s_t, a_t, r_{t+1}, ...$,
where $r_{t+1} = r(s_t,a_t)$.
An RL agent seeks for an optimal policy that
maximizes the expected long-term discounted return, 
$J(\pi) = \mathbb{E}_{s_0 \sim \rho_0, a \sim \pi} \big[ \sum_{t=0}^{\infty} \gamma^{t} r_{t+1} \big]$.

\paragraph{Value Function.}
Almost all RL algorithms involve value functions \citep{SuttonB98}, which estimate how good a state or a state-action pair is conditioned on a given policy.
The \emph{state-value function} $v^{\pi}(s)$ is defined in terms of the expected return obtained through following the policy $\pi$ from a state $s$:
\begin{equation*}
\label{eqation:VF}
    v^{\pi}(s) =  \mathbb{E}_{\pi}  \left[\sum_{t=0}^{\infty} \gamma^{t} r_{t+1}|s_0=s \right] \ \text{for all} \ s \in \mathcal{S}.
\end{equation*}
Similarly, \emph{action-value function} is defined for all state-action pairs
as $q^{\pi}(s,a) = \mathbb{E}_{\pi} \left[\sum_{t=0}^{\infty}\gamma^{t} r_{t+1}|s_0=s, a_0=a \right]$.
Typically, value functions are learned through Monte Carlo (MC) or Temporal Difference (TD) algorithms \citep{SuttonB98}.

Bellman equations defines the recursive relationships among value functions. 
The \emph{Bellman Expectation equation} of $v^{\pi}(s)$ has a matrix form as below \citep{SuttonB98}:
\begin{equation}
\label{equation:vectorBellmanV}
\begin{aligned}
    V^{\pi} = r^{\pi} + \gamma \mathcal{P}^{\pi} V^{\pi} = (\mathcal{I} - \gamma \mathcal{P}^{\pi})^{-1} r^{\pi},
\end{aligned}
\end{equation}
where $V^{\pi}$ is a $|\mathcal{S}|$-dimensional vector, $\mathcal{P}^{\pi}$ is the state-to-state transition matrix $\mathcal{P}^{\pi}(s^{\prime} | s) = \sum_{a \in \mathcal{A}} \pi(a|s) \mathcal{P}(s^{\prime} | s, a)$
and $r^{\pi}$ is the vector of expected rewards $r^{\pi}(s) = \sum_{a \in \mathcal{A}} \pi(a|s) r(s, a)$.
Eq.\ref{equation:vectorBellmanV} indicates that value function is determined by policy $\pi$ and environment models (i.e., $\mathcal{P}$ and $r$.
For a conventional value function, all of them are modeled implicitly within a table or a function approximator, 
i.e., a mapping from only states (and actions) to values.

\paragraph{Generalized Policy Iteration.}
\citet{SuttonB98} consider most RL algorithms can be described in the paradigm of Generalized Policy Iteration (GPI).
In recent decade, RL algorithms usually resort to function approximation (e.g., deep neural networks) to deal with large and continuous state space.
An illustration of GPI with function approximation is on the left of Fig.\ref{figure:GPI}.
We use $\theta$ to denote the parameters of parameterized value functions.
Without loss of generality, we do not plot the parameters of policy since it is not necessary for parameterized policy to exist, e.g., value-based RL algorithms \citep{Mnih2015DQN}.
For policy evaluation, value function approximators are updated in finite times to approximate the true values 
(i.e., $V_{\theta}(s) \rightarrow v^{\pi}(s)$, $Q_{\theta}(s,a) \rightarrow q^{\pi}(s,a)$), 
yet can never be perfect.
For policy improvement, the policy are improved with respected to the approximated value functions in an implicit (e.g., value-based RL) or explicit way (policy-based RL).
In deep RL, perfect policy evaluation and effective policy improvement are non-trivial to obtain with complex non-linear function approximation from deep neural networks,
thus most convergence and optimality results in conventional RL usually no longer hold.
From these two aspects, many works study how to improve the value function approximation 
\citep{Hasselt10DoubleQ,BellemareDM17C51,LanPFW20MMQ} 
and to propose more effective policy optimization or search algorithms 
\citep{Schulman2015TRPO,Schulman2017PPO,HaarnojaZAL18SAC}.

\subsection{A Unified View of Extensions of Conventional Value Function from the Vector Form of Bellman Equation}
\label{app:unified_view_of_extensions}

Recall the vector form of Bellman equation (Eq.\ref{equation:vectorBellmanV}),
it indicates that value function is a function of policy $\pi$ and environmental models (i.e., $\mathcal{P}$ and $r$).
In conventional value functions and approximators, only state (and action) is usually taken as input
while other components in Eq.\ref{equation:vectorBellmanV} are modeled implicitly.
Beyond state (and action), 
consider explicit representation of some of components in Eq.\ref{equation:vectorBellmanV} during value estimation 
can develop the ability of conventional value functions in different ways,
to solve challenging problems,
e.g., goal-conditioned RL \citep{SchaulHGS15UVFA,AndrychowiczCRS17HER}, Hierarchical RL \citep{Nachum19HIRO,WangY0R20I2HRL}, opponent modeling and ad-hoc team \citep{HeB16Opponent,GroverAGBE18MAPR,TacchettiSMZKRG19RMMA}, and context-based Meta-RL \citep{RakellyZFLQ19PEARL,Lee20CADM}.

Most extensions of conventional VFA mentioned above are proposed for the purpose of value generalization (among different space).
Therefore, we suggest such extensions are derived from the same start point (i.e., Eq.\ref{equation:vectorBellmanV}) and differ at the objective to represent and take as additional input explicitly of conventional value functions.
We provide a unified view of such extensions below:
\begin{itemize}
    \item Goal-conditioned RL and context-based meta-RL usually focus on a series of tasks with similar goals and environment models (i.e., $\mathcal{P}$ and $r$).
    With goal representation as input, usually a subspace of state space \citep{SchaulHGS15UVFA,AndrychowiczCRS17HER}, a value function approximation (VFA) can generalize values among goal space.
    Similarly, with context representation \citep{RakellyZFLQ19PEARL,Fu20CMRL,Raileanu20PDVF}, values generalize in meta tasks.
    \item Opponent modeling, ad-hoc team \citep{HeB16Opponent,GroverAGBE18MAPR,TacchettiSMZKRG19RMMA} seek to generalize among different opponents or teammates in a Multiagent System, with learned representation of opponents.
    This can be viewed as a special case of value generalization among environment models since from one agent view,
    other opponents are part of the environment which also determines the dynamics and rewards.
    In multiagent case, one can expand and decompose the corresponded joint policy in Eq.\ref{equation:vectorBellmanV} to see this.
    \item Hierarchical RL is also a special case of value generalization among environment models.
    In goal-reaching fashioned Hierarchical HRL \citep{Nachum19HIRO,LevyKPS19HAC,NachumGLL19NearOpt},
    high-level controllers (policy) issue goals for low-level controls at an abstract temporal scale, 
    while low-level controls take goals also as input and aim to reach the goals.
    For low-level policies, a VFA with a given or learned goal representation space can generalize values among different goals, similar to the goal-conditioned RL case as discussed above.
    Another perspective is to view the separate learning process of hierarchical policies for different levels as a multiagent learning system.
    Recently, a work \citep{WangY0R20I2HRL} follows this view and extends high-level policy with representation of low-level learning.
\end{itemize}
The common thing of above is that, they learn a representation of the environment (we call \emph{external variables}).
In contrast, we study value generalization among agent's own policies in this paper, which cares about \emph{internal variables}, 
i.e., the learning dynamics inside of the agent itself.

\paragraph{Relation between PeVFA Value Approximation and Context-based Meta-RL.}
For a given MDP, performing a policy in the MDP actually induces a Markov Reward Process (MRP) \citep{SuttonB98}.
One can view the policy and actions are absorbed in the transition function of MRP.
A value function defines the expected long-term returns starting from a state.
Therefore, different policies induces different MRPs and PeVFA value approximation can be considered as a meta prediction task.
In analogy to context-based Meta-RL 
where a task context is learned to capture the underlying transition function and reward function of a MDP (i.e., task), 
one can view policy representation as the context of corresponding MRP, 
since it is the underlying variable that determines the transition function of MRPs.

\subsection{A Review of Works on Policy Representation/Embedding Learning}
\label{app:pr_related_works}
Recent years, a few works involve representation or embedding learning for RL policy 
\citep{HausmanS0HR18PESkill,GroverAGBE18MAPR,ArnekvistKS19VPE,Raileanu20PDVF,WangY0R20I2HRL,Harb20PENs}.
We provide a brief review and summary for above works below.

The most common way to learn a policy representation is to extract from interaction trajectories through policy recovery (i.e., behavioral cloning).
For Multiagent Opponent Modeling \citep{GroverAGBE18MAPR},
a policy representation is learned from interaction episodes (i.e., state-action trajectories) through a \textit{generative loss} and \textit{discriminate loss}.
Generative loss is the same as the policy recovery auxiliary loss;
discriminate loss is a triplet loss that minimize the representation distance of the same policy and maximize those of different ones,
which has the similar idea of Contrastive Learning \citep{Oord18CPC,Srinivas20CURL}.
Such opponent policy representations are used for prediction of interaction outcomes for ad-hoc teams and are taken in policy network for some learning agent to facilitate the learning when cooperating or competing with unknown opponents.
More recently, in Hierarchical RL \citep{WangY0R20I2HRL},
a representation is learned to model the low-level policy through \textit{generative loss} mentioned above.
The low-level policy representation is taken in high-level policy to counter the non-stationarity issue of co-learning of hierarchical policies.
Later, \citet{Raileanu20PDVF} resort to almost the same method and the learned policy representation is taken in their proposed PDVF.
Along with a task context, the policy for a specific task can be optimized in policy representation space,
inducing a fast adaptation in new tasks.
In summary, such a representation learning paradigm can be considered as Surface Policy Representation (SPR) for policy data encoding (trajectories as a special form of state-action pairs) plus policy recovery auxiliary loss (AUX) as we introduced in Sec. \ref{sec:policy_repres}.

A recent work \citep{Harb20PENs} proposes Policy Evaluation Network (PVN) to approximate objective function $J(\pi)$.
We consider PVN as an predecessor of PDVF we mentioned above 
since offline policy optimization is also conducted in learned representation space in a single task after similarly well training the PVN on many policies.
The authors propose \textit{Network Fingerprint} to represent policy network.
To circumvent the difficulty of representing the parameters directly, 
policy network outputs (policy distribution) under a set of \textit{probing states} are concatenated and then taken as policy representation.
Such probing states are randomly sampled for initialization and also optimized with gradients through PVN and policies, 
like a search in joint state space.
In principle, we also consider this as a special instance of SPR,
because network fingerprint follows the idea of reflecting the information of how policy can behave under some states.
Intuitively from a geometric view, this can be viewed as using the concatenation of several representative (as denoted by the probing states) cross-sections in policy surface (e.g., Fig.\ref{figure:policy_geom}) to represent a policy.
For another view, one can imagine an equivalent case between SPR and network fingerprint,
when state-action pairs of a deterministic policy are processed in SPR 
and a representation consisting of a number of actions under some key states or representative states is used in network fingerprint.
Two potential issues may exist for network fingerprint.
First, the dimensionality of representation is proportional to the number of probing states (i.e., $n|\mathcal{A}|$), where a dilemma exists:
more probing states are more representative while dimensionality can increase correspondingly.
Second, it can be non-trivial and even unpractical to optimize probing states in the case with relatively state space of high dimension,
which introduces additional computational complexity and optimization difficulty.

In another concurrent work \citep{Faccio20PVFs}, a class of Parameter-based Value Functions (PVFs) are proposed.
Instead of learning or designing a representation of policy, PVFs simply parse all the policy weights as inputs to the value function (i.e., Raw Policy Representation as also mentioned in our paper), even in the nonlinear case.
Apparently, this can result in a unnecessarily large representation space which increase the difficulty of approximation and generalization.
The issues of naively flattening the policy into a vector input are also pointed out in PVN \citep{Harb20PENs}.

Others, several works in Policy Adaptation and Transfer \citep{HausmanS0HR18PESkill,ArnekvistKS19VPE}, 
Gaussian policy embedding representations are construct through Variantional Inference.

\subsection{More Discussions on GTPI}
\label{app:GTPI}
A widely known fact in conventional policy gradient RL is that, value functions are not directly differentiable to policy parameters, i.e., $\nabla_{\theta}q^{\pi_{\theta}}(s,a)$ and $\nabla_{\theta}v^{\pi_{\theta}}(s)$.
Therefore, for classical on-policy PG theorem proposed by \cite{SuttonB98}, $\nabla_{\theta}v^{\pi_{\theta}}(s)$ is recursively expanded and to add up infinite $\nabla_{\theta}\log \pi_{\theta}(a|s) \cdot q^{\pi_{\theta}(s,a)}$ terms which are differentiable.
For conventional off-policy PG theorems for both stochastic PG \cite{Degris2012OffPAC} and deterministic PG \cite{Silver2014DPG},
the $\nabla_{\theta}q^{\pi_{\theta}}(s,a)$ is dropped and only approximate gradients are used \cite{Faccio20PVFs}:
\begin{equation}
\begin{aligned}
    \nabla_{\theta}J(\pi_{\theta}) = \mathbb{E}_{s \sim d_{\mu}, a \sim \pi_{\theta}} \left[ \nabla_{\theta} \log \pi_{\theta}(a|s)q^{\pi_{\theta}(s,a)} + \nabla_{\theta} q^{\pi_{\theta}(s,a)} \right] \approx & \
    \mathbb{E}_{s \sim d_{\mu}, a \sim \pi_{\theta}} \left[ \nabla_{\theta} \log \pi_{\theta}(a|s)q^{\pi_{\theta}(s,a)}\right], \\
    \nabla_{\theta}J(\pi_{\theta}) = \mathbb{E}_{s \sim d_{\mu}} \left[ \nabla_{\theta} \pi_{\theta}(a|s) \nabla_{a}q^{\pi_{\theta}(s,a)}|_{a=\pi_{\theta}(s)} + \nabla_{\theta} q^{\pi_{\theta}(s,a)}|_{a=\pi_{\theta}(s)} \right] \approx & \
    \mathbb{E}_{s \sim d_{\mu}} \left[ \nabla_{\theta} \pi_{\theta}(a|s) \nabla_{a}q^{\pi_{\theta}(s,a)}|_{a=\pi_{\theta}(s)} \right], \\
\end{aligned}
\end{equation}
where $\mu$ is the behavior policy for off-policy sampling and $d_{\mu}$ is stationary state distribution of $\mu$.

Our related works, PVN \cite{Harb20PENs} and PVFs \cite{Faccio20PVFs} make use of the policy-based extension of conventional value function to let the extended value function be differentiable with policy parameters.
We denote such new gradients as the policy \textbf{g}radients by backpropagting \textbf{t}hrough the differentiable \textbf{p}olicy \textbf{i}nput, shortly GTPI,
with a general formal definition below:
\begin{defini}
\label{def:GTPI}
Consider a policy-extended value function $\mathbb{Q}(s,a,\chi_{\pi_{\theta}})$, where $\chi_{\pi_{\theta}}$ is a differentiable representation of policy $\pi_{\theta}$, obtained by representation function $g$, i.e., $\chi_{\pi_{\theta}} = g(\pi_{\theta})$.
We define GTPI to be the policy gradients $\nabla_{\theta}\mathbb{V}(s,\chi_{\pi_{\theta}})= \nabla_{\theta}\pi_{\theta} \nabla_{\pi_{\theta}}g(\pi_{\theta}) \nabla_{\chi_{\pi_{\theta}}} \mathbb{V}(s,\chi_{\pi_{\theta}})|_{\chi_{\pi_{\theta}} = g(\pi_{\theta})}$.
\end{defini}
According to Definition \ref{def:GTPI}, we can conclude the gradients used in PVN and PVFs by only considering different representation function $g$ used.
As mentioned in the main body of this paper, PVFs use policy network parameters as a straightforward representation.
PVN uses Network Fingerprint, i.e., concatenation of policy distribution under a set of representative states.
With GTPI, new ways of policy optimization can be obtained, as the ones proposed and evaluated in PVN and PVFs papers.
In general, we view such policy optimization based on GTPI as policy optimization in policy representation space.

We can see a few connections here:
\begin{itemize}
    \item First, we can view DPG \cite{Silver2014DPG} $\nabla_{\theta} \pi_{\theta}(a|s) \nabla_{a}q^{\pi_{\theta}(s,a)}|_{a=\pi_{\theta}(s)}$ as a special GTPI with $\chi_{\pi_{\theta}}(s) = a = \pi_{\theta}(s)$ as a local policy representation to optimize policy for state $s$.
    In this view, PVFs can be considered to perform a global policy update;
    while PVN is like to perform intermediate policy update on the states in the set of Network Fingerprint.
    
    \item Policy optimization in policy representation space can have a connection to gradient-based Bayesian Optimization (BO) \cite{Notin2021BlackBox,Grosnit2021HighD},
    policy (parameter) is a high-dimensional variable we want to optimize.
    According to the literature of BO, GTPI can face some optimization challenges.
    The objective function (i.e., policy-extended value function) is non-trivial to approximate due to the small number of policy samples, noisy reward signals, high-dimensional policy parameters. Without high-quality estimation of objective function, gradient optimization direction can be incorrect and noisy.
    A gradient issue on \textit{compatibility} \cite{Lillicrap2015DDPG,DOroJ20UsefulCritic} in such cases should also be considered.
    
\end{itemize}

Towards the scalable utilization of GTPI, an effective and efficient (e.g., low-dimensional) representation is necessary.
This adds more difficulties in using GTPI for policy optimization.
First, learning policy representation adds non-stationarity in the process of objective function approximation.
Moreover, the extrapolation generalization desired by policy optimization is non-trivial to have any guarantee in policy representation space,
especially when objective function is only trained on finite policies (usually much fewer than state-action samples) thus further resulting in erroneous policy gradients.
Due to the poor knowledge on GTPI,
we do not resort to GTPI for the policy update in our algorithms but focus on utilizing value generalization for more efficient value estimation in GTPI. 
Actually, we believe the potentials of GTPI and policy optimization in representation space,
and we leave this for one of our main future direction of PeVFA.

\section{Experimental Details and Complete Results}
\label{app:exp_details}

\subsection{Experimental Details}
\label{app:experimental_details}

\paragraph{Environment.}
We conduct our experiments on commonly adopted OpenAI Gym\footnote{\url{http://gym.openai.com/}} continuous control tasks \citep{Brockman2016Gym,Todorov2012MuJoCo}.
We use the OpenAI Gym with version 0.9.1, the mujoco-py with version 0.5.4 and the MuJoCo products with version \texttt{MJPRO131}.
Our codes are implemented with Python 3.6 and \texttt{Tensorflow}.

\paragraph{Resources/Equipment Used and Time/Space Cost.}
Our experiments are mainly conducted on a NVIDIA GeForce RTX 2080 Ti with 11 GB memory.
It takes 3-4 hours for PPO-PeVFA and about 2 hours for PPO to finish a trial of 2 million steps MuJoCo training and evaluation.
In our experiments, about 6 trials are running simultaneously on the same GPU.
Since PPO itself is quite light-weight (than off-policy RL agents like DDPG), computational cost is not a worry in our experiments.

For the additional space cost, PeVFA needs to store historical policies and their interaction experiences.
The interaction experiences (i.e., trajectories) are stored in a replay buffer as done conventionally in an off-policy RL algorithm.
Actually, we store the experiences in recent 200k time steps.
For the policies, about 1k to 2k policies in total are encountered during each run, for each of which we store the network parameters, about 5k-20k parameters depends on the state dimension of the environment.
In our experiments, the storage space for historical policies are negligible to the storage for experiences (e.g., 2 million interaction steps).

Overall, the time and space cost for PeVFA are influenced by the algorithm choice (e.g., policy representation training) and training scheme design.
There is much room for further improvement.

\paragraph{Implementation.}
We use Proximal Policy Optimization (PPO) \citep{Schulman2017PPO} with Generalized Advantage Estimator (GAE) \citep{Schulman2016GAE} as our baseline algorithm.
Recent works \citep{EngstromISTJRM20ImpMatters,Andrychowicz20OPMatters} point out code-level optimizations influence the performance of PPO a lot.
For a fair comparison and clear evaluation, we perform no code-level optimization in our experiments, e.g., state standardization, reward scaling, gradient clipping, parameter sharing and etc.
Our proposed algorithm PPO-PeVFA is implemented based on PPO, 
which only differs at the replacement for conventional value function network with PeVFA network.
Policy network is a 2-layer MLP with 64 units per layer and ReLU activation, 
outputting a Gaussian policy, i.e., a tanh-activated mean along with a state-independent vector-parameterized log standard deviation.
For PPO, the conventional value network $V_{\phi}$(s) (VFA) 
is a 2-layer 128-unit ReLU-activated MLP 
with state as input and value as output.
For PPO-PeVFA, 
the PeVFA network $\mathbb{V}_{\theta}(s,\chi_{\pi})$ takes as input state and policy representation $\chi_{\pi}$
which has the dimensionality of 64,
with the structure illustrated in Fig.\ref{figure:pevfa_architecture}.
We do not use parameter sharing between policy and value function approximators for more clear evaluation.

\paragraph{Training and Evaluation.}
We use Monte Carlo returns for value approximation.
In contrast to conventional VFA $V_{\phi}$ which approximates the value of current policy (e.g., Algorithm \ref{algorithm:PG_GPI}),
PeVFA $\mathbb{V}_{\theta}(s,\chi_{\pi})$ 
is additionally trained to approximate the values of 
all historical policies ($\{\pi_i\}_{i=0}^{t}$) along the policy improvement path (e.g., Algorithm \ref{algorithm:PG_PeVFA}).
The policy network parameterized by $\omega$ is then updated with following loss function:
\begin{equation}
\begin{aligned}
    \mathcal{L}^{\text{PPO}}(\omega) & =  - \mathbb{E}_{\pi_{\omega^{-}}} \left[ \min \left(\rho_t \hat{A}(s_t,a_t), \text{clip}(\rho_t, 1 - \epsilon, 1 + \epsilon) \hat{A}(s_t,a_t) \right) \right],
\end{aligned}
\label{eqation:PPO_objective}
\end{equation}
where $\hat{A}(s_t,a_t)$ is advantage estimation of old policy $\pi_{\omega^{-}}$, which is calculated by GAE based on conventional VFA $V_{\phi}$ or PeVFA $\mathbb{V}_{\theta}(s,\chi_{\pi})$ respectively,
and $\rho_t = \frac{\pi_{\omega}(a_t, s_t)}{\pi_{\omega^{-}}(a_t, s_t)}$ is the importance sampling ratio.
Note that both PPO and PPO-PeVFA update the policy according to Eq.\ref{eqation:PPO_objective} and only differ at advantage estimation based on conventional VFA $V_{\phi}$ or PeVFA $\mathbb{V}_{\theta}(s,\chi_{\pi})$.
This ensures that the performance difference comes only from different approximation of policy values.
Common learning parameters for PPO and PPO-PeVFA are shown in Tab.\ref{table:common_hyperparameter}.
For each iteration, we update value function approximators first and then the policy with updated values.
Such a training scheme is used for both PPO and PPO-PeVFA.
For evaluation, we evaluate the learning algorithm every 20k time steps, averaging the returns of 10 episodes.
Fewer evaluation points are selected and smoothed over neighbors and then plotted in our learning curves below.

\begin{table}[ht]
  \caption{Common hyperparameter choices of PPO and PPO-PeVFA.
  }
  \vspace{0.2cm}
  \centering
  \scalebox{1.0}{
  \begin{tabular}{c|c}
    \toprule
    \multicolumn{2}{c}{Hyperparameters for PPO \& PPO-PeVFA} \\
    \midrule
    Policy Learning Rate & 10$^{-4}$ \\
    Value Learning Rate & 10$^{-3}$ \\
    Clipping Range Parameter ($\epsilon$) & 0.2 \\
    GAE Parameter ($\lambda$) & 0.95 \\
    \midrule
    Discount Factor ($\gamma$) & 0.99  \\
    On-policy Samples Per Iteration & 5 episodes or 2000 time steps\\
    Batch Size & 128 \\
    Actor Epoch & 10 \\
    Critic Epoch & 10 \\
    Optimizer & Adam \\
    \bottomrule
  \end{tabular}
  }
\label{table:common_hyperparameter}
\end{table}

\paragraph{Details for PPO-PeVFA.}
For PeVFA, the training process also involves value approximation of historical policies and learning of policy representation.
Training details are shown in Tab.\ref{table:detail_pevfa}.
PeVFA $\mathbb{V}_{\theta}(s,\chi_{\pi})$ is trained every 10 steps with a batch of 64 samples
from an experience buffer with size of 200k steps.
Policy representation model is trained at intervals of 10 or 20 steps depending on OPR or SPR adopted.
Due to 1k - 2k policies are collected in total in each trial, a relatively small batch size of policy is used.
For OPR, Random Mask (Fig.\ref{figure:opr_data_aug}) is performed on all weights and biases of policy network except for the output layer (i.e., mean and log-std).
For SPR, two buffers of state-action pairs are maintained for each policy:
a small one is sampled for calculating SPR and the relatively larger one is sampled for auxiliary training (policy recovery).

\begin{table}[ht]
  \caption{Training details for PPO-PeVFA, including value approximation of historical policies and policy representation learning.
  CL is abbreviation for Contrastive Learning and AUX is for auxiliary loss of policy recovery.
  In our experiments, grid search is performed for the best hyperparamter configuration 
  regarding terms with multiple alternatives (i.e., \{\}).
  }
  \vspace{0.2cm}
  \centering
  \scalebox{1.0}{
  \begin{tabular}{c|c}
    \toprule
    \multicolumn{2}{c}{Value Approximation for Historical Policies} \\
    \midrule
    Value Learning Rate & 10$^{-3}$ \\
    Training Frequency & Every 10 time steps \\
    Batch Size & 64 \\
    Experience Buffer Size & 200k (steps) \\
    \midrule
    \multicolumn{2}{c}{Policy Representation Learning} \\
    \midrule
    Training Frequency & Every \{10, 20\} time steps \\
    Policy Num Per Batch & \{16, 32\} \\
    SPR $s,a$ Pair Num & \{200, 500\} \\
    CL Learning Rate & \{10$^{-3}$, 5 $\cdot$10$^{-4}$\} \\
    CL Momentum & \{5$\cdot$10$^{-2}$, 10$^{-2}$, 5$\cdot$10$^{-3}$\} \\
    CL Mask Ratio for OPR & \{0.1, 0.2\} \\
    CL Sample Ratio for SPR & 0.8 \\
    AUX Learning Rate & 10$^{-3}$ \\
    AUX Batch Size & \{128, 256\} \\
    \bottomrule
  \end{tabular}
  }
\label{table:detail_pevfa}
\end{table}

\subsection{Complete Learning Curves for Evaluation Results in Tab.\ref{table:overall_evaluations}}
\label{app:complete_learning_curves}
Corresponding to Tab.\ref{table:overall_evaluations},
an overall view of learning curves of all variants of PPO-PeVFA as well as baseline algorithms are shown in Fig.\ref{figure:exp_curves_overall}.
One can refer to Fig.\ref{figure:exp_curves_e2e} for a clearer view of the effects of PeVFA (with E2E policy representation),
and Fig.\ref{figure:exp_curves_cl}, \ref{figure:exp_curves_aux} 
for the effects of self-supervised policy representations, i.e., CL and AUX.

\begin{figure}
\centering
\subfigure[HalfCheetah-v1]{
\includegraphics[width=0.33\textwidth]{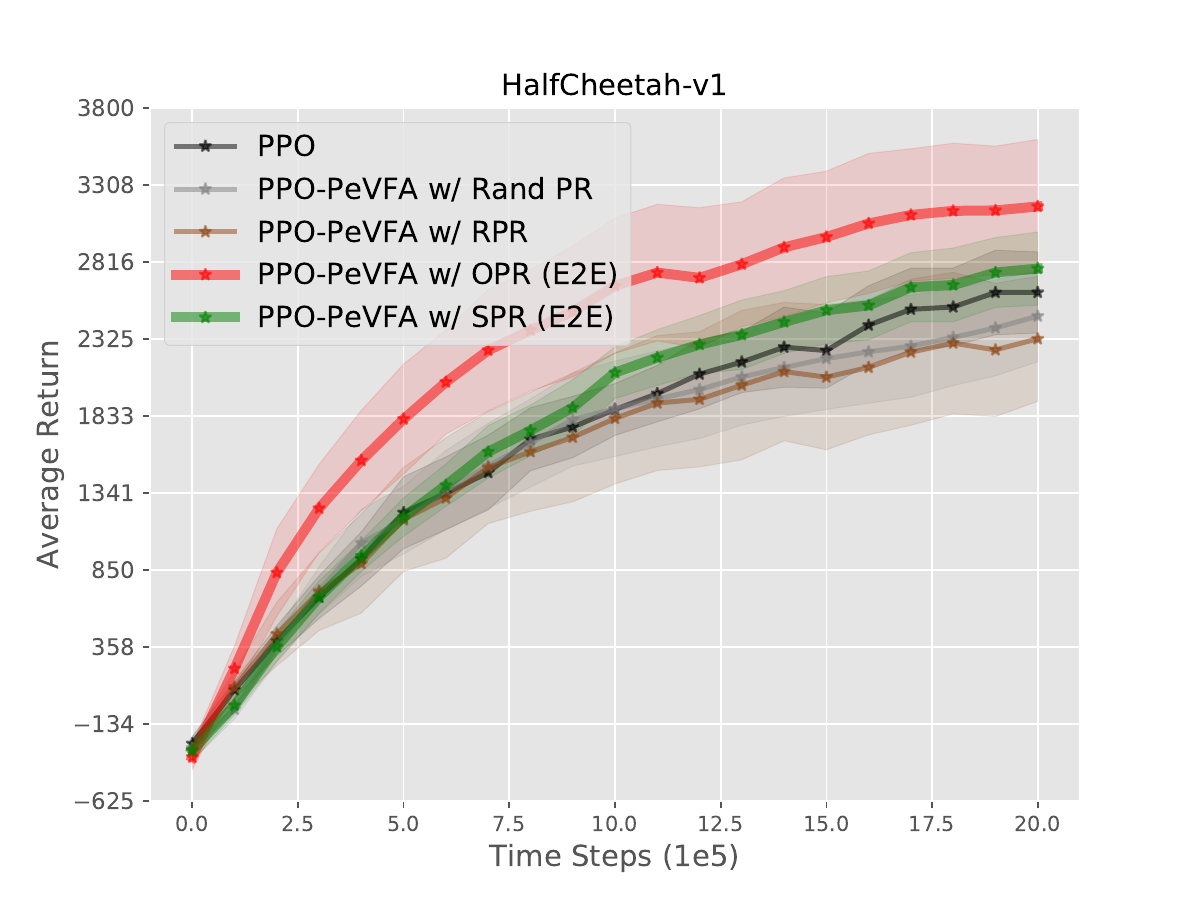}
}
\hspace{-0.5cm}
\subfigure[Hopper-v1]{
\includegraphics[width=0.33\textwidth]{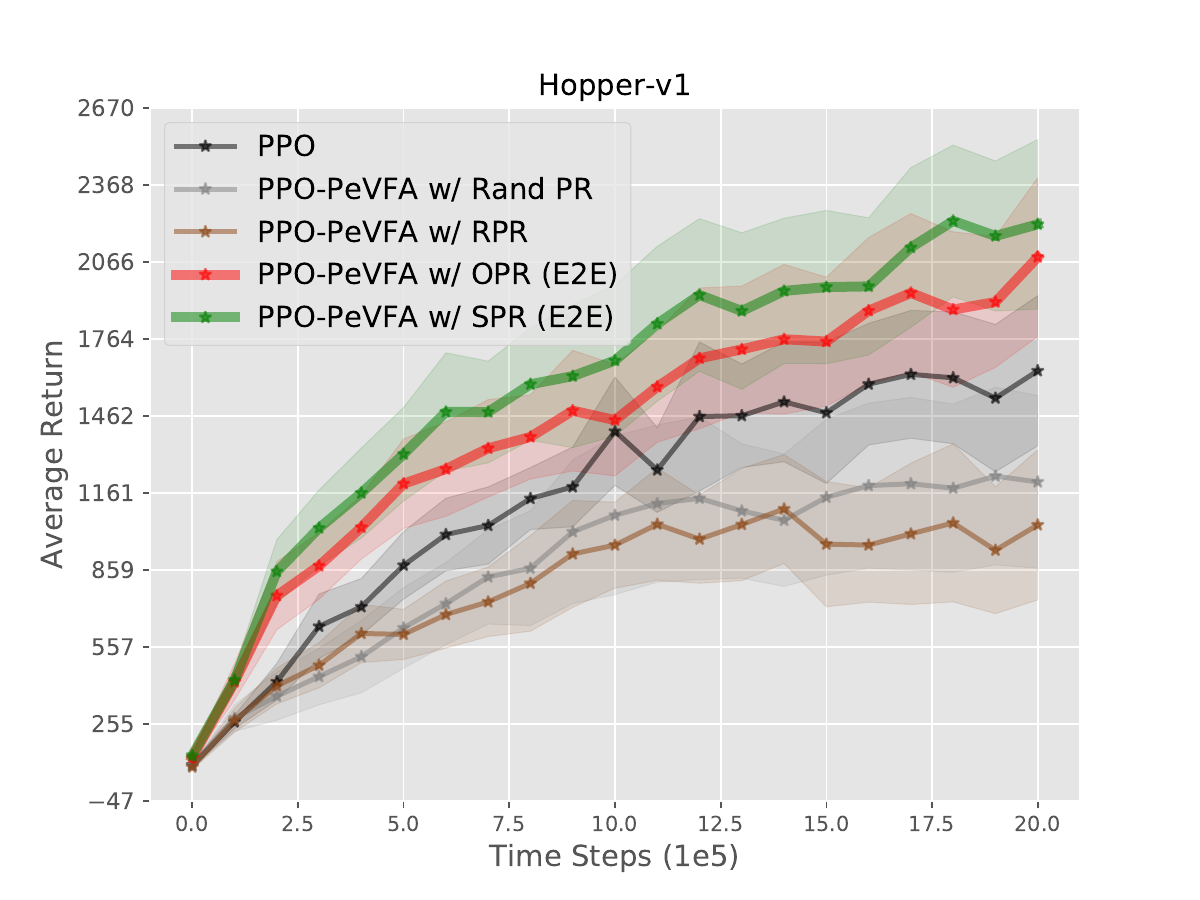}
}
\hspace{-0.5cm}
\subfigure[Walker2d-v1]{
\includegraphics[width=0.33\textwidth]{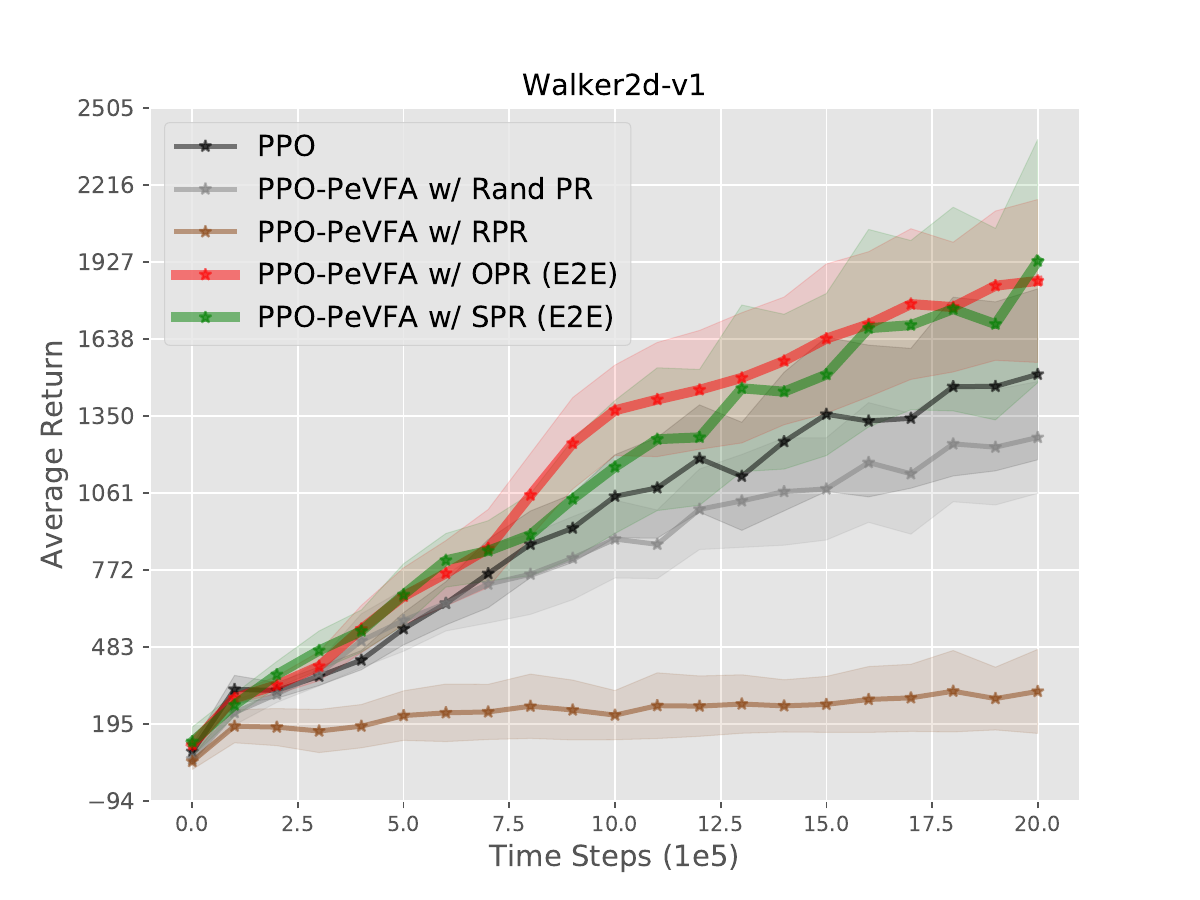}
}
\subfigure[Ant-v1]{
\includegraphics[width=0.33\textwidth]{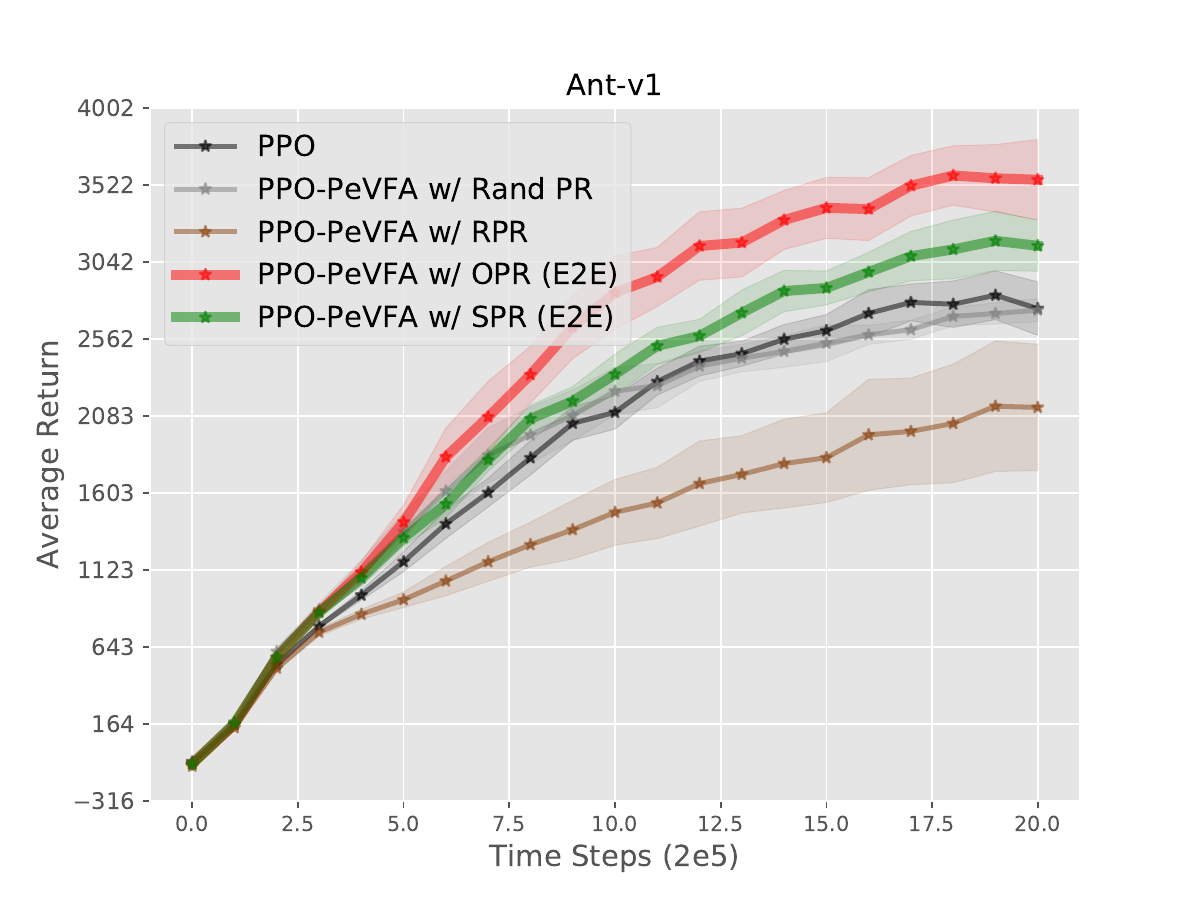}
}
\hspace{-0.5cm}
\subfigure[InvertedDoublePendulum-v1]{
\includegraphics[width=0.33\textwidth]{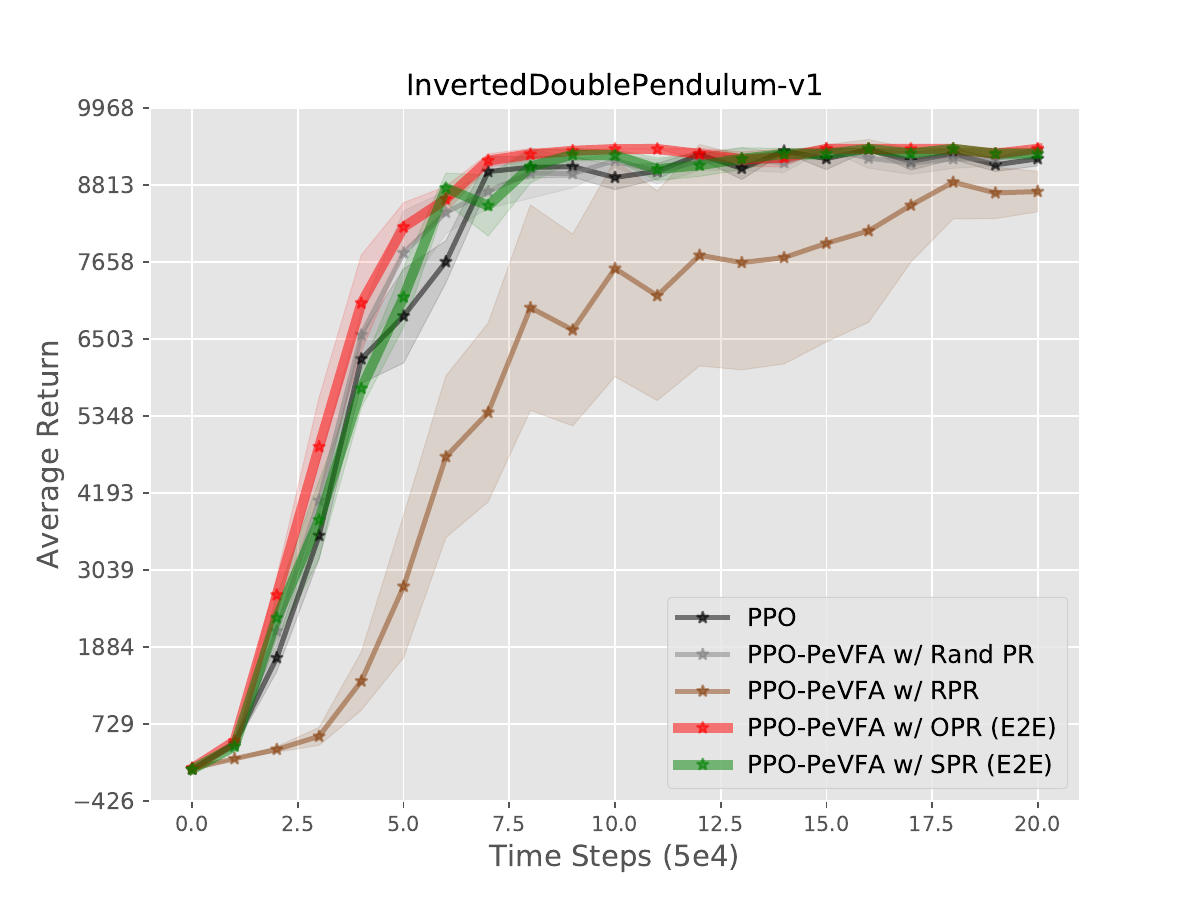}
}
\hspace{-0.5cm}
\subfigure[LunarLanderContinuous-v2]{
\includegraphics[width=0.33\textwidth]{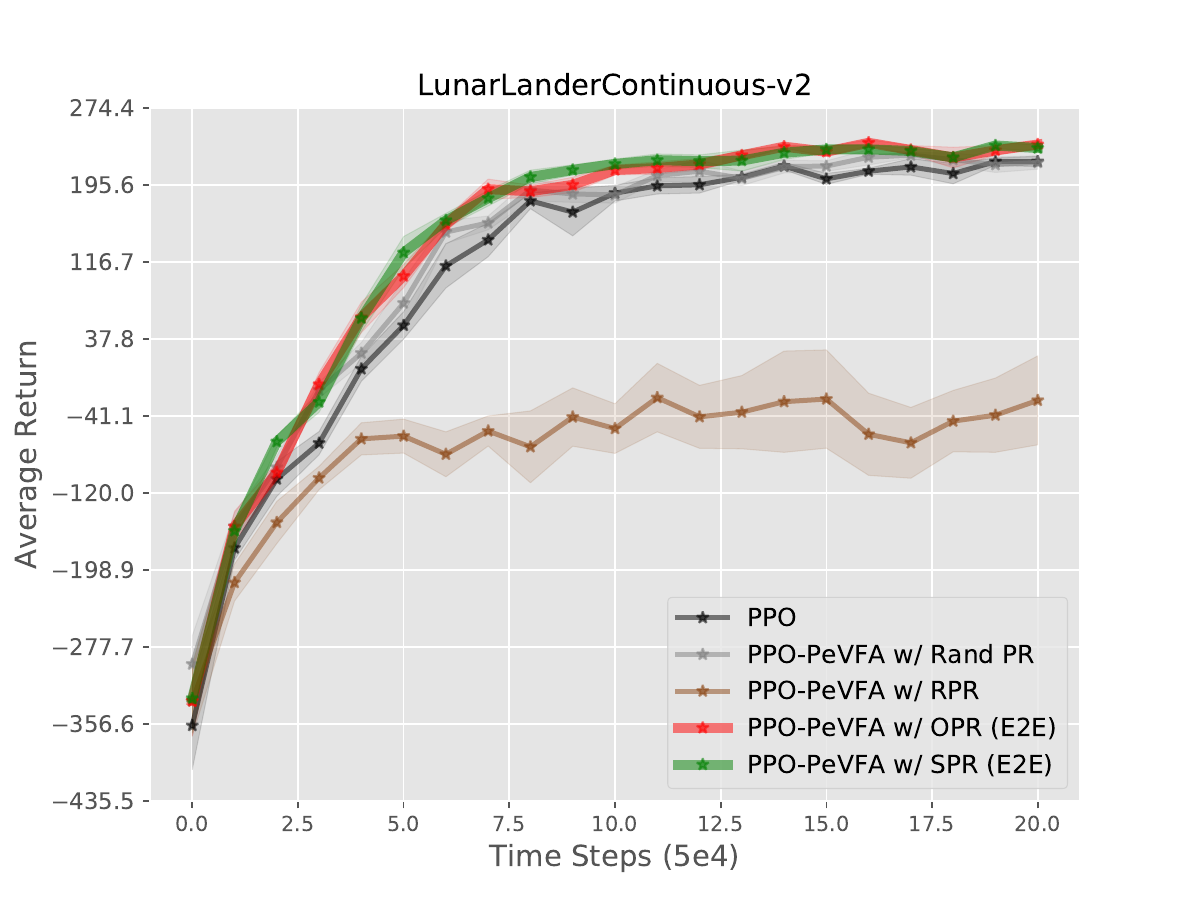}
}
\subfigure[InvertedPendulum-v1]{
\includegraphics[width=0.33\textwidth]{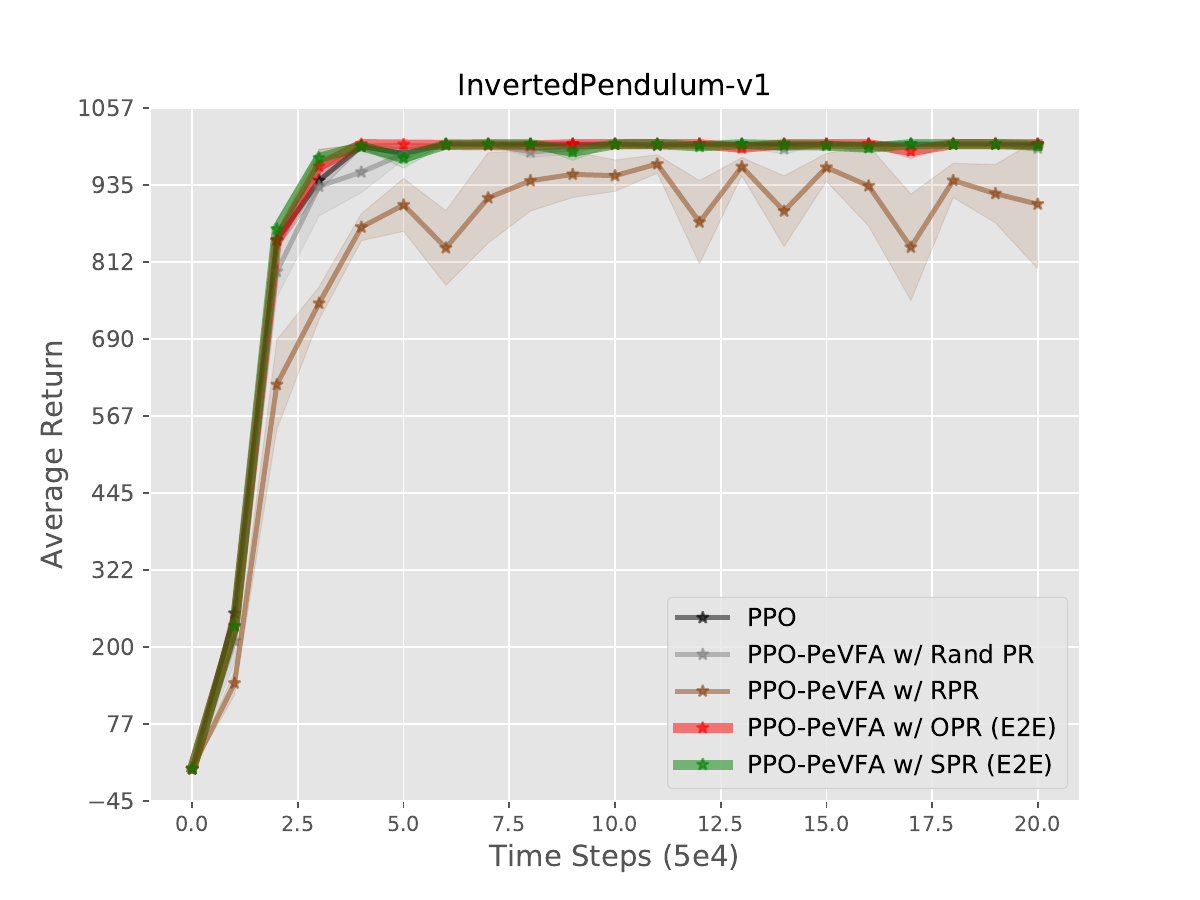}
}
\caption{
Evaluations of PPO-PeVFA with end-to-end (E2E) trained OPR and SPR in MuJoCo continuous control tasks.
The results demonstrate the effectiveness of PeVFA and two kinds of policy representation, 
answering the Question \ref{question:pevfa}.
The results are average returns and the shaded region denotes half a standard deviation over 10 trials.
}
\label{figure:exp_curves_e2e}
\end{figure}

\begin{figure}
\centering
\subfigure[HalfCheetah-v1]{
\includegraphics[width=0.33\textwidth]{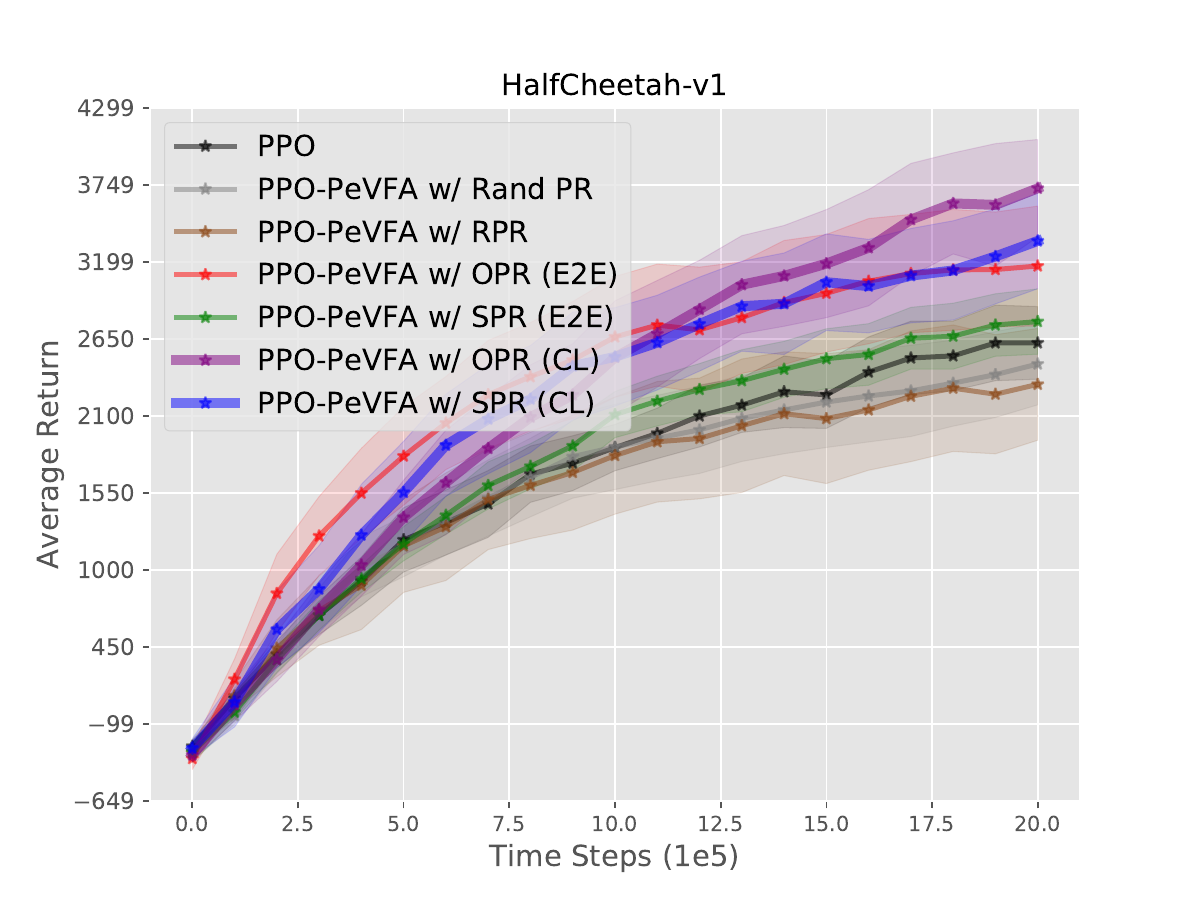}
}
\hspace{-0.5cm}
\subfigure[Hopper-v1]{
\includegraphics[width=0.33\textwidth]{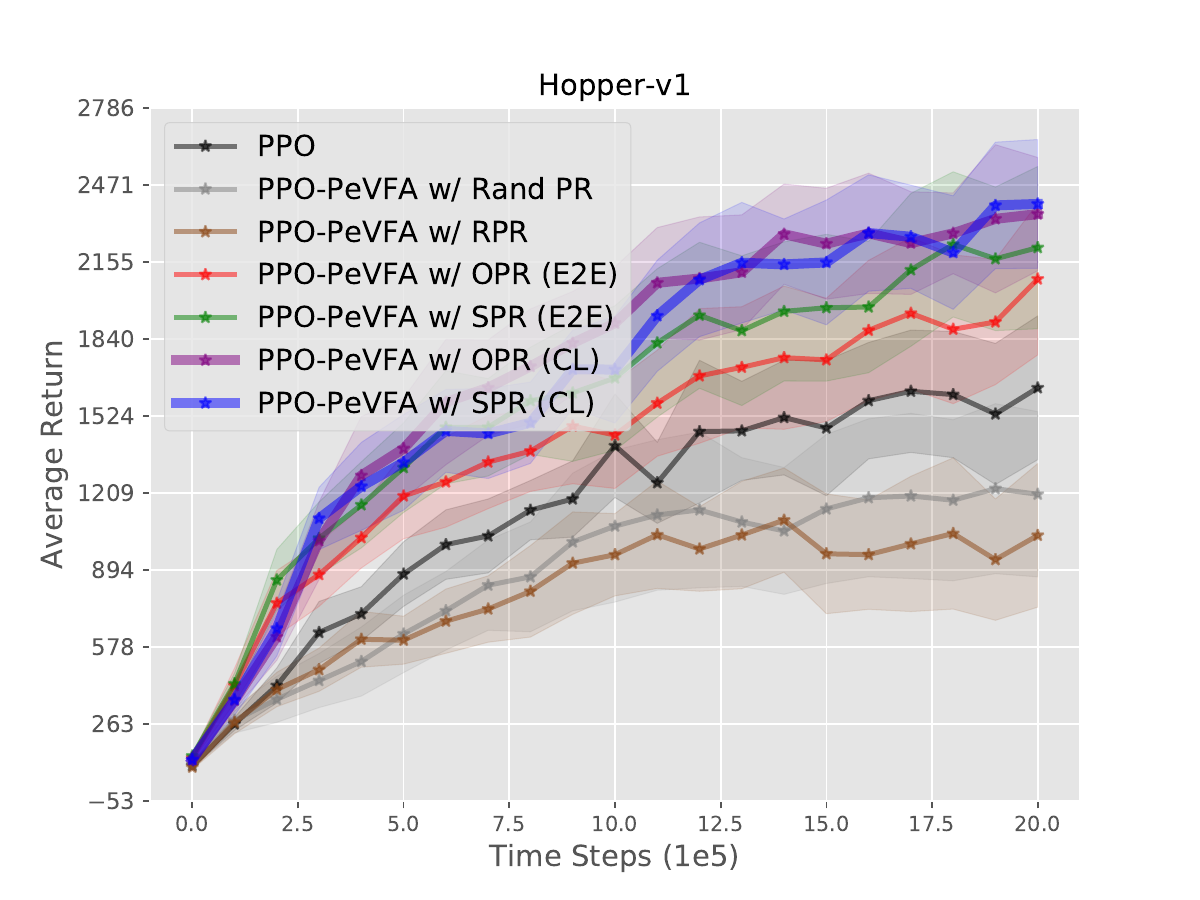}
}
\hspace{-0.5cm}
\subfigure[Walker2d-v1]{
\includegraphics[width=0.33\textwidth]{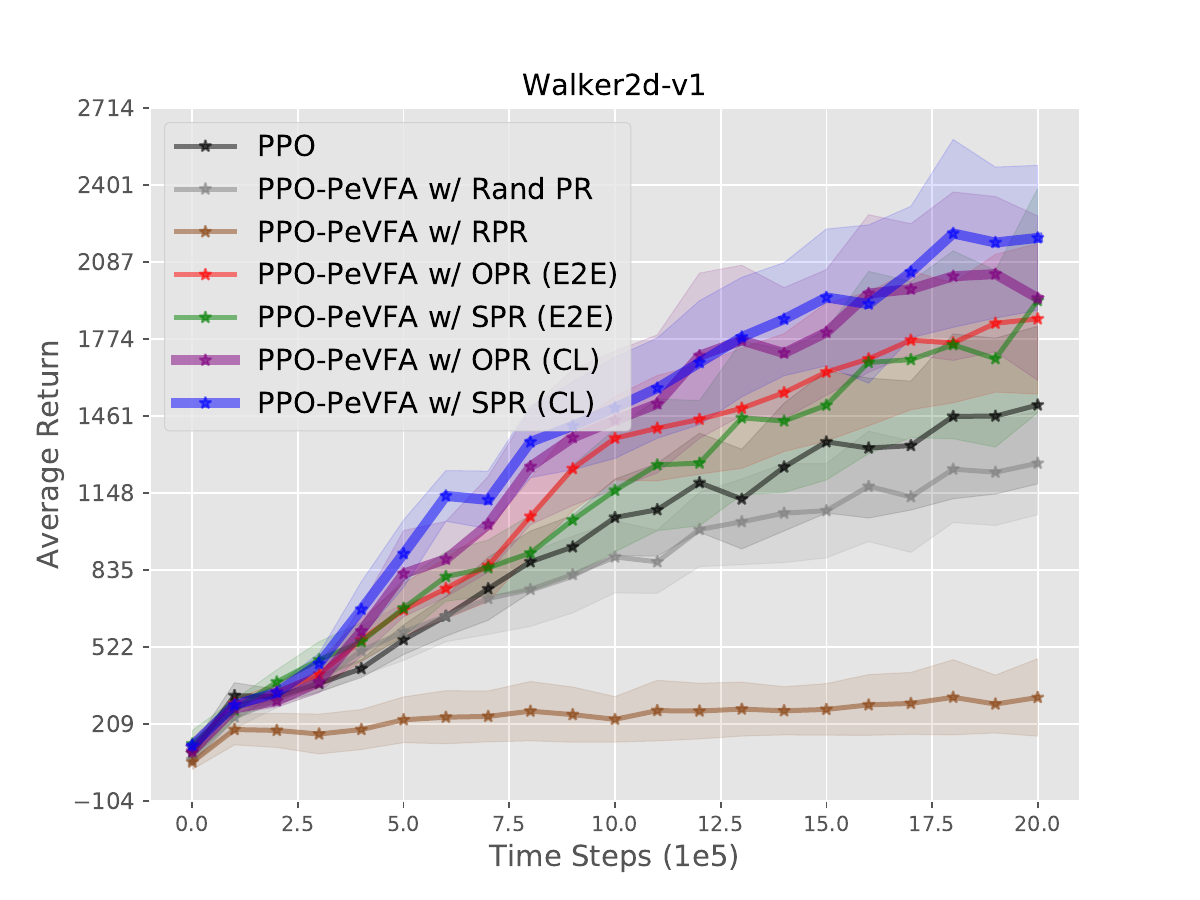}
}
\subfigure[Ant-v1]{
\includegraphics[width=0.33\textwidth]{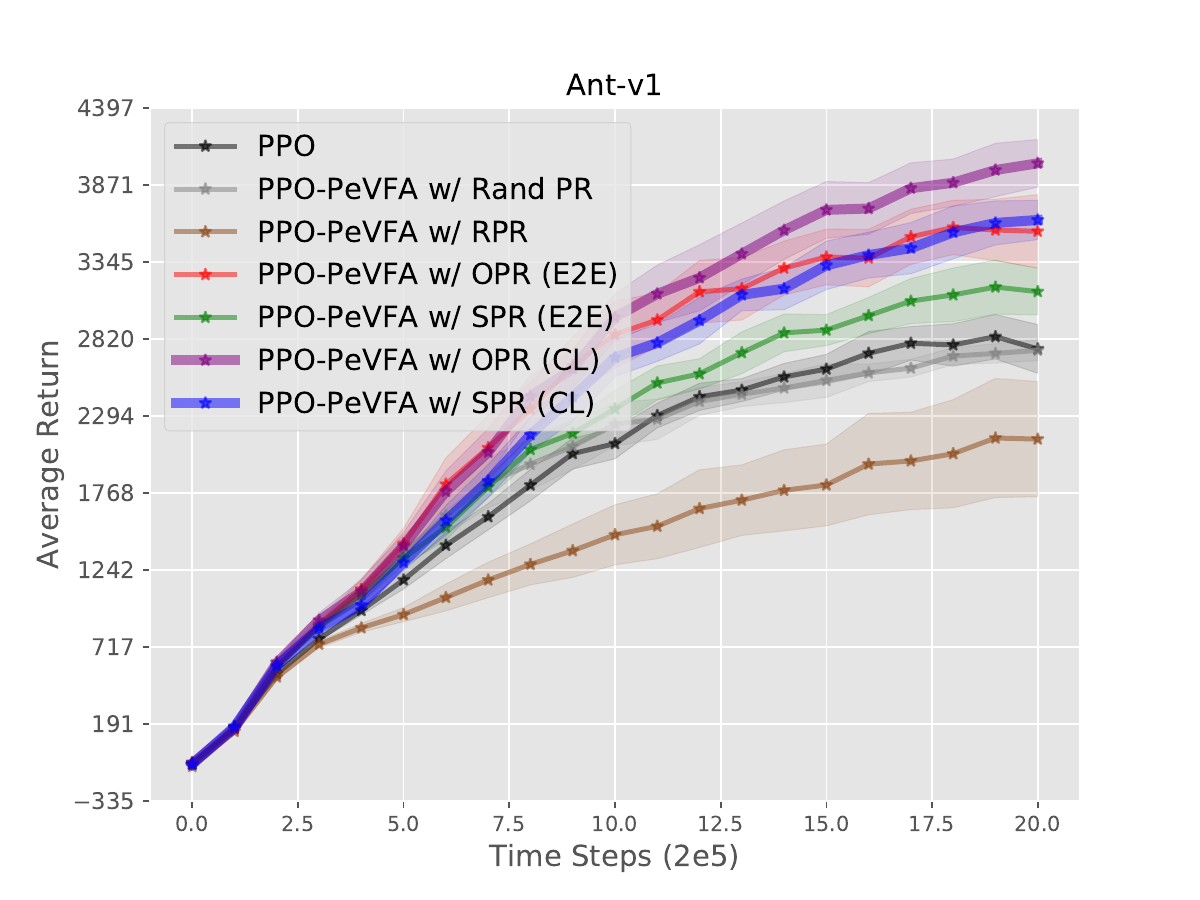}
}
\hspace{-0.5cm}
\subfigure[InvertedDoublePendulum-v1]{
\includegraphics[width=0.33\textwidth]{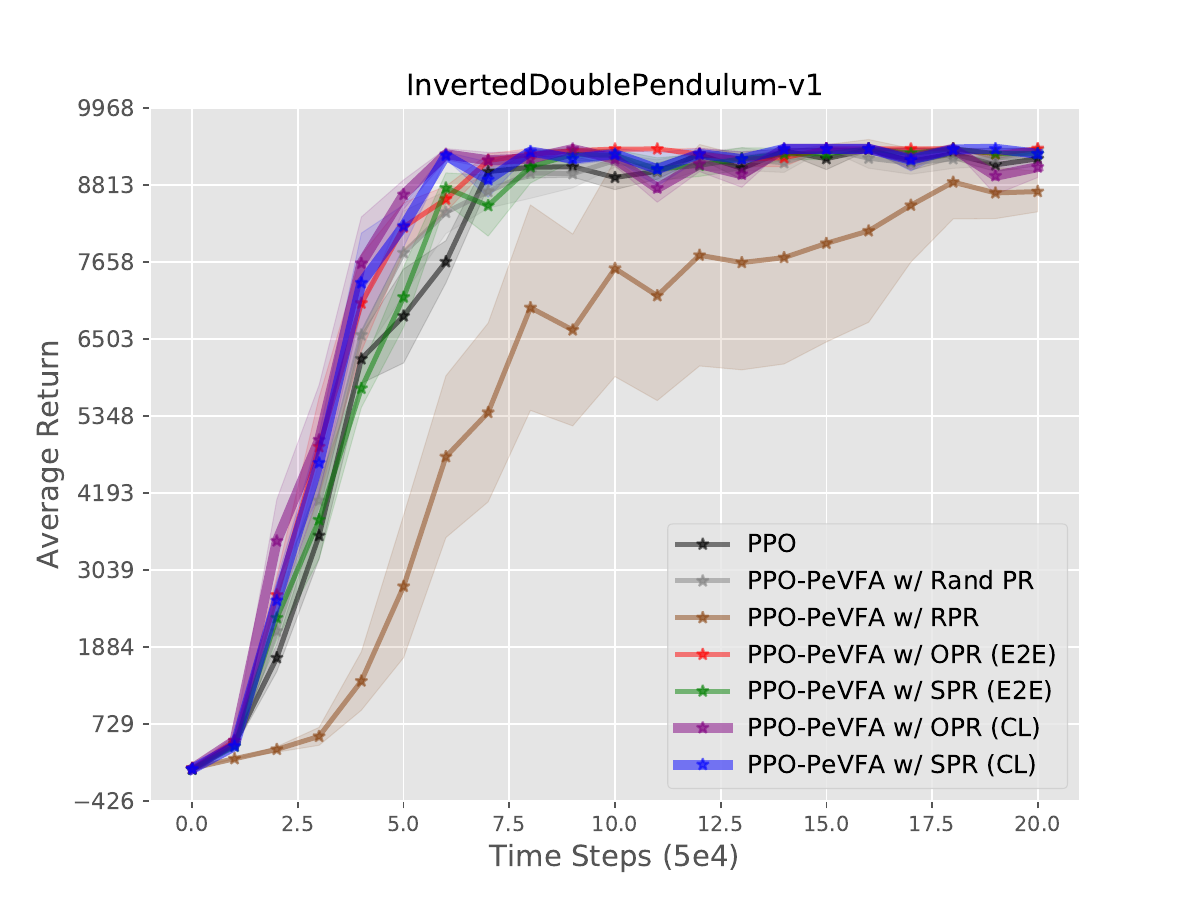}
}
\hspace{-0.5cm}
\subfigure[LunarLanderContinuous-v2]{
\includegraphics[width=0.33\textwidth]{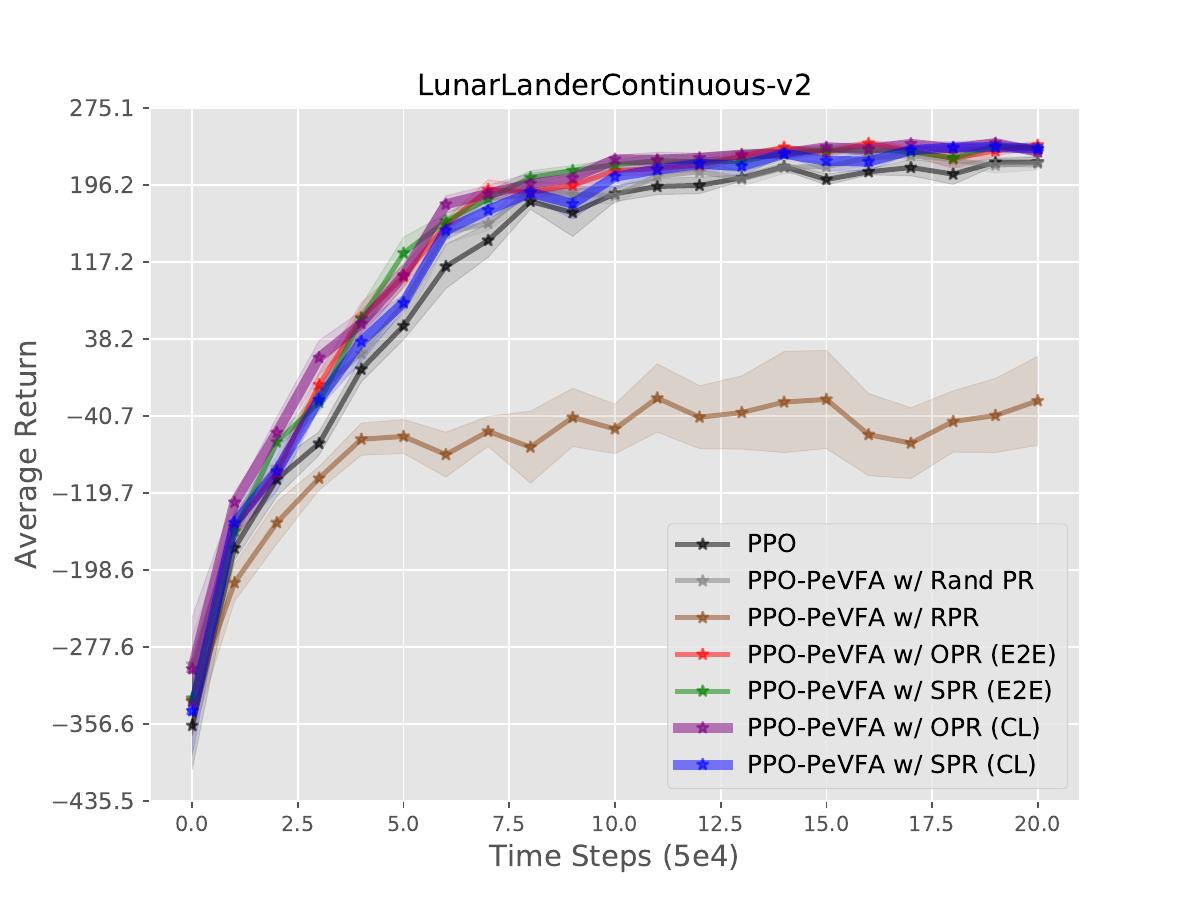}
}
\caption{
Evaluations of PPO-PeVFA with  OPR and SPR trained through contrastive learning (CL) in MuJoCo continuous control tasks.
The results are average returns and the shaded region denotes half a standard deviation over 10 trials.
}
\label{figure:exp_curves_cl}
\end{figure}

\begin{figure}
\centering
\subfigure[HalfCheetah-v1]{
\includegraphics[width=0.33\textwidth]{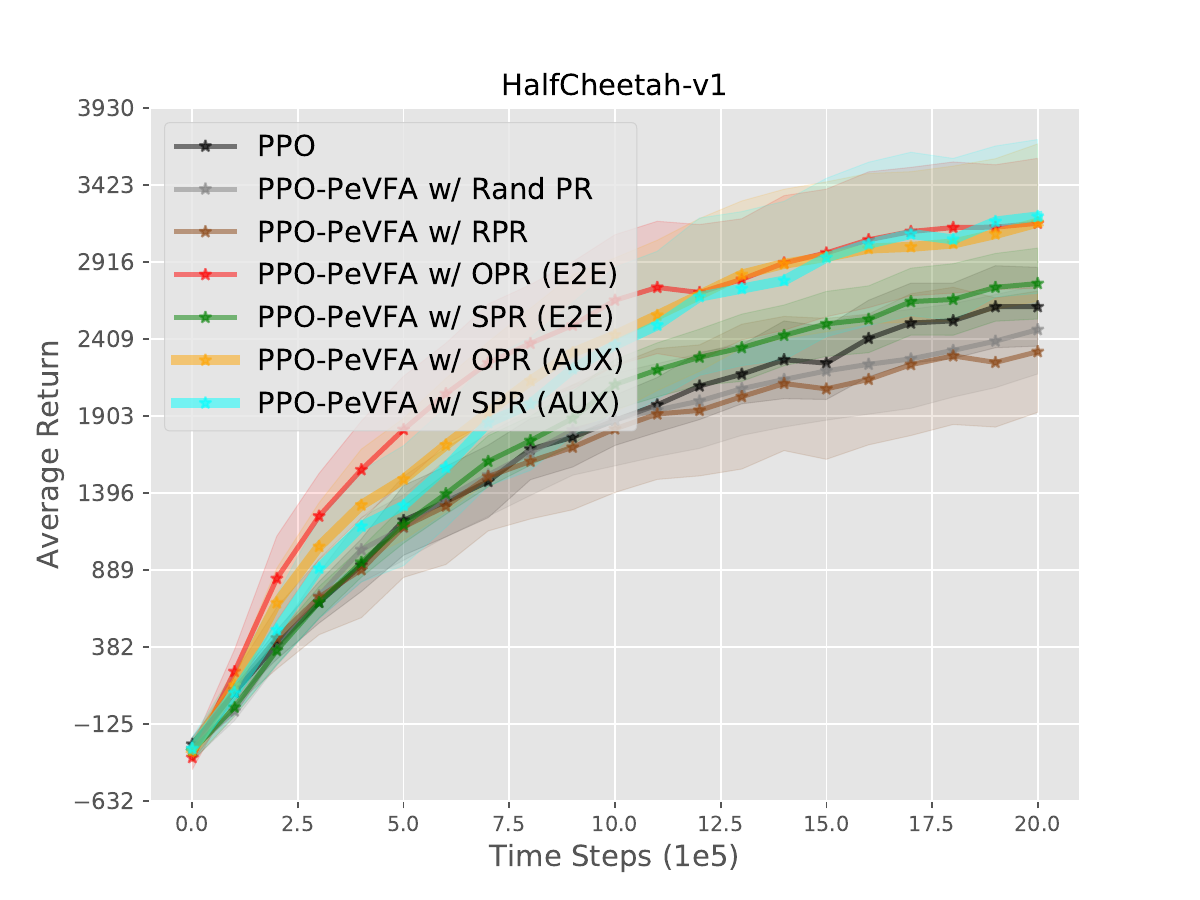}
}
\hspace{-0.5cm}
\subfigure[Hopper-v1]{
\includegraphics[width=0.33\textwidth]{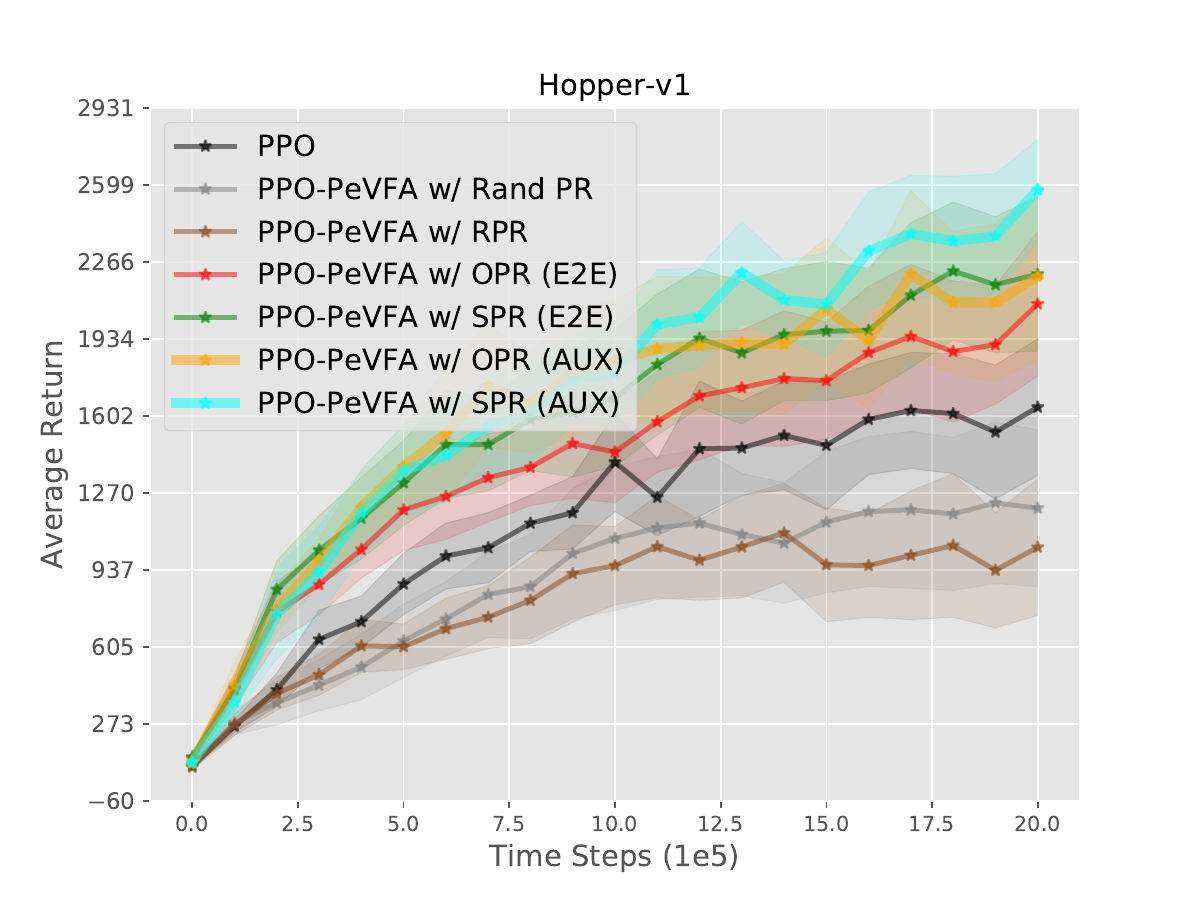}
}
\hspace{-0.5cm}
\subfigure[Walker2d-v1]{
\includegraphics[width=0.33\textwidth]{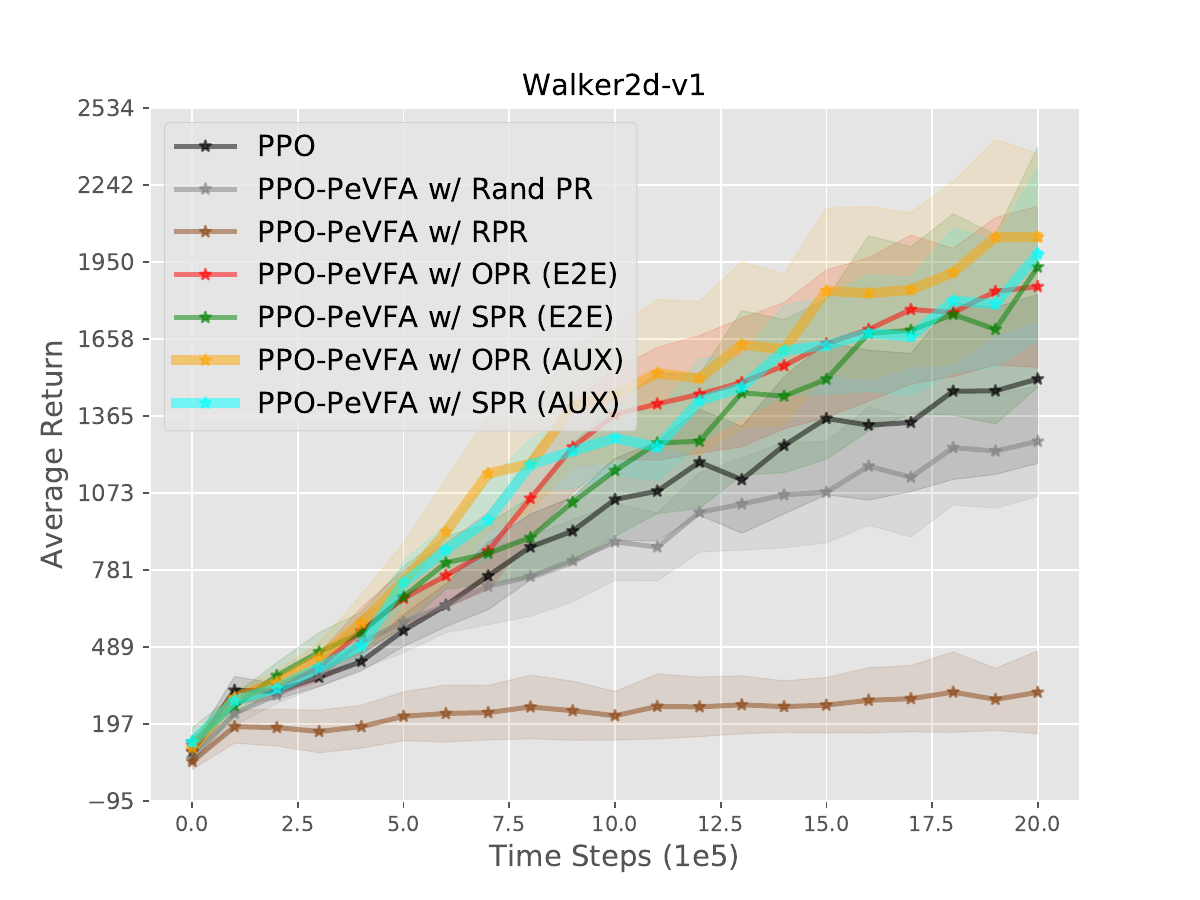}
}
\subfigure[Ant-v1]{
\includegraphics[width=0.33\textwidth]{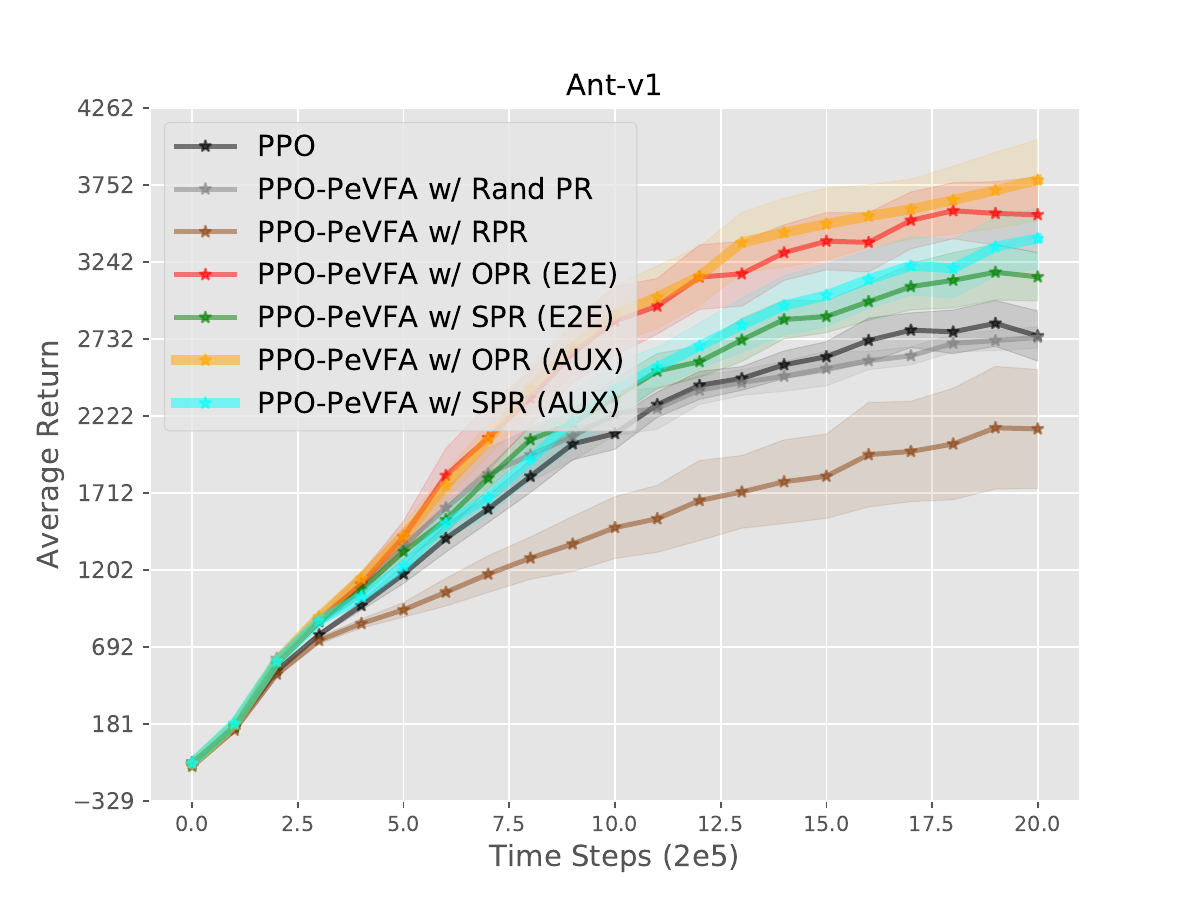}
}
\hspace{-0.5cm}
\subfigure[InvertedDoublePendulum-v1]{
\includegraphics[width=0.33\textwidth]{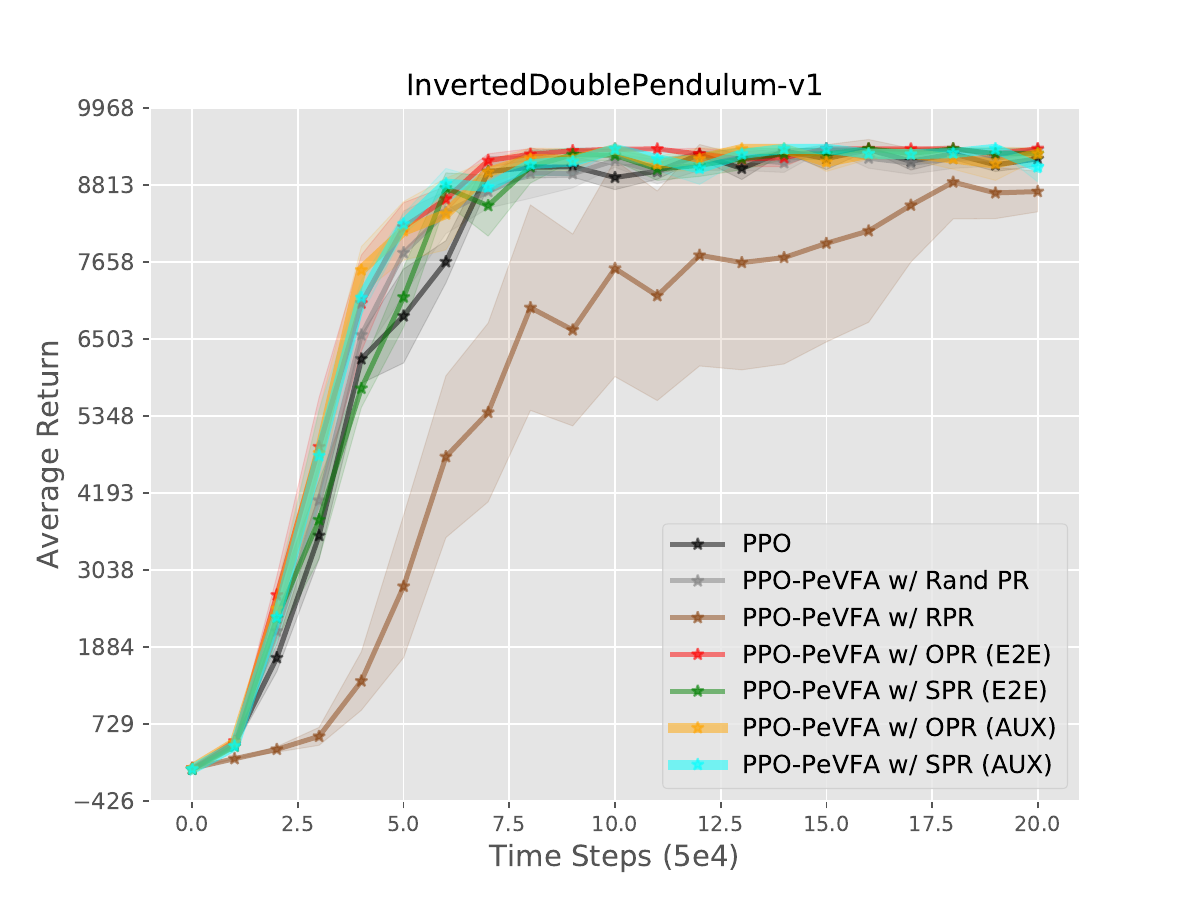}
}
\hspace{-0.5cm}
\subfigure[LunarLanderContinuous-v2]{
\includegraphics[width=0.33\textwidth]{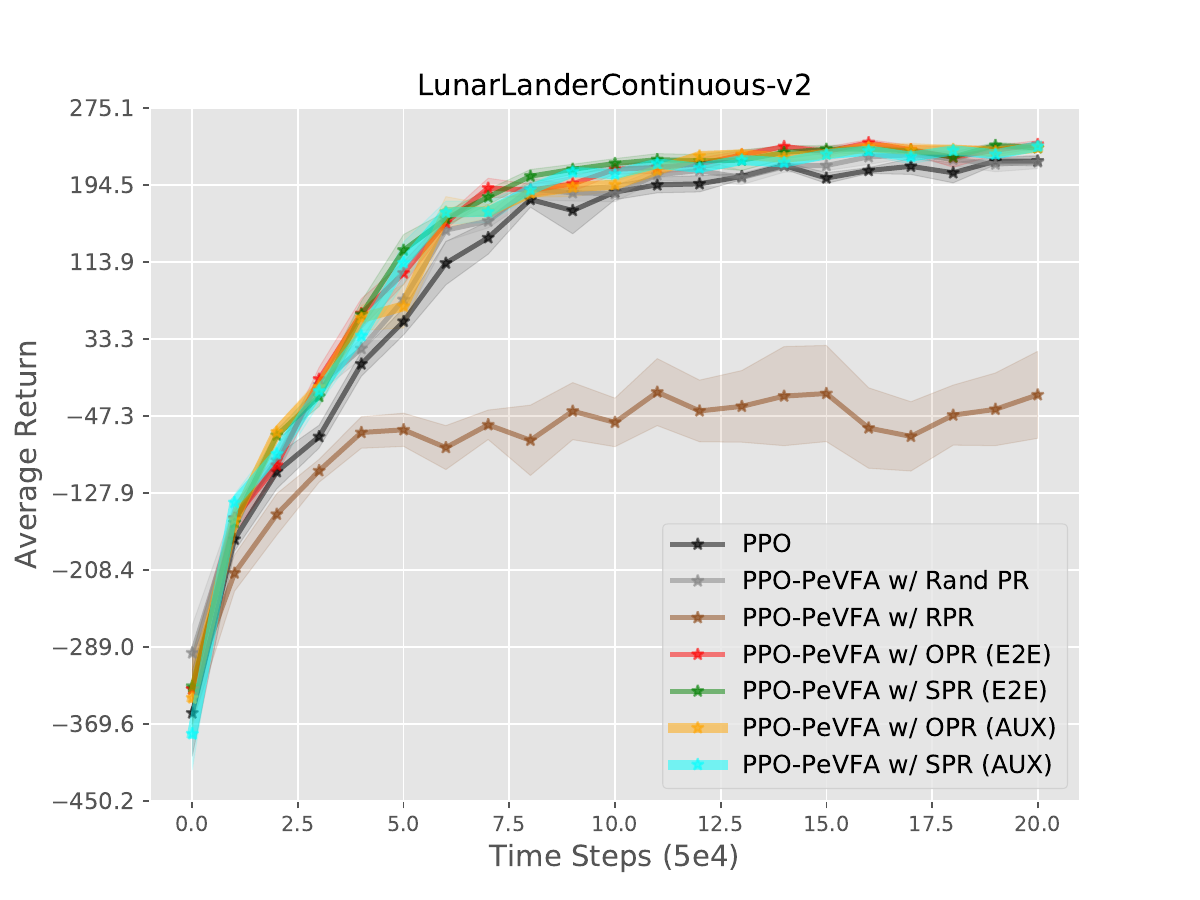}
}
\caption{
Evaluations of PPO-PeVFA with  OPR and SPR trained through auxiliary loss of policy recovery (AUX) in MuJoCo continuous control tasks.
The results are average returns and the shaded region denotes half a standard deviation over 10 trials.
}
\label{figure:exp_curves_aux}
\end{figure}

\begin{figure}[ht]
\centering
\subfigure[HalfCheetah-v1]{
\includegraphics[width=0.48\textwidth]{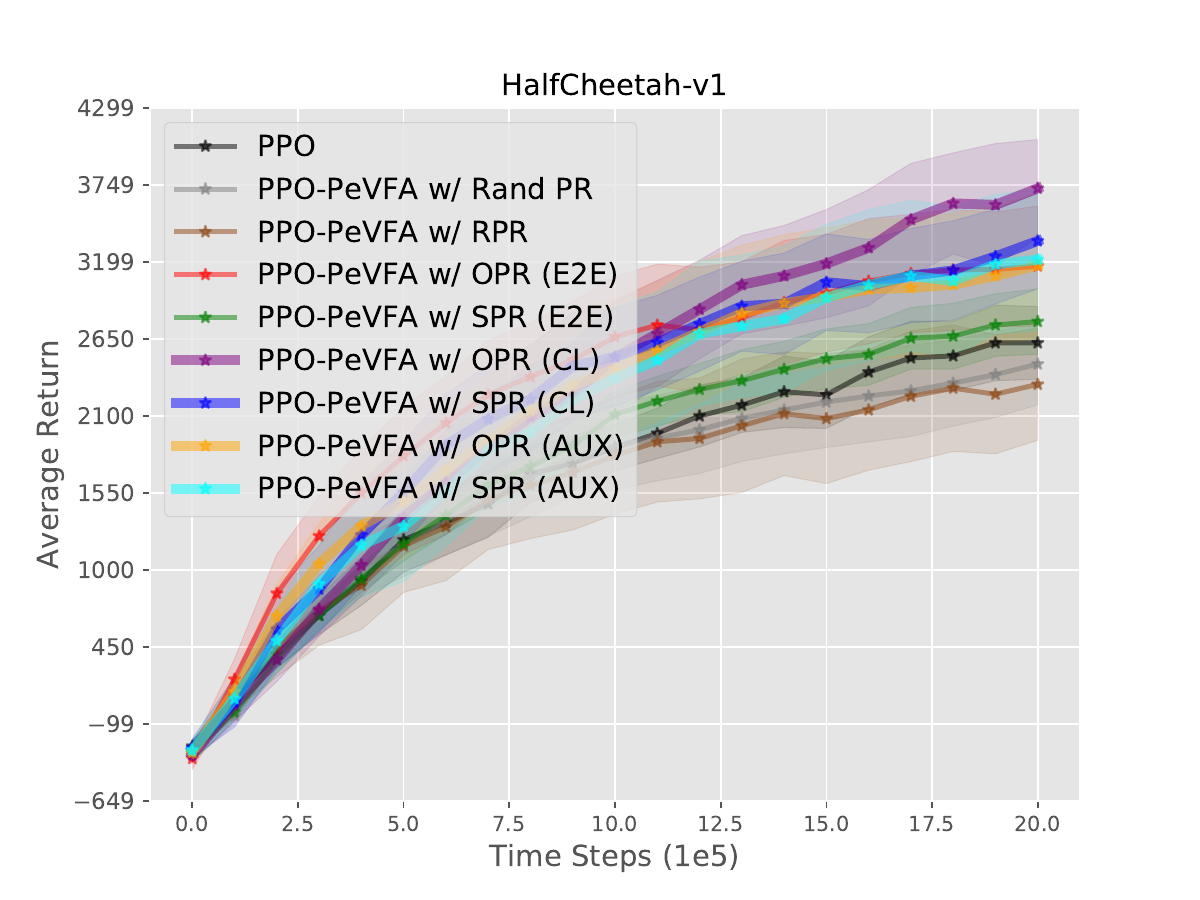}
}
\hspace{-0.8cm}
\subfigure[Hopper-v1]{
\includegraphics[width=0.48\textwidth]{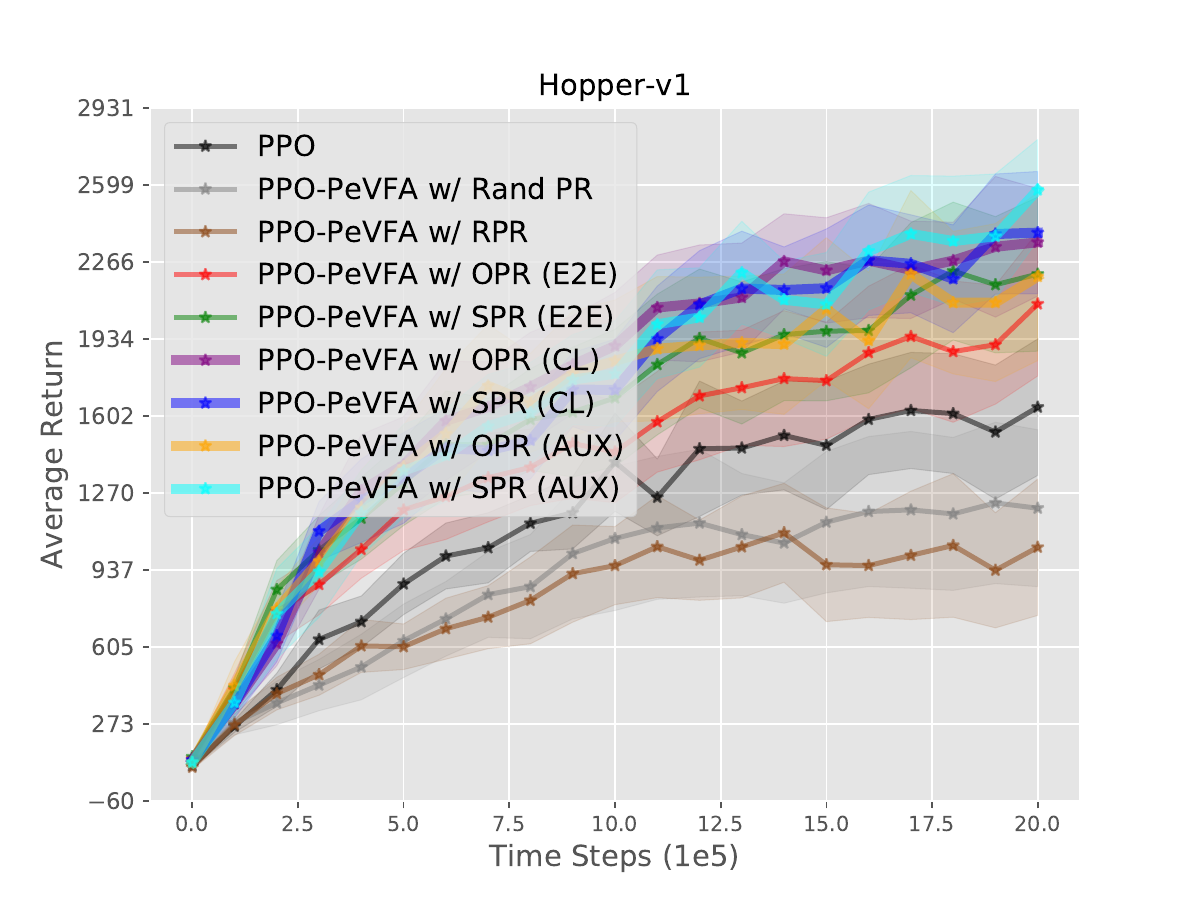}
}
\subfigure[Walker2d-v1]{
\includegraphics[width=0.48\textwidth]{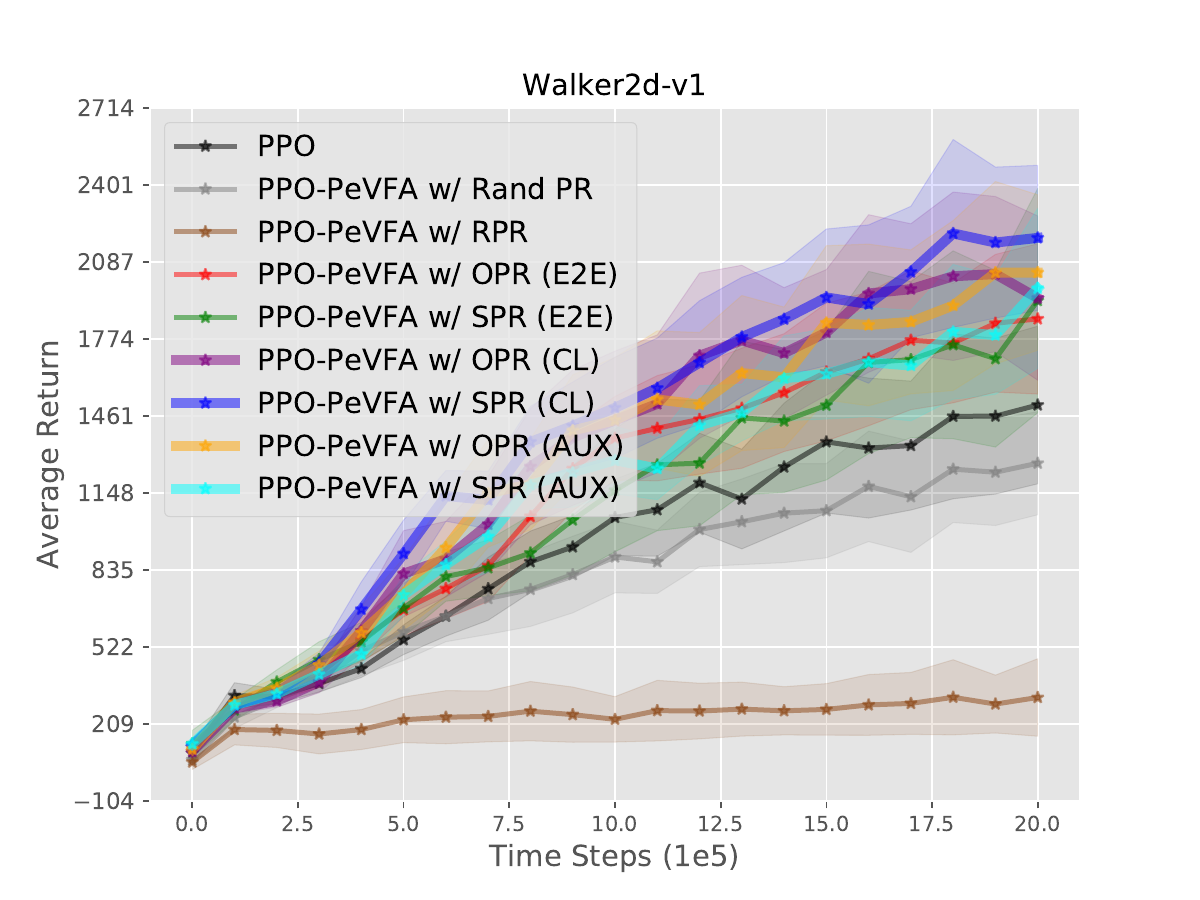}
}
\hspace{-0.8cm}
\subfigure[Ant-v1]{
\includegraphics[width=0.48\textwidth]{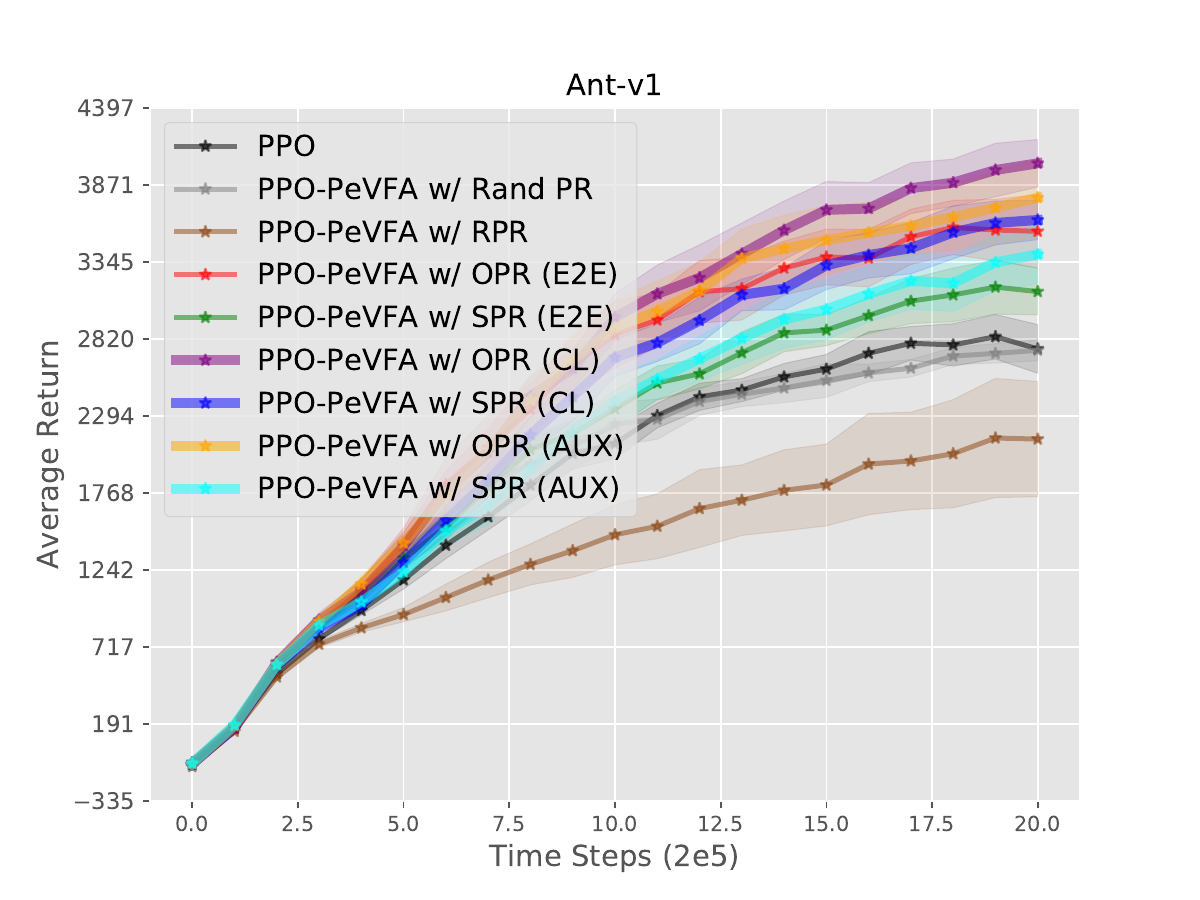}
}
\subfigure[InvertedDoublePendulum-v1]{
\includegraphics[width=0.48\textwidth]{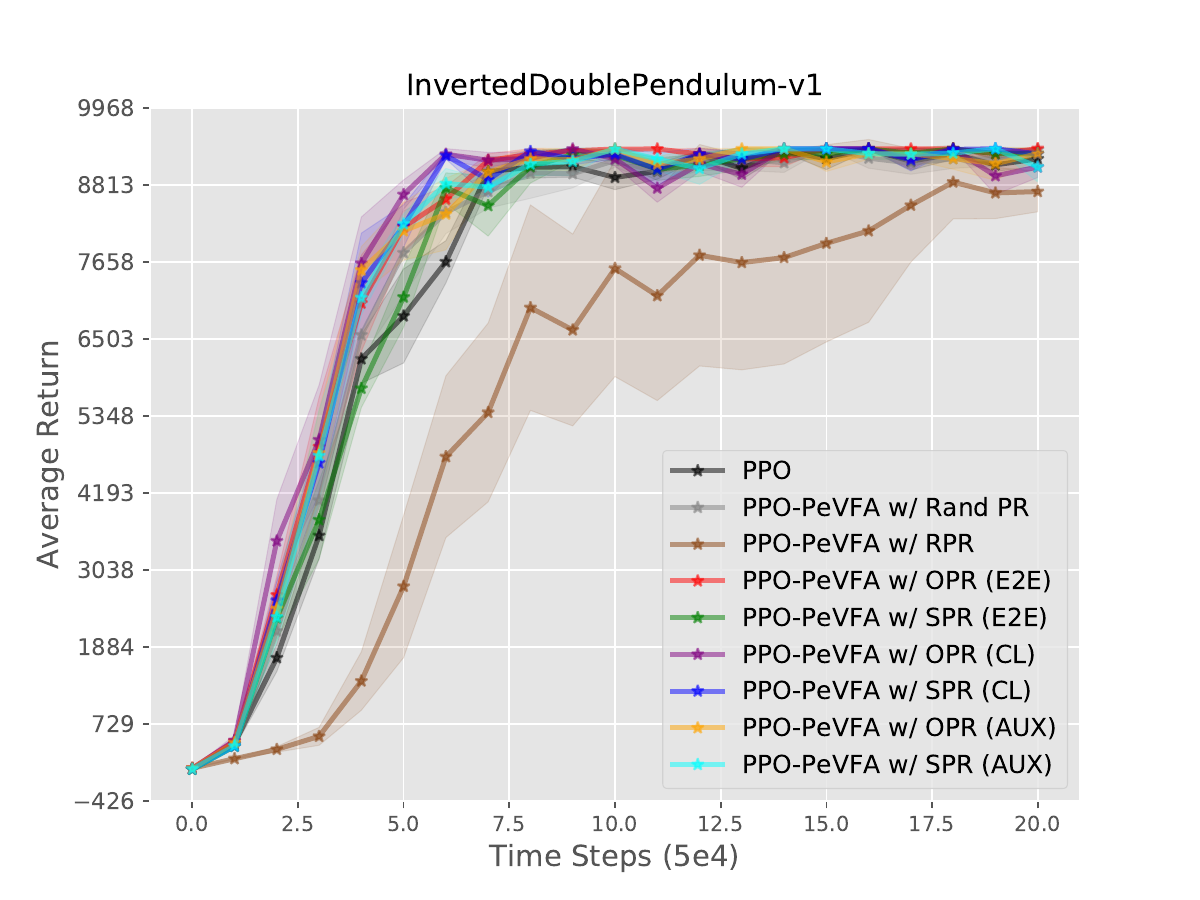}
}
\hspace{-0.8cm}
\subfigure[LunarLanderContinuous-v2]{
\includegraphics[width=0.48\textwidth]{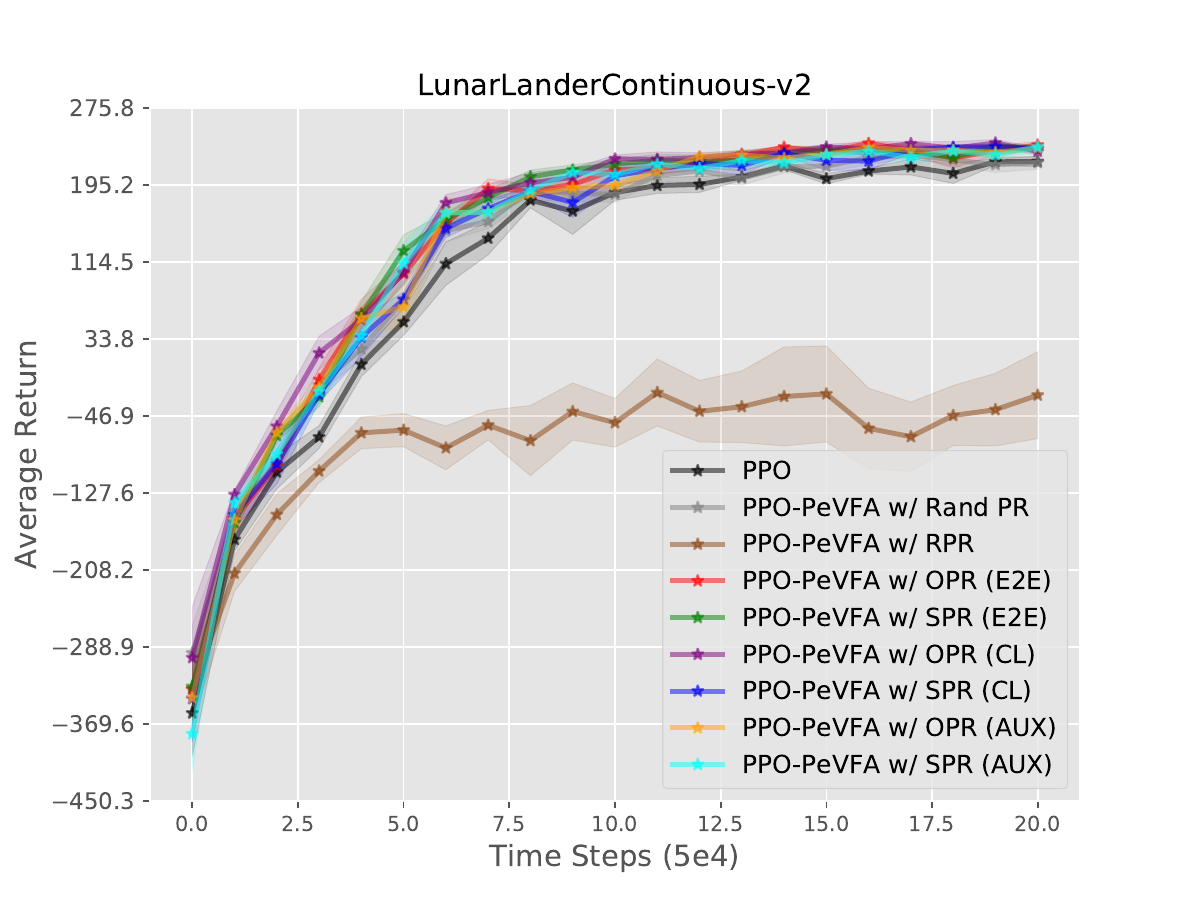}
}
\caption{
An overall view of performance evaluations of different algorithms in MuJoCo continuous control tasks.
The results are average returns and the shaded region denotes half a standard deviation over 10 trials.
}
\label{figure:exp_curves_overall}
\end{figure}




\subsection{Visualization of Learned Policy Representation}
\label{app:complete_visualization}
To show how the learned representation is like in a low-dimensional space,
we visualize the learned representation of policies encountered during the learning process.

\paragraph{Visualization Design.}
We record all policies on the policy improvement path during the learning process of a PPO-PeVFA agent.
For each trial in our experiments in MuJoCo continuous control tasks,
about 1k - 2k policies are collected.
We run 5 trials and 5k - 12k policies are collected in total for each task.
We also store the policy representation model at intervals for each trial,
and we use last three checkpoints to compute the representation of each policy collected.
For each policy collected during 5 trials, its representation for visualization is obtained by averaging the results of 3 checkpoints of each trial and then concatenating the results from 5 trials. 
Finally, we plot 2D embedding of policy representations prepared above through t-SNE \citep{Maaten2008TSNE} and Principal Component Analysis (PCA) in \texttt{sklearn}\footnote{\url{https://scikit-learn.org/stable/index.html}}.

\paragraph{Results and Analysis.}
Visualizations of OPR and SPR learned in an end-to-end fashion in HalfCheetah-v1 and Ant-v1 are in Fig.\ref{figure:visual_opr} and \ref{figure:visual_spr}.
We use different types of markers to distinguish policies from different trials to see how policy evolves in representation space from different random initialization.
Moreover, we provide two views: performance view and process view,
to see how policies are aligned in representation space regarding performance and `age' of policies respectively.

Visualization of OPR trained in end-to-end fashion is shown in Fig.\ref{figure:visual_opr}.
From the performance view, it is obvious that policies of poor and good performances are aligned from left to right in t-SNE representation space and are aligned at two distinct directions in PCA representation space.
An evolvement of policies from different trials can be observed in subplot (b) and (d).
Thus, policies from different trials are locally continuous;
while policies are globally consistent in representation space with respect to policy performance.
Moreover, we can observe multimodality for policies with comparable performance.
This means that the obtained representation not only reflects optimality information but also maintains the behavioral characteristic of policy.

\begin{figure}[ht]
\centering
\hspace{-0.3cm}
\subfigure[Performance View in HalfCheetah-v1]{
\includegraphics[width=0.49\textwidth]{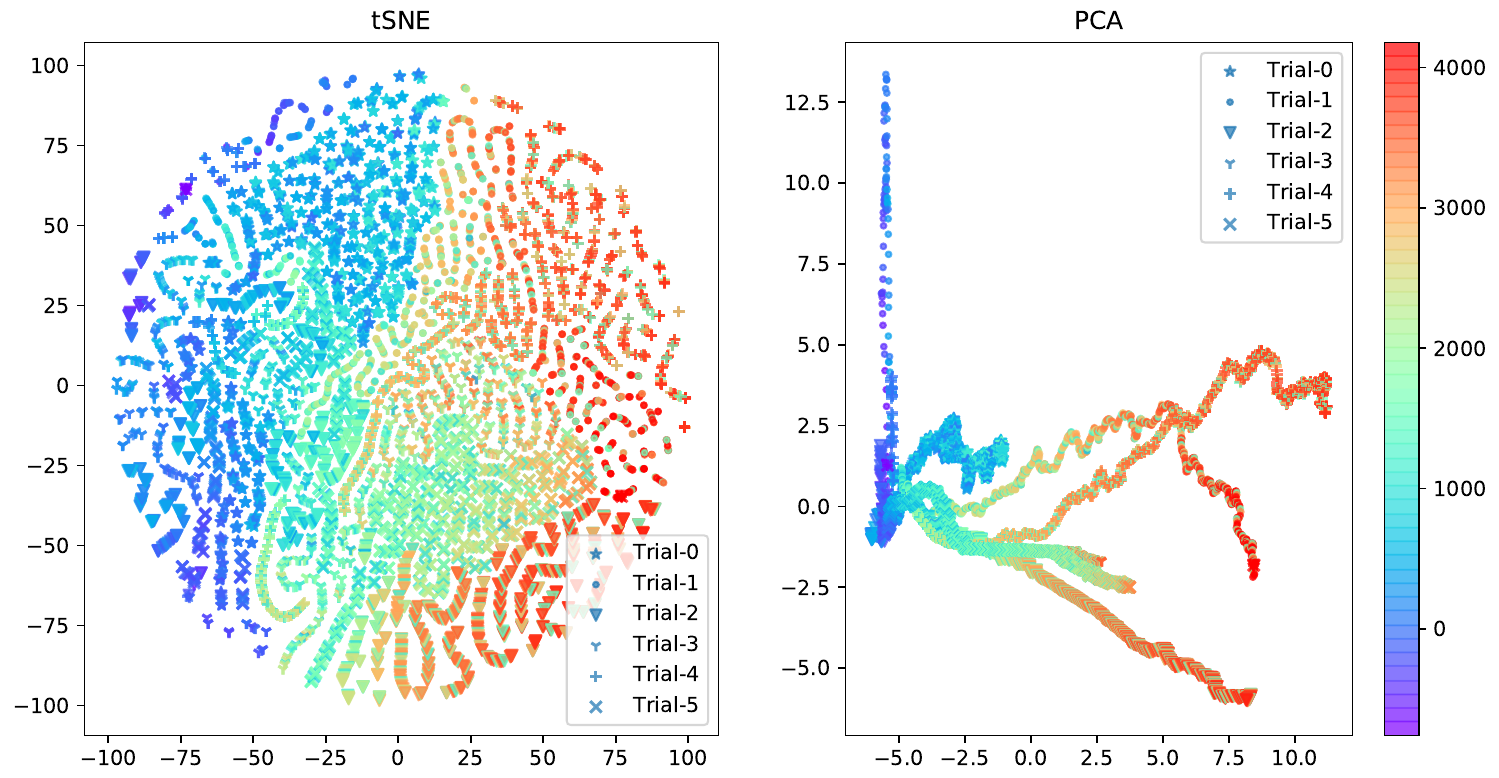}
}
\hspace{-0.25cm}
\subfigure[Process View in HalfCheetah-v1]{
\includegraphics[width=0.49\textwidth]{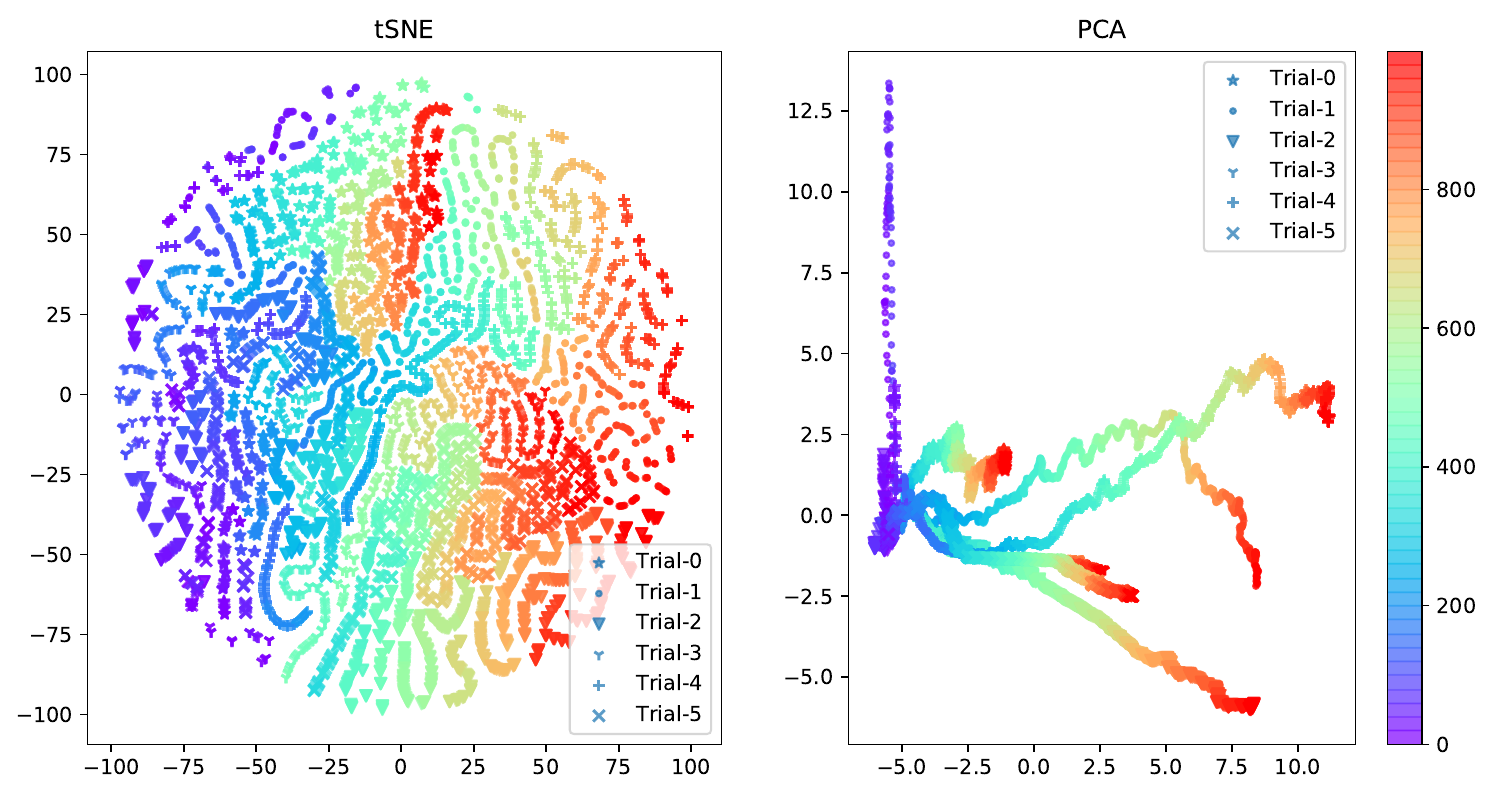}
}
\hspace{-0.3cm}
\subfigure[Performance View in Ant-v1]{
\includegraphics[width=0.49\textwidth]{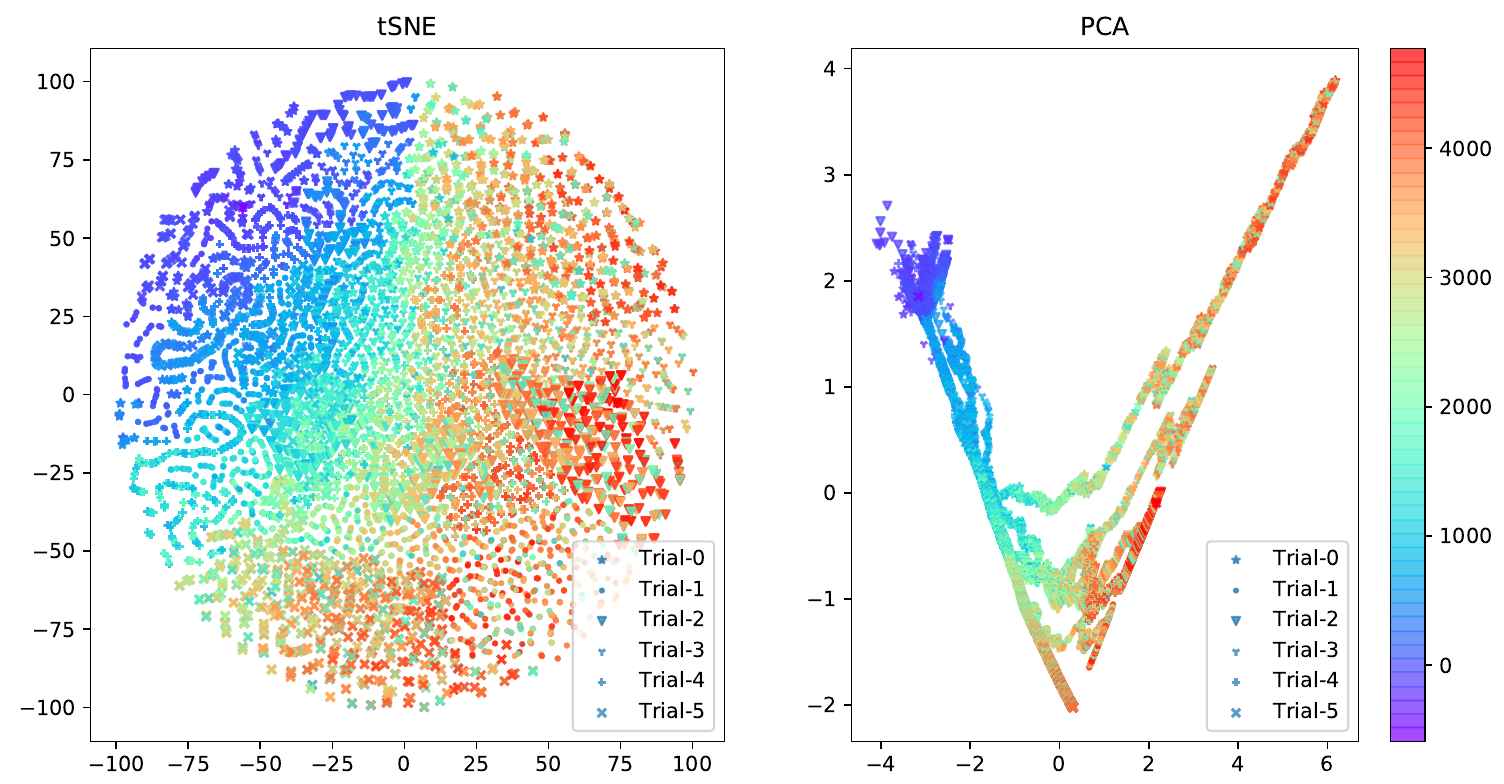}
}
\hspace{-0.25cm}
\subfigure[Process View in Ant-v1]{
\includegraphics[width=0.49\textwidth]{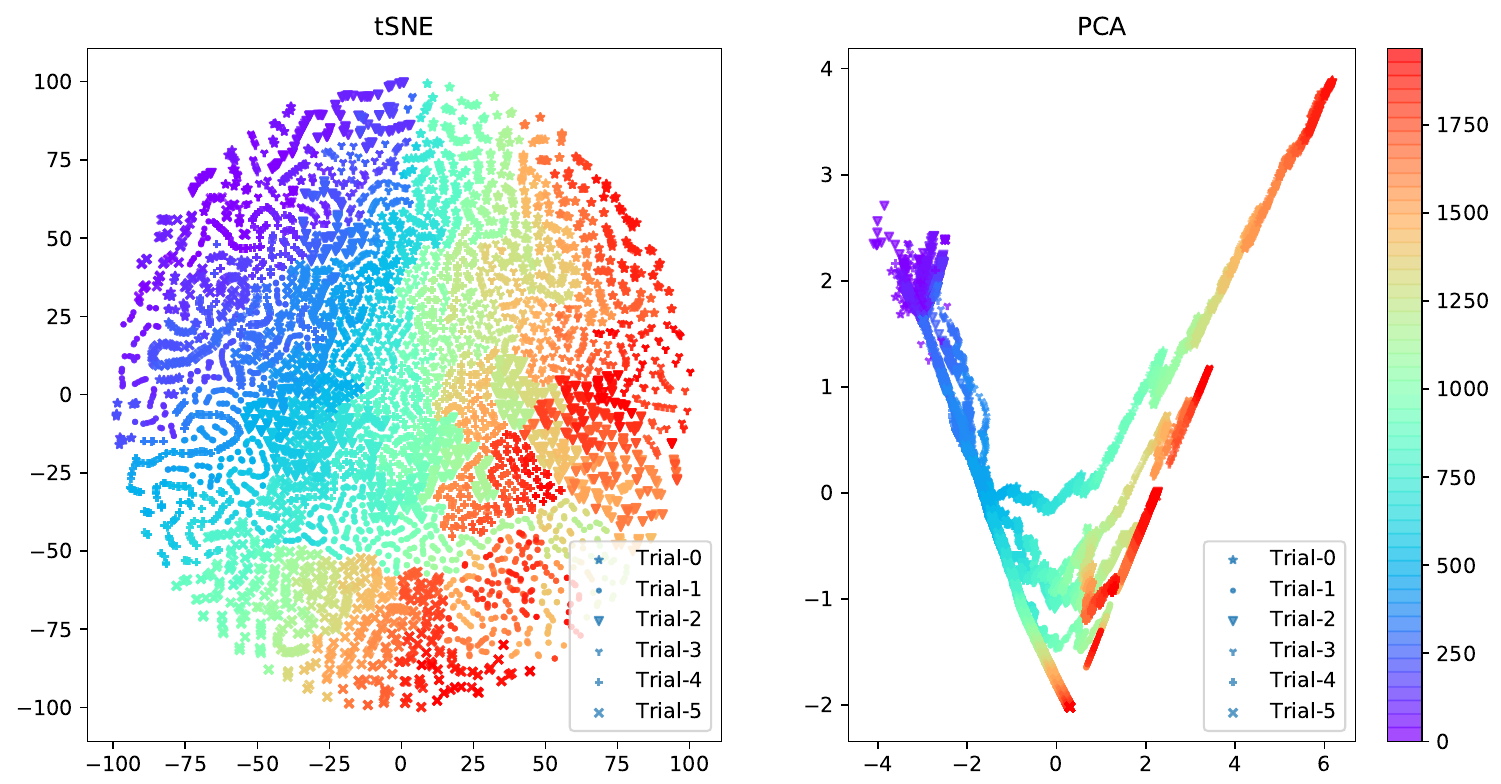}
}
\caption{
Visualizations of end-to-end (E2E) learned Origin Policy Representation (OPR) for policies collected during 5 trials (denoted by different kinds of markers).
In total, about 6k policies are plotted for HalfCheetah-v1 (\textit{a-b}) and 12k for Ant-v1 (\textit{c-d}).
In each subplot, t-SNE and PCA 2D embeddings are at left and right respectively.
In performance view, 
each policy (i.e., marker) is colored by its performance evaluation (averaged return).
In process view, each policy is colored by its corresponding iteration ID during GPI process.}
\label{figure:visual_opr}
\end{figure}

Parallel to OPR, end-to-end trained SPR is visualized in Fig.\ref{figure:visual_spr}.
A more obvious multimodality can be observed in both t-SNE and PCA space:
policies from different trials start from the same region and then diverge during the following learning process. 
Different from OPR, SPR shows more distinction among different trials since SPR is a more direct reflection of policy behavior (\textit{dynamics} property as mentioned in Sec. \ref{app:critira}).
Another thing is, policies from different trials forms wide `strands' especially in t-SNE representation space.
We conjecture that it is because SPR is a more stochastic way to obtain representation as random selected state-action pairs are used.

\begin{figure}[ht]
\centering
\hspace{-0.3cm}
\subfigure[Performance View in HalfCheetah-v1]{
\includegraphics[width=0.49\textwidth]{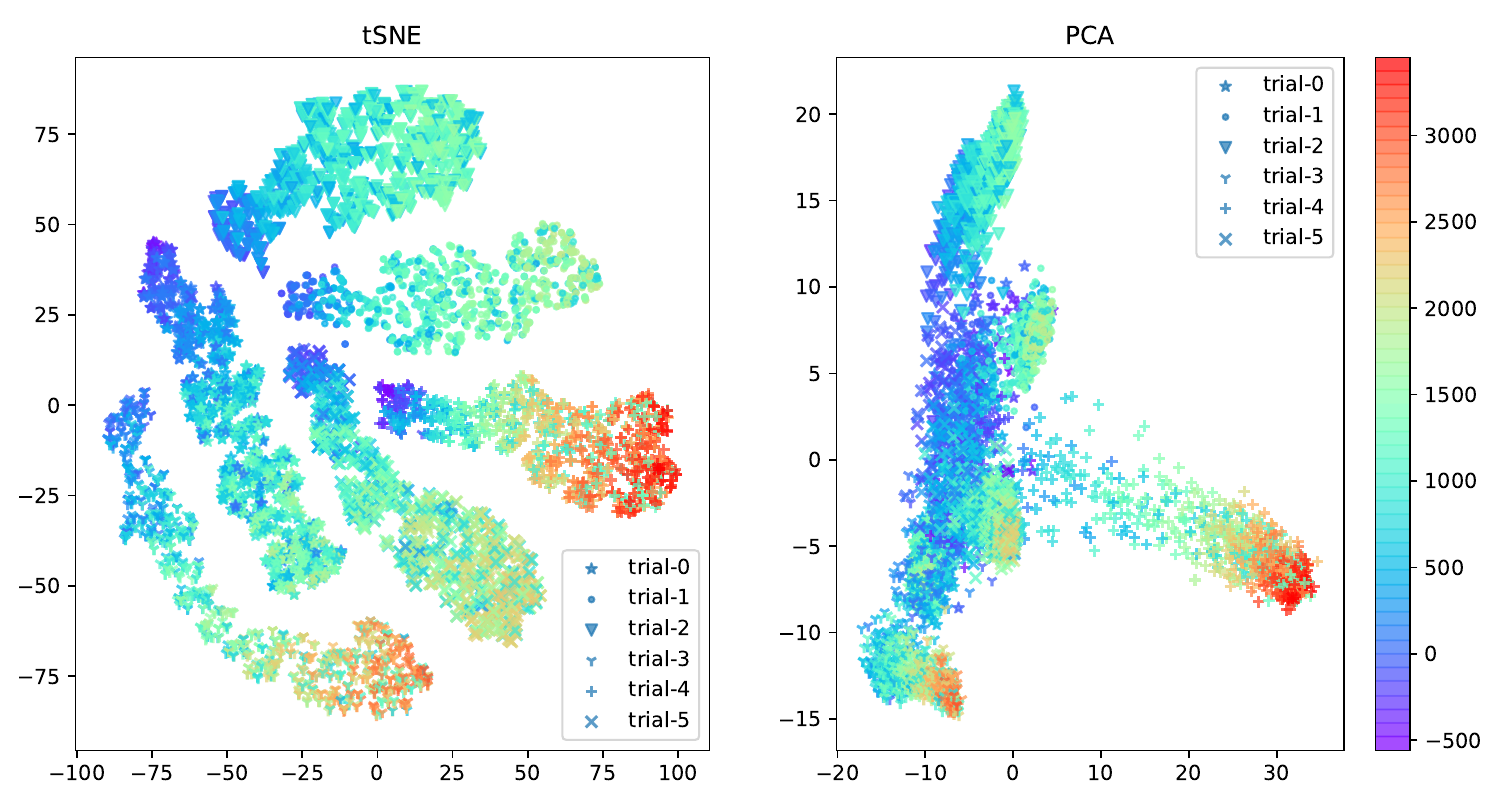}
}
\hspace{-0.25cm}
\subfigure[Process View in HalfCheetah-v1]{
\includegraphics[width=0.49\textwidth]{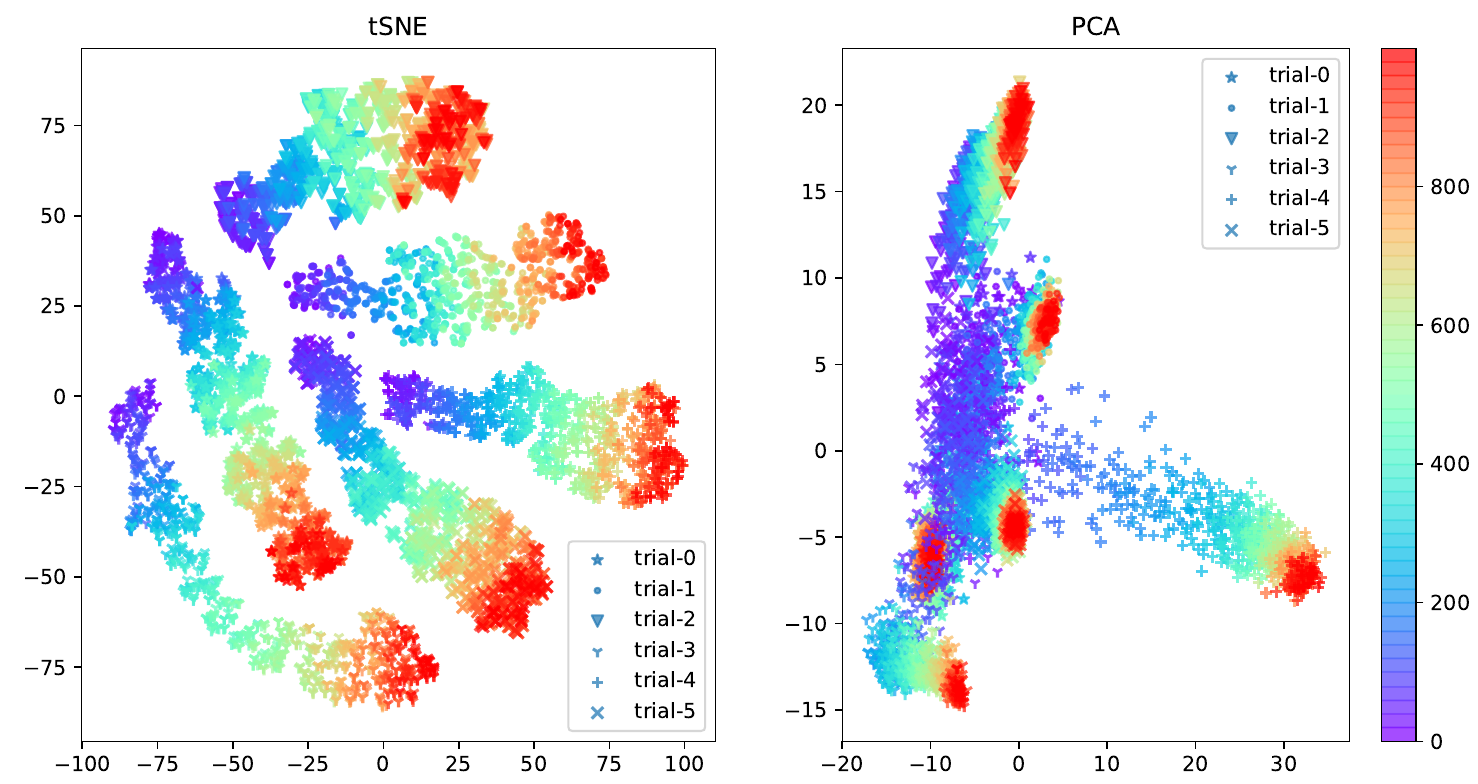}
}
\hspace{-0.3cm}
\subfigure[Performance View in Ant-v1]{
\includegraphics[width=0.49\textwidth]{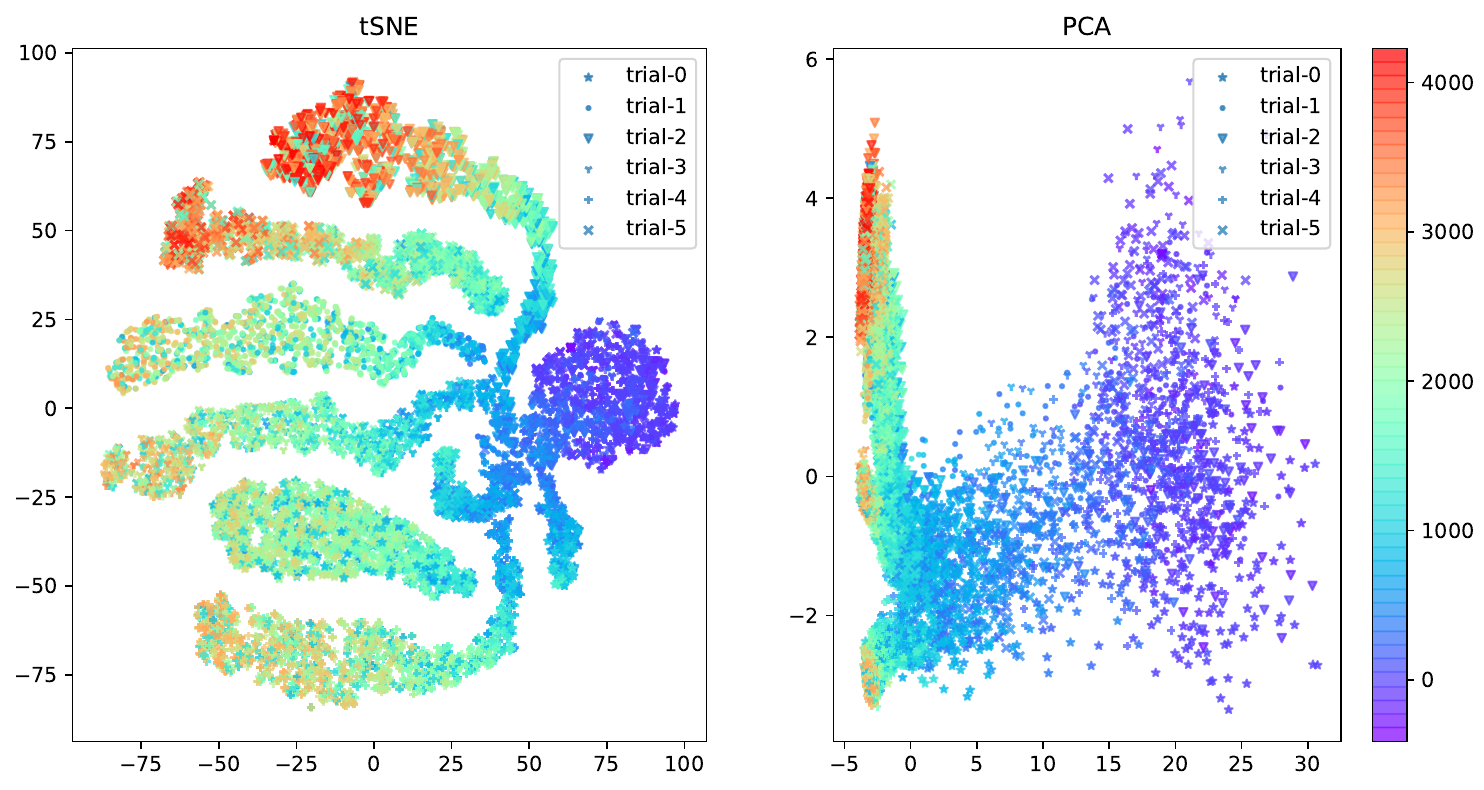}
}
\hspace{-0.25cm}
\subfigure[Process View in Ant-v1]{
\includegraphics[width=0.49\textwidth]{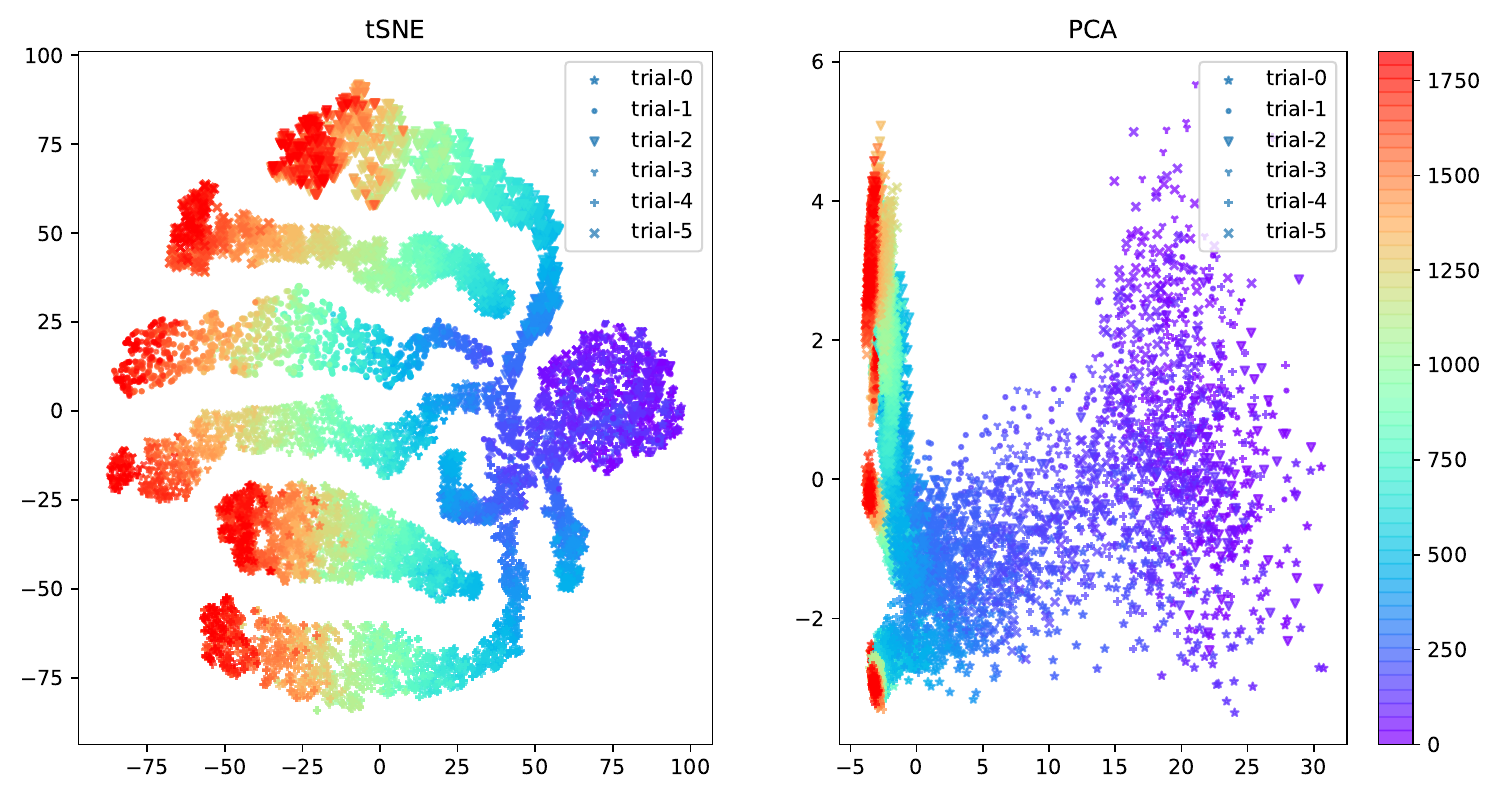}
}
\caption{
Visualizations of end-to-end (E2E) learned Surface Policy Representation (SPR) for policies collected during 5 trials (denoted by different kinds of markers).
In performance view, 
each policy (i.e., marker) is colored by its performance evaluation (averaged return).
In process view, each policy is colored by its corresponding iteration ID during GPI process.
}
\label{figure:visual_spr}
\end{figure}


\end{document}